\documentclass{article}

\PassOptionsToPackage{numbers}{natbib}
\usepackage[final]{neurips_2021}

\usepackage[utf8]{inputenc} 
\usepackage[T1]{fontenc}    
\usepackage{hyperref}       
\usepackage{url}            
\usepackage{booktabs}       
\usepackage{amsfonts}       
\usepackage{nicefrac}       
\usepackage{microtype}      
\usepackage{xcolor}         

\usepackage{color}
\usepackage{amsmath}
\usepackage{amsthm}
\usepackage{amsfonts}
\usepackage{amssymb}
\usepackage{algorithm}
\usepackage{algorithmic}
\usepackage{graphicx}
\usepackage{multirow}
\usepackage{subcaption}

\newcommand{\exprv}{E}
\newcommand{\expreal}{e}
\newcommand{\gumbrv}{G}
\newcommand{\gumbreal}{g}
\newcommand{\structrv}{X}
\newcommand{\structreal}{x}
\newcommand{\catrv}{X}
\newcommand{\catreal}{x}
\newcommand{\tracerv}{T}
\newcommand{\tracereal}{t}
\newcommand{\thetadim}{d}
\newcommand{\gumbparam}{\theta}
\newcommand{\expparam}{\lambda}

\newcommand{\logprob}[1]{\log \mathbb{P}{(#1)}}
\newcommand{\logprobnew}[2]{\log \mathbb{P}_{#1}{(#2)}}
\newcommand{\prob}[1]{\mathbb{P}{(#1)}}
\newcommand{\probnew}[2]{\mathbb{P}_{#1}{(#2)}}
\newcommand{\loss}[1]{\mathcal{L}(#1)}
\newcommand{\temp}{\tau}

\newcommand{\argtopk}{\arg \operatorname{top} k}

\newcommand{\var}[2]{\operatorname{var}_{#1}{[#2]}}

\newtheorem{lemma}{Lemma}
\newtheorem{theorem}{Theorem}
\newtheorem{corollary}{Corollary}

\floatname{algorithm}{Algorithm}
\renewcommand{\algorithmicrequire}{\textbf{Input:}}
\renewcommand{\algorithmicensure}{\textbf{Output:}}

\title{Leveraging Recursive Gumbel-Max Trick for Approximate Inference in Combinatorial Spaces}
\author{%
  Kirill Struminsky\thanks{Equal contribution} \\
  HSE University\\
  Moscow, Russia\\
  \texttt{k.struminsky@gmail.com}
  \And
  Artyom Gadetsky\footnotemark[1]\ \ \thanks{Corresponding author} \\
  HSE University\\
  Moscow, Russia \\
  \texttt{artygadetsky@yandex.ru}
  \And
  Denis Rakitin\footnotemark[1] \\
  HSE University, Skoltech\thanks{Skolkovo Institute of Science and Technology} \\
  Moscow, Russia \\
  \texttt{rakitindenis32@gmail.com}
  \And
  Danil Karpushkin \\
  AIRI,\thanks{Airtificial Intelligence Research Institute}\; Sber AI Lab, MIPT\thanks{Moscow Institute of Physics and Technology}\\
  Moscow, Russia\\
  \texttt{kardanil@mail.ru}
  \And
  Dmitry Vetrov \\
  HSE University, AIRI\footnotemark[4]\\
  Moscow, Russia \\
  \texttt{vetrovd@yandex.ru}
}

\begin{document}

\maketitle

\begin{abstract}
Structured latent variables allow incorporating meaningful prior knowledge into deep learning models.
However, learning with such variables remains challenging because of their discrete nature.
Nowadays, the standard learning approach is to define a latent variable as a perturbed algorithm output and to use a differentiable surrogate for training.
In general, the surrogate puts additional constraints on the model and inevitably leads to biased gradients.
To alleviate these shortcomings, we extend the Gumbel-Max trick to define distributions over structured domains.
We avoid the differentiable surrogates by leveraging the score function estimators for optimization.
In particular, we highlight a family of recursive algorithms with a common feature we call stochastic invariant.
The feature allows us to construct reliable gradient estimates and control variates without additional constraints on the model.
In our experiments, we consider various structured latent variable models and achieve results competitive with relaxation-based counterparts.
\end{abstract}

\section{Introduction}
To this day, the majority of deep learning architectures consists of differentiable computation blocks and relies on gradient estimates for learning. At the same time, architectures with discrete intermediate components are a good fit for incorporating inductive biases \cite{cordonnier2021differentiable, xu2015show} or dynamic control flow \cite{le2020revisiting, graves2016adaptive}. One of the approaches to train such architectures is to replace the discrete component with a stochastic latent variable and optimize the expected objective.

In practice, the expectation has high computational cost, thus one typically resorts to stochastic estimates for the expectation and its gradient. Particularly, the two prevalent approaches to estimate the gradient of the objective are the score function estimator \cite{williams1992simple} and the reparameterization trick \cite{kingma2013auto, rezende2014stochastic} for relaxed discrete variables \cite{maddison2016concrete, jang2016categorical}. The former puts mild assumptions on the distribution and the objective, requiring the gradient of log-probability with respect to the distribution parameters to be differentiable, and provides unbiased estimates for the objective gradient. However, the naive estimate suffers from high variance and is less intuitive in implementation. In comparison, the reparameterized gradient estimates seamlessly integrate within the backpropagation algorithm and exhibit low variance out of the box. At the same time, the relaxation requires an architecture to be defined on the extended domain of the relaxed variable and introduces bias to the gradient estimate.

In the recent years, the attention of the community shifted towards models with structured latent variables. Informally, a structured variable models a distribution over structured objects such as graphs \cite{corro2019differentiable, paulus2020gradient}, sequences \cite{fu2020latent} or matchings \cite{mena2018learning}. Such latent variable may alter the computation graph or represent a generative process of data. Often, a structured variable is represented as a sequence of categorical random variables with a joint distribution incorporating the structure constraints (e.g., the fixed number of edges in an adjacency matrix of a tree). Recent works on structured latent variables address model training largely through the reparameterization trick using relaxed variables. In fact, the Gumbel-Softmax trick naturally translates to structured variables when $\operatorname{arg}\max$ operator is applied over a structured domain rather than component-wise \cite{paulus2020gradient}. In contrast, score function estimators are now less common in structured domain, with a few exceptions such as \cite{yogatama2016learning, havrylov2019cooperative}. The primary difficulty is the sample score function: neither Gibbs distributions, nor distribution defined through a generative process have a general shortcut to compute it.

In our work, we develop a framework to define structured variables along with a low-variance score function estimator. Our goal is to allow training models that do not admit relaxed variables and to improve optimization by alleviating the bias of the relaxed estimators. To achieve the goal we define the structured variable as an output of an algorithm with a perturbed input. Then, we outline a family of algorithms with a common property we call \emph{stochastic invariant}. The property was inspired by the observation in \cite[Appendix, Sec. B]{paulus2020gradient}, where the authors showed that the Kruskal's algorithm \cite{kruskal1956shortest} and the CLE algorithm \cite{edmonds1967optimum} are recursively applying the Gumbel-Max trick. We construct new algorithms with the same property and show how to use the property to learn structured latent variables. In the experimental section, we report performance on par with relaxation-based methods and apply the framework in a setup that does not allow relaxations.


\section{The Recursive Gumbel-Max Trick in Algorithms With Stochastic Invariants}
\label{sec:main}

The section below shows how to define a distribution over structured domain. Conceptually, we define a structured random variable as an output of an algorithm with a random input (e.g., to generate a random tree we return the minimum spanning tree of a graph with random weights). A common solution to incorporate such variable in a latent variable model is to replace the original algorithm with a differentiable approximation to allow gradient-based learning\cite{corro2019differentiable,mena2018learning}. Such solution bypasses the difficutly of computing the sample probability. In contrast, we outline a family of algorithms for which we can get the probability of each intermediate computation step. To get the probabilities we restrict our attention to algorithms with specific recursive structure and random inputs with exponential distribution. In the next section, we leverage the probabilities for approximate inference in latent variable models without the differentiable approximations of the algorithm.

\subsection{The Gumbel-Max Trick in \texorpdfstring{$\argtopk$}{Lg}}

We illustrate our framework with a recursive algorithm generating a subset of a fixed size. The lemma below is a well-known result used to generate categorical random variables using a sequence of exponential random variables.
\begin{lemma}
\label{lemma:exp-min}
\textit{(the Exponential-Min trick)} If $\exprv_i \sim \operatorname{Exp}{(\expparam_i)}, i \in \{1, \dots, \thetadim\}$ are independent, then for $\catrv := \operatorname{argmin}\limits_{i} \exprv_i$
\begin{enumerate}
\item the outcome probability is $\probnew{\catrv}{\catrv = \catreal; \expparam} \propto \expparam_\catreal$;
\item random variables $\exprv_i' := \exprv_i - \exprv_\catrv, i \in \{1, \dots, \thetadim\}$ are mutually independent given $\catrv$ with
$\exprv_i' \mid \catrv \sim \operatorname{Exp}(\expparam_i)$ when $i \neq \catrv$ and $\exprv_i' = 0$ otherwise.\footnote{As a convention, we assume that $0 \overset{d}{=} \operatorname{Exp}{(\infty)}$.}
\end{enumerate}
\end{lemma}
The lemma is equivalent to \emph{The Gumbel-Max trick}, defined for the variables $\gumbrv_i := - \log \exprv_i, i \in \{1,\dots,\thetadim\}$ and the maximum position $\operatorname{argmax}\limits_{i} \gumbrv_i$. In the above case, a variable $\gumbrv_i$ has a Gumbel distribution with the location parameter $\gumbparam_i = \log \expparam_i$, hence the name. Though the Gumbel-Max trick formulation is more common in the literature, we formulate the framework in terms of the exponential distribution and the equivalent Exponential-Min trick. Although the two statements are equivalent and we use their names interchangeably, some of the examples have a natural formulation in terms of the exponential distribution.

Importantly, the second claim in Lemma~\ref{lemma:exp-min} allows applying the Exponential-Min trick succesively. We illustrate this idea with an algorithm for finding top-k elements. We present the recursive form of $\argtopk$ in Algorithm~\ref{alg:top-k}. For each recursion level, the algorithm finds the minimum element, decrements $k$ and calls itself to find the subset excluding the minimum variable. For reasons explained below, the algorithm subtracts the minimum from the sequence $\exprv_{j}' = \exprv_j - \exprv_\tracerv$ before the recursion. This step does not change the output and may seem redundant.

Assuming the input of the algorithm is a vector~$\exprv$ of independent exponential variables with rate parameters $\expparam$, the first argument in the recursive call~$\exprv'$ is again a vector of independent exponential variables (given $\tracerv$) due to Lemma~\ref{lemma:exp-min}. In other words, the input distribution class is \emph{invariant} throughout the recursion. Subtraction of the minimum is not necessary, but it allows to apply Lemma~\ref{lemma:exp-min} directly and simplifies the analysis of the algorithms. Besides that, for each recursion level variable $\tracerv$ has categorical distribution (conditioned on $\tracerv$ found in the above calls) with output probabilities proportional to $\expparam_k, k \in K$.

We use upper indices to denote the recursion depth and, with a slight abuse of notation, denote the concatenation of variables $\tracerv$ for each recursion depth as $\tracerv := (\tracerv^1, \dots, \tracerv^k)$. The output $\structrv$ is a set and does not take into account the order in $\tracerv$. Intuitively, $\tracerv$ acts as \emph{the execution trace} of the algorithm, whereas $\structrv$ contains only partial information about $\tracerv$. The marginal probability of $\structreal$ is the sum $\probnew{\structrv}{\structrv = \structreal; \expparam} = \sum_{\tracereal \in \structrv^{-1}(\structreal)} \probnew{\tracerv}{\tracerv = \tracereal; \expparam}$ over all possible orderings of $\structreal = \{\structreal_1, \dots, \structreal_k\}$ denoted as $\structrv^{-1}(\structreal)$. The direct computation of such sum is prohibitive even for moderate $k$.

The $\argtopk$ illustration is a well-known extension of the Exponential-Min trick. In particular, the distribution of $\tracerv$ is known as the Plackett-Luce distribution \citep{plackett1975analysis} and coincides with $k$ categorical samples without replacement. Following the recursion, the observation probability factorizes according to the chain rule with $i$-th factor governed by equation
$
\probnew{\tracerv_i}{\tracerv_i = \tracereal_i \mid \tracereal_1, \dots, \tracereal_{i-1}; \expparam} = \tfrac{\expparam_{\tracereal_i}}{\sum_{j=1}^{\thetadim} \expparam_j - \sum_{j=1}^{i-1} \expparam_{\tracereal_j}}
$.
We discuss the multiple applications of the trick in Section~\ref{sec:related-work}. Next, we extend Algorithm~\ref{alg:top-k} beyond subset selection.

\begin{figure}
\centering
\begin{minipage}{0.48\textwidth}
\begin{algorithm}[H]
\caption{$F_{\text{top-k}}(E, K, k)$ - finds $k$ smallest elements in a sequence $\exprv$, where $K$ is the set of indices \emph{(keys)} of $\exprv$}
\label{alg:top-k}
\begin{algorithmic}
    \REQUIRE $\exprv, K, k$
    \ENSURE $\structrv$
    \IF{$k = 0$}
    \STATE {\bf return}
    \ENDIF
    \STATE \COMMENT{Find the smallest element}
    \STATE $\tracerv \Leftarrow \arg\min_{j \in K} \exprv_j$ 
    \FOR{$j \in K$}
    \STATE $\exprv_j' \Leftarrow \exprv_j - \exprv_\tracerv$
    \ENDFOR
    \STATE \COMMENT{Exclude $\operatorname{arg}\min$ index $\tracerv$, decrement $k$}
    \STATE $K', k' \Leftarrow K \setminus \{T\}, k - 1$
    \STATE $\exprv' \Leftarrow  \{ \exprv_k' \mid k \in K'\}$
    \STATE \COMMENT{Solve for $k' = k - 1$}
    \STATE $\structrv' \Leftarrow F_{\text{top-k}}(E', K', k')$
    \STATE \textbf{return} $\{\tracerv\} \cup \structrv'$
\end{algorithmic}
\end{algorithm}
\end{minipage}
\enspace
\begin{minipage}{0.48\textwidth}
\begin{algorithm}[H]
\caption{$F_{\text{struct}}(\exprv, K, R)$ - returns structured variable $\structrv$ based on utilities $\exprv$ indexed by $K$ and an auxiliary variable $R$}
\label{alg:general}
\begin{algorithmic}
    \REQUIRE $\exprv, K, R$
    \ENSURE $\structrv$
    \IF{$f_{\text{stop}}(K, R)$}
    \STATE {\bf return}
    \ENDIF
    \STATE $P_1, \dots, P_m \Leftarrow f_{\text{split}}(K, R)$ \hfill \COMMENT{$\sqcup_{i=1}^m P_i = K$}
    \FOR{$i=1$ to $m$}
    \STATE $\tracerv_i \Leftarrow \arg\min_{j \in P_i} \exprv_j$
    \FOR{$j \in P_i$}
    \STATE $\exprv_j' \Leftarrow \exprv_j - \exprv_{\tracerv_i}$
    \ENDFOR
    \ENDFOR
    \STATE $K', R' \Leftarrow f_{\text{map}}(K, R, \{\tracerv_i\}_{i=1}^m)$ \hfill \COMMENT{$K' \subsetneq K$}
    \STATE $E' \Leftarrow \{E_k' \mid k \in K'\}$
    \STATE $\structrv' \Leftarrow F_{\text{struct}}(E', K', R')$ \hfill \COMMENT{Recursive call}
    \STATE {\bf return} $f_{\text{combine}}(\structrv', K, R, \{\tracerv_i\}_{i=1}^m)$
\end{algorithmic}
\end{algorithm}
\end{minipage}
\caption{The recursive algorithm for $\argtopk$ and the general algorithm with the stochastic invariant put side-by-side. Both algorithm perform the Exponential-Min trick and proceed with recursion using a subset of variables. The output $\structrv$ combines the current trace $\tracerv$ and the recursion output $\structrv'$.}
\end{figure}

\subsection{General Algorithm With the Stochastic Invariant}
\label{sec:general}
In this section, we generalize Algorithm~\ref{alg:top-k}. The idea is to preserve the property of Algorithm~\ref{alg:top-k} that allows applying the Exponential-Min trick and abstract away the details to allow various instantiations of the algorithm. Algorithm~\ref{alg:general} is the generalization we put next to Algorithm~\ref{alg:top-k} for comparison. It has a similar recursive structure and abstracts away the details using the auxiliary subrouties: $f_{\text{stop}}$ is the stop condition, $f_{\text{map}}$ and $f_{\text{combine}}$ handle the recursion and $f_{\text{split}}$ is an optional subroutine for the Exponential-Min trick. Additionally, we replace $k$ with an auxiliary argument $R$ used to accumulate information from the above recursion calls. Below, we motivate the subroutines and discuss the properties of a arbitrary instance of Algorithm~\ref{alg:general}.

After checking the stop condition with $f_{\text{stop}}$, Algorithm~\ref{alg:general} applies the Exponential-Min trick simultaneously over $m$ disjoint sets rather than the whole index set $K$. 
For example, such operation occurs when we find columnwise minimum in CLE algorithm\cite{edmonds1967optimum}.
To allow the operation we construct \emph{a partition} of indices $P_1, \dots, P_m$ and find the $\operatorname{arg}\min$ across the partition sets.
To generate the partition, we introduce a new subroutine $f_{split}$ taking the index set $K$ and the auxiliary argument $R$ as inputs.
The partition size $m$ may also be variable.

After the $m$ simultaneous Exponential-Min tricks, the generalized algorithm calls $f_{\text{map}}$ to select a subset of indices $K' \subsetneq K$ and to accumulate the necessary information for the next call in $R'$. Intuitively, the argument $R'$ represents \emph{a reduction} to a smaller problem solved with a recursive call. In the $\argtopk$ example, $K'$ is $K \setminus \{\tracerv\}$ and $R'$ is the decrement $k - 1$. Note that Algorithm~\ref{alg:general} does not allow to capture such information with the other inputs $\exprv'$ and $K'$ exclusively.

Finally, the algorithm calls $f_{\text{combine}}$ to construct the structured variable $\structrv$ using the recursive call output $\structrv'$ and the other variables.
In the top-k example, $f_{\text{combine}}$ appends the minimum variable index $\tracerv$ to the set $\structrv'$.

Now we argue that Algorithm~\ref{alg:general} preserves the invariant observed in Algorithm~\ref{alg:top-k}. Again, we call the sequence of variables $\tracerv = (\tracerv_1, \dots, \tracerv_m)$ \emph{the trace} of the algorithm. By design, if the input $\exprv$ is a sequence of independent exponential random variables, then the recursion input $\exprv'$ conditioned on $\tracerv$ is again a sequence of independent exponential distributions. For short, we call this property \emph{the stochastic invariant}. The key to the stochastic invariant is the signature of the subroutines Algorithm~\ref{alg:general} uses. The algorithm only accesses $\exprv$ values though the Exponential-Min trick. As a result, the intermediate variables $K'$ and $R'$ as well as the output $\structrv$ depend on $\exprv$ only through $\tracerv$. In other words, the execution trace is a function of perturbation $\tracerv = \tracerv(\exprv)$ and the structured variable $\structrv = \structrv(\tracerv)$ is a function of the trace. Additionally, due to Lemma~\ref{lemma:exp-min}, the trace components $\tracerv_1, \dots, \tracerv_m$ have categorical distributions, whereas $\exprv'_k, k \in K$ are exponential random variables. We prove these properties by induction w.r.t. the recursion depth in Appendix~\ref{sec:proofs}.

Given the above, we derive two modifications of Algorithm~\ref{alg:general} generalizing Lemma~\ref{lemma:exp-min} and the Plackett-Luce distribution from the illustration. Algorithm~\ref{alg:log-prob} computes the log-probability $\logprobnew{\tracerv}{\tracereal; \expparam}$ of a trace realization $\tracereal$. In Section~\ref{sec:estimation}, we use the algorithm output to construct gradient estimators. Again, the pseudo-code introduces index $j$ to denote the recursion depth and assumes the input $\tracereal = \{t^j_i\}_{i,j}$ is the concatenation of trace variables for all recursion depths $j=1,\dots,k$. Similarly, in Appendix~\ref{sec:algorithms} we present an algorithm returning a sample from $\exprv \mid \tracerv = \tracereal$ given trace realization $\tracereal$.

\subsection{Further Examples}
\label{sec:examples}
This subsection contains an overview of algorithms with stochastic invariants along with the corresponding structured variables. We present the details and the pseudo-code in Appendix~\ref{sec:pseudo-code}.

Analogous to the $\argtopk$ and the subset variable, the insertion sorting algorithm is an algorithm with the stochastic invariant. In the case of sorting, we do not omit the order of the trace variable $\tracerv$ and return the permutation $\structrv = \tracerv$. The resulting variable $\structrv$ has the Plackett-Luce distribution. We use the variable as a latent variable for insertion-based non-monotonic generation \cite{gu2019insertion}. As an alternative to the Plackett-Luce distribution, we consider a square parameter matrix and find a matching between rows and columns. We perturb the matrix and iteratively find the minimum element in the matrix. We exclude the row and the column containing the element and proceed to the next recursion step. Notably, in contrast to this algorithm, the Hungarian algorithm \cite{kuhn1955hungarian} for the minimum weight matching does not have the stochastic invariant.

As \cite{paulus2020gradient} observe, Kruskal's algorithm \cite{kruskal1956shortest} and Chu-Liu-Edmonds \cite{edmonds1967optimum} algorithm recursively apply the Exponential-Min trick, thus have the stochastic invariant. The former constructs the minimum spanning (MST) tree edge-by-edge. The corresponding trace variable $\tracerv$ is a sequence of edges, whereas $\structrv$ is an unordered set of edges. Interestingly, we could not represent Prim's algorithm \cite{prim1957shortest, dasgupta2008algorithms} for the MST as an algorithm with the stochastic invariant. The Chu-Liu-Edmonds algorithm is an analog of Kruskal's algorithm for directed graphs. It returns the minimum tree $\structrv$ with a fixed root.

Additionally, we construct a latent binary tree variable specifically for the non-monotonic generation model \cite{welleck2019non}. In this model, each token in the sentence corresponds to a node in the tree. We assign weights to nodes and perturb the weights to obtain a sample. During the recursion, we set the minimum weight node to be the parent. We put the nodes on the left-hand side to the left subtree and the nodes on the right-hand side to the right subtree.

\section{Gradient Estimation for the Recursive Gumbel-Max Trick}
\label{sec:estimation}
\begin{algorithm}[t]
\caption{$F_{\text{log-prob}}(\tracereal, \expparam, K, R)$ - returns $\logprobnew{\tracerv}{\tracereal; \expparam}$ for trace $\tracereal$, rates $\expparam$, $K$ and $R$ as in Alg.~\ref{alg:general}}
\label{alg:log-prob}
\begin{algorithmic}
    \REQUIRE $\tracereal, \expparam, K, R$
    \ENSURE $\logprobnew{\tracerv}{\tracereal; \expparam}$
    \IF{$f_{\text{stop}}(K, R)$}
    \STATE {\bf return}
    \ENDIF
    \STATE $P_1, \dots, P_m \Leftarrow f_{\text{split}}(K, R)$ 
    \FOR{$i=1$ to $m$}
    \STATE $\logprobnew{\tracerv}{\tracereal^1_i; \expparam} \Leftarrow \log \expparam_{\tracereal^1_i} - \log \left(\sum_{k \in P_i} \expparam_k \right)$ \hfill \COMMENT{Index $j$ in $T^j_i$ denotes the recursion level}
    \FOR{$k \in P_i \setminus \{\tracereal^1_i\}$ }
    \STATE $\expparam'_{k} \Leftarrow \expparam_k$
    \ENDFOR
    \STATE $\expparam'_{\tracereal^1_i} \Leftarrow +\infty$ \hfill \COMMENT{Because $\exprv'(\tracereal^1_i) = 0$}
    \ENDFOR
    \STATE $K', R' \Leftarrow f_{\text{map}}(K, R, \{\tracereal^1_i\}_{i=1}^m)$
    \STATE $\expparam' \Leftarrow \{ \expparam_k' \mid k \in K'\}$
    \STATE $\logprobnew{\tracerv}{\tracereal^{>1} \mid \tracerv^{1} = \tracereal^{1}; \expparam} \Leftarrow F_{\text{log-prob}}(\tracereal^{>1}, \expparam', K', R')$ \hfill \COMMENT{Compute log-prob of $\tracereal^{>1} := \{\tracereal^j_i\}_{j > 1}$}
    \STATE {\bf return} $ \sum_{i=1}^m \logprobnew{\tracerv}{\tracereal^1_i; \expparam} + \logprobnew{\tracerv}{\tracereal^{>1} \mid \tracerv^1 = \tracereal^1 ; \expparam}$
\end{algorithmic}
\end{algorithm}

In this section, we develop the gradient estimates for the structured distributions defined in Subsection~\ref{sec:general}. We start with a brief introduction of the two main approaches to gradient estimation for discrete categorical variables: the score function estimator~\cite{williams1992simple} and the Gumbel-Softmax estimator~\cite{maddison2016concrete, jang2016categorical}. Then, we propose a low-variance modification of the score function estimator for the structured variables based on the intermediate representation of the variable. Finally, we conclude with a discussion of control variates we use together with the proposed estimator.

\subsection{Gradient Estimation for Categorical Variables}

We consider gradient estimates of an expected objective $\nabla_\expparam \mathbb E_\catrv \loss{\catrv}$, where a discrete random variable $\catrv$ has parametric distribution $\probnew{\catrv}{\cdot; \expparam}$ with finite support. The basic \emph{score function estimator}~\citep{williams1992simple}, also known as REINFORCE, defines an unbiased estimate for the gradient using a sample $\catreal$ as $\loss{\catreal} \nabla_\expparam \logprobnew{\catrv}{\catrv = \catreal; \expparam}$. The estimator does not make any assumptions about $\loss{\cdot}$, but requires an efficient sampling algorithm for ${\catrv}$ and the score function $\nabla_\expparam \logprobnew{\catrv}{\catrv = \catreal; \expparam}$.
For a categorical random variable~$\catrv$ with outcome probabilities $\probnew{\catrv}{\catrv = k; \expparam} \propto \expparam_k$ computation of $\nabla_{\expparam} \logprobnew{\catrv}{\catrv = \catreal; \expparam}$ is linear in the number of outcomes $\thetadim$. Therefore, the gradient estimation is fast when $\thetadim$ is small. However, for structured variables, such as graphs or sequences, the number of outcomes $\thetadim$ grows rapidly with the structure size. In this case, the estimator requires custom algorithms for sampling and estimating the score function.

\emph{The Gumbel-softmax estimator}, introduced in \cite{jang2016categorical, maddison2016concrete}, is an alternative estimator that defines a continuous relaxation based on Lemma~\ref{lemma:exp-min}. On the forward pass, it replaces the categorical variable $\catrv$ with a differentiable surrogate $\tilde{\catreal} = \operatorname{soft}\max{(\tfrac{\gumbreal}{\temp})}$, where the input $\gumbreal := (-\log \expreal_1, \dots, -\log \expreal_\thetadim)$ is a component-wise transformation of exponential samples. Due to Lemma~\ref{lemma:exp-min}, the surrogate converges to the one-hot encoding of a categorical sample $\catreal$ as $\operatorname{soft}\max$ converges to $\operatorname{arg}\max$ when $\temp \rightarrow 0$. On the backward pass, the estimator uses the chain rule to construct the gradient $\nabla_\expparam \loss{\tilde{\catreal}}$ using the reparameterization trick~\cite{kingma2013auto, rezende2014stochastic} to define the partial derivative of a sample as $\tfrac{\partial \expreal_i}{\partial \expparam_i} = -\tfrac{\expreal_i}{\expparam_i}$. The Gumbel-Softmax estimator naturally extends to structured variables \cite{mena2018learning, paulus2020gradient}. Specifically, the component-wise optimization in Lemma~\ref{lemma:exp-min} can be replaced with a linear program over a structured set to generate structured variables and a relaxation can be used to define gradients. In the experimental section, we consider Stochastic Softmax Tricks (SST), introduced in \cite{paulus2020gradient}, as the relaxation-based baseline for comparison.

As opposed to the score function estimator, the Gumbel-Softmax estimator requires a differentiable loss~$\loss{\cdot}$. Such requirement imposes an additional restriction on a model architecture. The architecture must be defined for the relaxed samples as well as the hard samples, a non-trivial requirement for the models where discrete variables define branching \cite{le2020revisiting} or the inputs the model is not allowed to see \cite{kool2018attention,gu2019insertion}. In practice, the vanilla score function estimator has notoriously higher variance compared to the Gumbel-Softmax estimator and requires a control variate to improve the gradient descent convergence.

\subsection{The Score Function Estimator for the Recursive Gumbel-Max Trick}

In Subsection~\ref{sec:general}, we have introduced a probabilistic model involving an exponential variable $\exprv$ and the structured variable $\structrv$ defined as an output of an algorithm with input $\exprv$. Additionally, we have defined an intermediate trace variable $\tracerv = \tracerv(\exprv)$ such that $\structrv$ is a function $\structrv = \structrv(\tracerv)$. In this subsection, we apply $\tracerv$ to estimate gradients of $\mathbb E_\structrv \loss{\structrv}$.

In our setup, the score function $\nabla_\expparam \logprobnew{\exprv}{\exprv=\expreal;\expparam}$ is available out of the box.
However, the score function estimator 
\begin{equation}
g_\exprv := \loss{\structrv(\expreal)} \nabla_\expparam \logprobnew{\exprv}{\exprv = \expreal; \expparam},
\end{equation}
which we refer to as $\exprv$-REINFORCE, is rarely used in practice. In fact, the variance of the score function estimator using $\exprv$ as an action space exceeds the variance of the estimator using $\structrv$. On the other hand, the score function for the structured variable $\nabla_\expparam \logprobnew{\structrv}{\structrv = \structreal; \expparam}$ involves marginalization w.r.t. $\exprv$ and may require significant computation resources to estimate.

To mitigate the variance of $g_\exprv$, we follow the observation of \cite[Appendix, Sec. B]{tucker2017rebar} and define another estimator as
\begin{equation}
g_\tracerv := \loss{\structrv(\tracereal)} \nabla_\expparam \logprobnew{\tracerv}{\tracerv = \tracereal; \expparam},
\end{equation}
which is essentially the score function estimator that uses $\tracerv$ rather than $\exprv$. Below we refer to it as $\tracerv$-REINFORCE.
Such estimate can be seen as the score function estimator $g_\exprv$ marginalized over $\exprv$ given $\tracerv$ (Appendix~\ref{sec:proofs} contains the detailed derivation)
\begin{equation}
\loss{\structrv} \nabla_\expparam \logprobnew{\tracerv}{\tracerv; \expparam}
=
\mathbb E_{\exprv \mid \tracerv} \left[ \loss{\structrv} \nabla_\expparam \logprobnew{\exprv}{\exprv;\expparam} \mid \tracerv \right].
\end{equation}
As a result, the proposed gradient estimate $g_\tracerv$ is unbiased
\begin{equation}
\mathbb E_{\exprv} \loss{\structrv} \nabla_\expparam \logprobnew{\exprv}{\exprv;\expparam} = \mathbb E_{\tracerv} \mathbb E_{\exprv \mid \tracerv} \left[ \loss{\structrv} \nabla_\expparam \logprobnew{\exprv}{\exprv;\expparam} \mid \tracerv \right],
\end{equation}
whereas the variance of the estimate does not exceed the variance of $g_\exprv$
\begin{equation}
\var{\exprv}{g_\exprv} = \var{\tracerv}{\mathbb E_{\exprv \mid \tracerv} g_\exprv} + \mathbb E_{\tracerv} \var{\exprv \mid \tracerv}{g_\exprv} = \var{\tracerv}{g_\tracerv} + \mathbb E_{\tracerv}\var{\exprv \mid \tracerv}{g_\exprv} \geq \var{\tracerv}{g_\tracerv}.
\end{equation}
In fact, in our experiments, we observed a significant difference in optimization due to the reduced variance of the estimator.

As we have argued in Subsection~\ref{sec:general}, we can compute the score function for trace variable and apply the estimator $g_\tracerv$ in practice.
Similarly, marginalization with respect to $\tracerv \mid \structrv$ leads to the score function estimator $g_\structrv := \loss{\structreal} \nabla_\expparam \logprobnew{\structrv}{\structrv = \structreal; \expparam}$ and reduces the variance even further $g_\tracerv: \var{\tracerv}{g_\tracerv} \geq \var{\structrv}{g_\structrv}$. Therefore, the standard score function estimator is preferable when $\nabla_\expparam \logprobnew{\structrv}{\structrv = \structreal; \expparam}$ is available. In other cases, $g_\tracerv$ is a practical alternative. 

\subsection{Further Variance Reduction for the Score Function Estimator}

In addition to the marginalization described above, we mitigate the variance of the score function estimator with control variates. We use two strategies to construct the control variates. The first strategy uses the algorithm for conditional reparameterization of $\exprv \mid \tracerv$ (Appendix~\ref{sec:algorithms}, Algorithm~\ref{alg:conditional}) and defines a family of sample-dependent control variates for the score function estimator \citep{tucker2017rebar, grathwohl2017backpropagation}.
The estimator generates a sample $\expreal$, runs the corresponding algorithm to obtain $\tracereal$ and $\structreal = \structrv(\tracereal)$, adds a control variate $c(\expreal)$ and uses an independent sample $\tilde{\expreal}$ from the conditional distribution $\exprv \mid \tracerv = \tracereal$ to eliminate the introduced bias
\begin{equation}
\left(\loss{\structrv(\tracereal)} - c({\tilde{\expreal}})\right) \nabla_\expparam \logprobnew{\tracerv}{\tracerv=\tracereal; \expparam} - \nabla_\expparam c(\tilde{\expreal}) + \nabla_\expparam c(\expreal),
\end{equation}
In general, the above estimate extends to any pair of random variables $(B, Z)$ such that $B = B(Z)$ and the conditional distribution $Z \mid B$ admits the reparameterization trick. In \cite{tucker2017rebar}, the control variate used the relaxed loss~$\loss{\cdot}$, whereas \cite{grathwohl2017backpropagation} proposed to learn the control variate to improve the training dynamic. In our experiments, we use the estimator of \cite{grathwohl2017backpropagation} and refer to it as RELAX.

The second family of control variates we consider uses $K > 1$ samples $\tracereal_1,\dots, \tracereal_K$ to reduce the variance. Besides averaging the independent estimates, it uses the objective sample mean $\bar{\mathcal{L}} := \tfrac{\sum_{i=1}^K \loss{\structrv(\tracereal_i)}}{K}$ to reduce the variance even further:
\begin{equation}
\label{eq:reinforce-plus}
\frac{1}{K - 1} \sum_{i=1}^{K} \left(\loss{\structrv(\tracereal_i)} - \overline{\mathcal{L}}\right) \nabla_{\expparam} \logprobnew{\tracerv}{\tracerv = \tracereal_i; \expparam}.
\end{equation}
Despite being quite simple, the above leave-one-out estimator~\cite{kool2019buy} proved to be competitive with multiple recent works \cite{dongcoupled, richter2020vargrad}. In our experiments, we refer to such estimator as $\tracerv$-REINFORCE+.\footnote{We denote the analogue, which uses the exponential score instead of the score of the trace, as $\exprv$-REINFORCE+} 
To facilitate fair comparison, in a batch training setup we reduce the batch size proportionally to $K$.


\section{Related Work}
\label{sec:related-work}

Early models with structured latent variables include HMMs \cite{rabiner1989tutorial}, PCFGs \cite{petrov2008discriminative}, make strong assumptions about the model structure, and typically use EM-algorithm variations for training. This paper continues the line of work on perturbation models \cite{papandreou2011perturb} for distributions over combinatorial sets. Initially, perturbation models approximated Gibbs distributions with an efficient sampling procedure using the MAP oracle for the Gibbs distribution. Later, \cite{jang2016categorical, maddison2016concrete} proposed to relax the component-wise optimization used in the Gumbel-Max trick to facilitate gradient-based learning for Gibbs distributions. The combination of the two approaches, namely a perturbed model together with a bespoke relaxation of the MAP oracle, allows designing learning algorithms for latent subsets \citep{xie2019reparameterizable}, permutations \citep{mena2018learning}, trees \citep{corro2019differentiable} and sequences \citep{fu2020latent}. Recently, \cite{paulus2020gradient} developed a systematic approach to design relaxations for perturbed models with linear MAP oracle. Unlike the previous works, we mainly focus on the score function estimators \citep{williams1992simple} for learning. 

We illustrate the framework with the well-known Gumbel-Top-k trick \cite{yellott1977relationship}. Among the various applications, \cite{xie2020differentiable} used the trick to define a differentiable relaxation for the subset model considered in our paper; meanwhile, \cite{gadetsky2020low, santucci2020gradient} used the trick to define score function estimators for latent permutations. Besides that, \cite{kool2019stochastic, kool2020ancestral} leveraged the trick for sampling without replacement for a certain family of graphical models and design a gradient estimator using the sampler \cite{kool2019buy}. Importantly, \cite[Appendix, Sec. B]{paulus2020gradient} showed that the Kruskal's algorithm and the Chu-Liu-Edmonds algorithm extend the Gumbel-Top-k trick. They used the observation to argue in favor of exponential perturbations. In turn, we generalize the observation and propose a learning algorithm based on the generalization. 

The conditional reparameterization scheme, proposed in our work, allows action-dependent control variates \citep{tucker2017rebar, grathwohl2017backpropagation} for learning the structured variables. However, similarly to \cite{richter2020vargrad, dongcoupled} we observed that often a simpler leave-one-out baseline~\citep{kool2019buy} has better performance. Besides the control variates, \cite{paulus2020rao} recently adopted conditional reparameterization to construct an improved Gumbel Straight-Through estimator \citep{jang2016categorical} for the categorical latent variables.

\section{Applications}
\label{sec:applications}
In the section below, we study various instances of algorithms with stochastic invariants and compare them against relaxation-based Stochastic Softmax Tricks (SST) introduced in \cite{paulus2020gradient}. SST offers a generalization of the well-known Gumbel-Softmax trick to the case of different structured random variables. The experimental setup is largely inherited from \cite{paulus2020gradient}, apart from Subsection \ref{sec:nonmon_main} where the specifics of the problem do not allow relaxation-based gradient estimators, thus showing broader applicability of the score function-based gradient estimators. The main goal of \cite{paulus2020gradient} was to show that introducing structure will lead to superior performance compared to unstructured baselines. In turn, we focus on studying the benefits that one could get from using score-function gradient estimators based on $\tracerv$ rather than $\exprv$ as well as showing competitive performance compared to relaxation-based SSTs. 

Concerning the efficiency of the proposed score-function-based gradient estimators, the recursive form of Algorithm~\ref{alg:general} does not facilitate batch parallelization in general. Specifically, the recursion depth and the decrease in the size of $\exprv$ may differ within a batch. Therefore, Algorithms~\ref{alg:general},\ref{alg:log-prob} may require further optimization. We discuss our implementations in Appendix, Section~\ref{sec:details}, and provide speed benchmarks in Table~\ref{timings-table}. It shows that, in practice, the performance is not much inferior to the relaxation-based competitor. The implementation and how-to-use examples are publicly available\footnote{https://github.com/RakitinDen/pytorch-recursive-gumbel-max-trick}.
\subsection{Learning to Explain by Finding a Fixed Size Subset}



\begin{table*}[t]
\centering
\caption{Results of $k$-subset selection on Aroma aspect data. MSE ($\times 10^{-2}$) and subset precision (\%) is shown for best models selected on validation averaged across different random seeds.}  
\label{table:l2x}
\begin{center}
\begin{small}
\begin{sc}
\scalebox{0.85}{
\begin{tabular}{@{}llllllll@{}}
\toprule
\multirow{2.5}{*}{\shortstack[l]{Model}}& \multirow{2.5}{*}{\shortstack[l]{Estimator}} & \multicolumn{2}{c}{$k=5$}  & \multicolumn{2}{c}{$k=10$} & \multicolumn{2}{c}{$k=15$}  \\ \cmidrule(r){3-4} \cmidrule(r){5-6} \cmidrule(r){7-8}
 &  & mean $\pm$ std & Prec. & mean $\pm$ std & Prec. & mean $\pm$ std & Prec. \\
\cmidrule[\heavyrulewidth]{1-8}
\multirow{4.5}{*}{\shortstack[l]{Simple}}
  & \emph{SST (Our Impl.)}
  	& $\mathbf{3.6 \pm 0.1}$ 
  	& $28 \pm 1.4$ 
  	& $3.21 \pm 0.12$ 
  	& $29.5 \pm 1.7$ 
  	& $2.77 \pm 0.09$ 
  	& $\mathbf{28.1 \pm 1.2}$ \\
  \cmidrule[0.15pt]{2-8}
    & \emph{\exprv-REINFORCE+} 
    & $3.89 \pm 0.2$ 
    & $25 \pm 1.4$ 
    & $3.77 \pm 0.23$ 
    & $26.7 \pm 3.4$ 
    & $3.16 \pm 0.16$ 
    & $25.3 \pm 1.3$ \\
  & \emph{\tracerv-REINFORCE+}
    & $3.79 \pm 0.13$ 
    & $\mathbf{30.5 \pm 2.2}$ 
    & $\mathbf{3.14 \pm 0.16}$ 
    & $\mathbf{31 \pm 2.9}$ 
    & $\mathbf{2.69 \pm 0.11}$ 
    & $27.6 \pm 0.9$ \\
  & \emph{RELAX} 
  	& $3.76 \pm 0.11$ 
  	& $24 \pm 1.9$ 
  	& $3.5 \pm 0.13$ 
  	& $28.9 \pm 1.9$ 
  	& $2.95 \pm 0.15$ 
  	& $26.1 \pm 1.9$ \\
  \cmidrule[\heavyrulewidth]{1-8}
  \multirow{4.5}{*}{\shortstack[l]{Complex}}
  & \emph{SST (Our Impl.)}
    & $2.93 \pm 0.09$ 
    & $\mathbf{56 \pm 2.1}$ 
    & $2.55 \pm 0.08$ 
    & $\mathbf{49.4 \pm 3.1}$ 
    & $2.51 \pm 0.05$ 
    & $40.3 \pm 0.9$ \\
  \cmidrule[0.15pt]{2-8}
    & \emph{\exprv-REINFORCE+} 
    & $3.03 \pm 0.2$ 
    & $49.4 \pm 3.3$ 
    & $2.92 \pm 0.12$ 
    & $45 \pm 2.3$ 
    & $2.76 \pm 0.22$ 
    & $42 \pm 2.2$ \\
  & \emph{\tracerv-REINFORCE+}
    & $\mathbf{2.75 \pm 0.08}$ 
    & $55.8 \pm 3.4$ 
    & $\mathbf{2.48 \pm 0.05}$ 
    & $48.6 \pm 2.4$ 
    & $\mathbf{2.4 \pm 0.03}$ 
    & $\mathbf{44.2 \pm 1.3}$ \\
  & \emph{RELAX} 
    & $2.8 \pm 0.08$ 
    & $54.4 \pm 2.1$ 
    & $2.58 \pm 0.09$ 
    & $47.6 \pm 1.6$ 
    & $2.46 \pm 0.08$ 
    & $42.1 \pm 1.5$ \\
\bottomrule
\end{tabular}
}
\end{sc}
\end{small}
\end{center}
\end{table*}
We evaluated our method on the experimental setup from L2X \citep{chen2018learning} on the BeerAdvocate \citep{mcauley2012learning} dataset. The dataset consists of textual beer reviews and numerical ratings for four beer aspects (\textit{Aroma},  \textit{Taste},  \textit{Appearance},  \textit{Palate}). The model utilizes encoder-decoder architecture. The encoder outputs parameters of top-$k$ distribution over subsets of a review given the entire review. The decoder predicts a rating given the review subset of size $k$. Following setup of \citep{paulus2020gradient} we use $k = \{5,10,15\}$ and two CNN architectures: \textit{Simple} --- a one-layer CNN and  \textit{Complex} --- a three-layer CNN and train our models for each aspect using MSE loss. We train models using several score function based estimators and compare them with our implementation of SST. We report means and standard deviations for loss and precision averaged across different random model initializations. Table~\ref{table:l2x} shows the obtained results for \textit{Aroma} aspect. The best metrics with respect to means are highlighted in bold. Detailed experimental setup and description of models can be found in Appendix~\ref{sec:topkexp}

\subsection{Learning Latent Spanning Trees with Kruskal's Algorithm}
\begin{table*}[t]
\centering
\caption{Graph Layout experiment results for T=10 iterations. Metrics are obtained by choosing models with best validation ELBO and averaging results across different random seeds on the test set.}
\label{nri1-table}
\begin{center}
\begin{small}
\begin{sc}
\begin{tabular}{@{}lcccccc@{}}
    \toprule
     & \multicolumn{6}{c}{$T=10$} \\ 
     \cmidrule(l){2-7}     
     Estimator & \multicolumn{2}{c}{ELBO} & \multicolumn{2}{c}{Edge Prec.} & \multicolumn{2}{c}{Edge Rec.} \\
               & mean $\pm$ std & max & mean $\pm$ std & max & mean $\pm$ std & max \\ 
     \midrule
    \emph{SST (Our Impl.)} 
    & $-1860.79 \pm 1116.83$ 
    & $-1374.81$ 
    & $\mathbf{87 \pm 21}$ 
    & $\mathbf{95}$ 
    & $\mathbf{93 \pm 3}$ 
    & $\mathbf{95}$ \\
    \cmidrule[0.15pt]{1-7}
    \emph{\tracerv-REINFORCE+} 
    & $\mathbf{-1582.72 \pm 571.18}$ 
    & $\mathbf{-1192.04}$ 
    & $70 \pm 31$ 
    & $91$ 
    & $86 \pm 8$ 
    & $91$ \\
    \emph{RELAX} 
    & $-2079.18 \pm 569.49$ 
    & $-1205.87$ 
    & $43 \pm 31$ 
    & $90$ 
    & $81 \pm 6$ 
    & $90$ \\
    \bottomrule
    \end{tabular}
\end{sc}
\end{small}
\end{center}
\end{table*}

Given a system of interacting particles, the dependencies of their states can be formally described as an undirected graph. We use Neural Relational Inference \cite{kipf2018neural}, initially representing a relaxation-based approach, and examine its performance as a generative model and ability to reconstruct system interconnections, when applying score function techniques instead. We build an experiment in line with \cite{paulus2020gradient}, generating data by translating a ground truth latent spanning tree (corresponding to the connections in the system) into a sequence of positions of points on a real plane, representing dynamics of the system over time. These points are obtained executing force-directed algorithm \cite{fruchterman1991graph} for $T = 10$ or $T = 20$ iterations and fed into the model.

The architecture of the model consists of a graph neural network (GNN) encoder, producing distribution over spanning trees, and a GNN decoder, producing distribution over time series of points positions. Model is trained in a manner of variational autoencoders (VAEs), optimizing ELBO, a lower bound on the joint log-likelihood of the observed data points at all timesteps. 
 
We measure precision and recall with respect to the encoder samples and the ground truth dependency spanning tree. Table~\ref{nri1-table} shows the results for T=10 iterations. Overall, score function methods performed better than their relaxation-based counterpart, achieving higher values of ELBO on the test set, but slightly worse performance in terms of structure recovery metrics. The results for T=20 and the detailed experimental setup are described in the Appendix~\ref{sec:graphlayoutexp}.

\subsection{Unsupervised Parsing with Rooted Trees with CLE Algorithm}
\begin{table*}[t]
\centering
\caption{Unsupervised Parsing on ListOps. We report the average test-performance of the model with the best validation accuracy across different random initializations.}
\label{edmonds-table}
\begin{center}
\begin{small}
\begin{sc}
\begin{tabular}{*7l}
\toprule 
Estimator &
\multicolumn{2}{c}{Accuracy} &
\multicolumn{2}{c}{Precision} &
\multicolumn{2}{c}{Recall} \\
& mean $\pm$ std & max
& mean $\pm$ std & max
& mean $\pm$ std & max
 \\
\midrule
   \emph{SST (Our Impl.)}
   & $78.42 \pm 8.14$ & $\mathbf{93.78}$ 
    & $56.84 \pm 20.08$ & $\mathbf{82.40}$
    & $30.18 \pm 19.10$ & $73.11$
    \\
  \cmidrule[0.15pt]{1-7}
     \emph{\exprv-REINFORCE+}
   & $60.25 \pm 2.29$ & $64.47$
   & $40.87 \pm 6.90$ & $45.74$
   & $40.74 \pm 6.93$ & $45.46$
     \\
   \emph{\tracerv-REINFORCE+}
   & $\mathbf{87.34 \pm 3.00}$ & $91.97$
   & $\mathbf{77.93 \pm 7.36}$ & $79.65$
   & $\mathbf{61.10 \pm 14.11}$ & $\mathbf{79.65}$
    \\
   \emph{RELAX} 
   & $79.60 \pm 9.36$ & $88.64$ 
   & $54.73 \pm 17.48$ & $75.27$ 
   & $53.61 \pm 17.14$ & $75.27$
   \\
\bottomrule
\end{tabular}
\end{sc}
\end{small}
\end{center}
\end{table*}

We study the ability of the proposed score function estimators to recover the latent structure of the data in a setting, where it can be quite accurately described with an arborescence. Following details about data and models outlined by \cite{paulus2020gradient}, we use a simplified version of the ListOps \cite{nangia2018listops} dataset. It consists of mathematical expressions (e.g. $\texttt{min[3 med[3 5 4] 2]}$), written in prefix form along with results of their evaluation, which are integers in $[0, 9]$. Given a prefix expression, one can algorithmically recover its structure as a parse tree. We bound maximal length of expressions and maximal depth of their parses along with removing the examples with $\texttt{summod}$ operator. These limitations sufficiently decrease the amount of memory a model should have to calculate the result and facilitates the usage of GNNs which now become capable of evaluating expressions by a bounded number of message passing steps.

Our model consists of two parts: an encoder and a classifier. The encoder is a pair of LSTMs that generate parameters of the distribution over rooted arborescence on token nodes. The classifier is a GNN, which passes a fixed number of messages over the sampled arborescence and feeds the resulting embedding of the first token into the final MLP. Models are trained simultaneously to minimize cross-entropy. We examine the performance of the models by measuring classification accuracy along with precision and recall with respect to the edges of ground truth parse trees. Table~\ref{edmonds-table} shows score function based estimators, particularly \tracerv-REINFORCE+, show more stable performance in comparison to relaxation-based estimator. Detailed description of the experiment can be found in the Appendix~\ref{sec:edmondsexp}.

\subsection{Non-monotonic Generation of Balanced Parentheses with Binary Trees}\label{sec:nonmon_main}
\begin{figure}[ht]
\begin{center}
\centerline{\includegraphics[width=\textwidth]{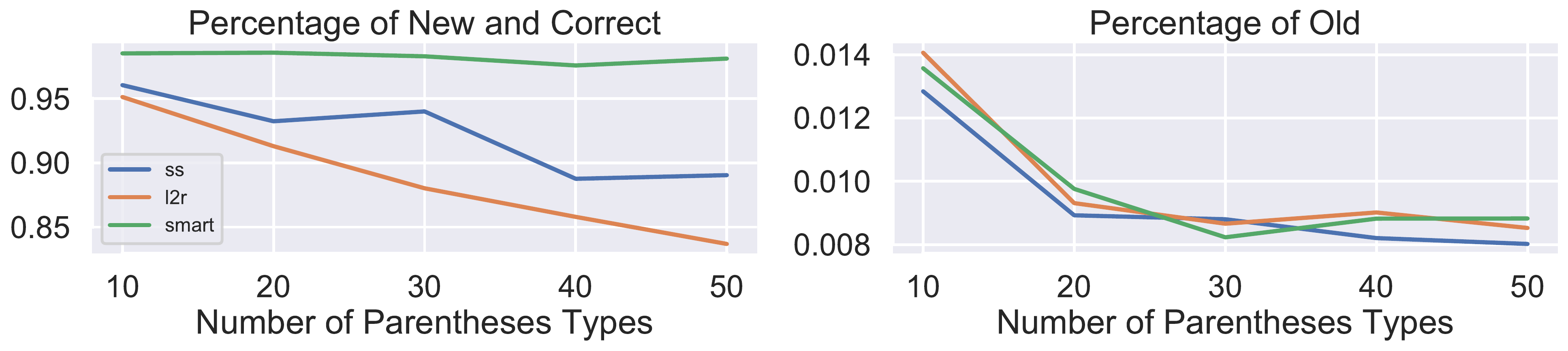}}
\caption{Generative statistics of non-monotonic language model with different orders}
\label{toy-nonmonotonic}
\end{center}
\end{figure}
We apply our methods in a setting where the application of continuous relaxations is seemingly not possible. We construct a simple dataset by sampling a fixed number of balanced parentheses of various types (from 10 to 50 types) and model them with non-monotonic architecture, defined in \cite{welleck2019non}. Here, a language model generates sequences by consecutively adding new tokens in between previously generated ones. This process can be seen as modeling joint distribution over binary trees (treated as orders of generation) and sequences (which are generated along these orders).

We refer to this probabilistic model as the decoder, fix its architecture to be a single-layer LSTM and use teacher-forcing for training. More precisely, we compare two different modes of training decoder: fixed order and semi-supervised trainable order. In the first case, the order is obtained by applying a deterministic function to the input sequence. We study two fixed orders: \textit{left-to-right}, corresponding to a degenerate tree with only right child nodes, and a more natural \textit{smart} order, described in Appendix~\ref{sec:nonmonexp}. In the semi-supervised (\textit{ss}) case, 10\% of sequences are paired with the smart order, and the overall model is trained as a VAE. We train the decoder by directly backpropagating the output signal and obtain the gradients for the encoder using RELAX estimator.

We choose models with the best validation perplexity and evaluate them by generating 200000 unique sequences and measuring portion of those which are present in the train dataset (Old) and those which are balanced and not seen during training (New and Correct). Results of this toy experiment show that it is possible to improve generative metrics by considering non-trivial orders of generation. Experiment is described in details in the Appendix~\ref{sec:nonmonexp}.

\section{Discussion}
Below, we speculate about the pros and cons of the relaxation-based \cite{paulus2020gradient} and the score function-based approaches to training latent structured variable models.
While both build upon the Gumbel-Max trick, the generalizations develop different ideas.
Our work summarizes the recent observations and formulates a general algorithm for which the Gumbel-Max applies recursively.
We utilize the properties of these algorithms to construct unbiased score function-based gradient estimators.
In contrast, \cite{paulus2020gradient} generalizes the relaxation of the Gumbel-Max trick to different combinatorial structures and produces biased reparametrization-based gradient estimators.
While the relaxation-based gradient estimators are biased and limited to differentiable objectives, they have lower variance and faster batch processing time (Tables~\ref{nri-var-table} and \ref{timings-table}).
At the same time, score function-based estimators are unbiased and apply when the loss functions do not admit relaxations (Section~\ref{sec:nonmon_main}).
Occasionally, they lead to better optimization and lower objective values (Tables~\ref{table:l2x-train}, \ref{nri1-train-table} and \ref{edmonds-table-train}).
The choice of control variates also introduces a trade-off in our framework.
As the experiments show, if the objective is highly parallelable, the multi-sample $\tracerv$-REINFORCE+ estimator is preferable.
However, a single-sample RELAX estimator is a more suitable choice when we cannot obtain multiple objective samples (i.e., in reinforcement learning). 
Finally, as a direction for future improvements, we suggest applying the conditional reparameterization scheme used in the control variate to improve the Gumbel straight-through estimators.



\subsubsection*{Acknowledgements}
The authors thank the reviewers for the valuable feedback. The research was supported by the Russian Science Foundation grant no. 19-71-30020 and through the computational resources of HPC facilities at HSE University\cite{kostenetskiy2021hpc}.

\bibliographystyle{plain}
\bibliography{references}

\newpage
\appendix

\section*{Appendix Outline}

\begin{itemize}
\item Section~\ref{sec:proofs}: Proofs for the Exponential-Min trick, the variance inequality for the estimators, the correctness of the trace probability, and the correctness of the conditional sampling;
\item Section~\ref{sec:algorithms}: Pseudo-code for $\logprob{\tracerv; \expparam}$ and the conditional reparameterization of~$\exprv \mid \tracerv$;
\item Section~\ref{sec:details}: Discussion of the algorithm implementations;
\item Section~\ref{sec:experiments}: Additional results and experimental details;
\item Section~\ref{sec:pseudo-code}: Pseudo-code for the algorithms;
\end{itemize}

\section{Proofs}
\label{sec:proofs}
In this section, we provide the proofs for the lemmas and the theorem formulated in the main part.
\subsection{The Exponential-Min Trick}
\begin{lemma}
\label{lemma:exp-min-appendix}
\textit{(the Exponential-Min trick)} If $\exprv_i \sim \operatorname{Exp}{(\expparam_i)}, i \in \{1, \dots, \thetadim\}$ are independent, then for $\catrv := \operatorname{argmin}\limits_{i} \exprv_i$
\begin{enumerate}
\item the outcome probability is $\probnew{\catrv}{\catrv = \catreal; \expparam} \propto \expparam_\catreal$;
\item random variables $\exprv_i' := \exprv_i - \exprv_\catrv, i \in \{1, \dots, \thetadim\}$ are mutually independent given $\catrv$ with
$\exprv_i' \mid \catrv \sim \operatorname{Exp}(\expparam_i)$ when $i \neq \catrv$ and $\exprv_i' = 0$ otherwise.\footnote{As a convention, we assume that $0 \overset{d}{=} \operatorname{Exp}{(\infty)}$.}
\end{enumerate}
\end{lemma}
\begin{proof}
Starting with the joint density of $\catrv$ realization $\catreal$ and $\exprv$ realization $\expreal$
\begin{equation}
\prod_{j=1}^{\thetadim} \left( \expparam_j e^{-\expparam_j e_j} \cdot \mathbb{I}(e_x \leq e_j) \right),
\end{equation}
we make the substitution $e'_j := e_j - e_x$ for $j \in \{1, \dots, \thetadim\}$ and rewrite the density
\begin{align}
\prod_{j=1}^{\thetadim} \left( \expparam_j e^{-\expparam_j (e'_j + e_x)} \cdot \mathbb{I}(0 \leq e'_j) \right) = 
\expparam_x e^{-\sum_{j=1}^\thetadim \expparam_j e_x} \prod_{j = 1}^{\thetadim} \left( \expparam_j e^{-\expparam_j e'_j} \mathbb{I}(0 \leq e'_j) \right) = \\
\tfrac{\expparam_x}{\sum_{j=1}^\thetadim \expparam_j} \times \left(\sum_{j=1}^\thetadim \expparam_j\right) e^{-(\sum_{j=1}^\thetadim \expparam_j) e_x} 
\times \prod_{j = 1}^{\thetadim} \left( \expparam_j e^{-\expparam_j e'_j} \mathbb{I}(0 \leq e'_j) \right).
\end{align}
The latter is the joint density of
\begin{itemize}
\item a categorical latent variable $\catrv$ with $\prob{\catrv = x} \propto \expparam_x$;
\item an independent exponential random variable $\exprv_\catrv$ with rate parameter $\sum_{j=1}^\thetadim \expparam_j$;
\item and a sequence of random variables $\exprv'_j := \exprv_j - \exprv_\catrv, j \neq x$ with mutually independent exponential distributions $\operatorname{Exp}{(\expparam_j)}$ conditioned on $\catrv = \catreal$.
\end{itemize}
\end{proof}

\subsection{Properties of the Score Function Estimators}
This subsection contains the analysis of mean and variance of the defined score function estimators with respect to variables $\exprv$, $\tracerv$ and $\structrv$. We follow the derivations in \cite[Appendix, Sec. B]{tucker2017rebar} and start with lemma about the conditional marginalization. We assume applicability of the log-derivative trick every time it is used.
\begin{lemma}
Consider a random variable $Y$ with distribution from parametric family $\probnew{Y}{\cdot ; \expparam}$ and a function $Z = Z(Y)$. Then, $Z$-REINFORCE estimator is the marginalization of $Y$-REINFORCE estimator with respect to the distribution $Y \mid Z$:

\begin{equation}
\mathbb E_{Y \mid Z} \left[\loss{Z}\nabla_\expparam \logprobnew{Y}{Y ; \expparam}\mid Z \right] = \loss{Z}\nabla_\expparam \logprobnew{Z}{Z ; \expparam}.
\end{equation}
\label{lemma:rao}
\end{lemma}
\begin{proof}
Since $\loss{Z}$ is a function of $Z$, it can be moved outside of the conditional expectation:
\begin{equation}
\mathbb E_{Y \mid Z} \left[\loss{Z}\nabla_\expparam \logprobnew{Y}{Y ; \expparam}\mid Z \right] = \loss{Z} \mathbb E_{Y \mid Z} \left[\nabla_\expparam \logprobnew{Y}{Y ; \expparam}\mid Z \right].
\end{equation}

The only thing that remains is to show that the $Y$-score marginalized over $Y \mid Z$ equals $Z$-score:
\begin{equation}
\mathbb E_{Y \mid Z} \left[\nabla_\expparam \logprobnew{Y}{Y ; \expparam}\mid Z \right] = \nabla_\expparam \logprobnew{Z}{Z ; \expparam}.
\end{equation}
We start by rewriting the log-probability of $Y$ as the difference between the joint and the conditional log-probabilities:
\begin{equation}
\mathbb E_{Y \mid Z} \left[\nabla_\expparam \logprobnew{Y}{Y ; \expparam}\mid Z \right] = \mathbb E_{Y \mid Z} \left[\nabla_\expparam \logprobnew{Y, Z}{Y, Z ;\expparam} - \nabla_\expparam \logprobnew{Z \mid Y}{Z \mid Y ; \expparam} \mid Z \right].
\end{equation}

Next, we observe that the conditional log-probability $\logprobnew{Z \mid Y}{Z \mid Y ; \expparam}$ is concentrated in $Z = Z(Y)$ and equals $\log \mathbb I\left(Z = Z(Y)\right)$, which is zero for all $Y \sim \probnew{Y \mid Z}{\cdot \mid Z ; \expparam}$. Thus, the second summand cancels out:
\begin{equation}
\mathbb E_{Y \mid Z} \left[\nabla_\expparam \logprobnew{Y}{Y ; \expparam}\mid Z \right] = \mathbb E_{Y \mid Z} \left[\nabla_\expparam \logprobnew{Y, Z}{Y, Z ;\expparam} \mid Z \right].
\end{equation}

Then we rewrite the joint score as the sum of marginal $Z$-score and the conditional score of $Y$ given~$Z$:
\begin{equation}
\mathbb E_{Y \mid Z} \left[\nabla_\expparam \logprobnew{Y, Z}{Y, Z ;\expparam} \mid Z \right] = \mathbb E_{Y \mid Z} \left[\nabla_\expparam \logprobnew{Z}{Z ;\expparam} + \nabla_\expparam \logprobnew{Y \mid Z}{Y \mid Z ;\expparam}\right].
\end{equation}

Finally, the expectation of the score with respect to the corresponding distribution is zero:\footnote{Here we apply the log-derivative trick to $\logprobnew{Y \mid Z}{Y\mid Z; \expparam}$. While the trick may not apply to an arbitrary distribution, it is easy to show that the trick is correct for the conditional distributions we consider in our work.}
\begin{equation}
\mathbb E_{Y \mid Z} \left[\nabla_\expparam \logprobnew{Y \mid Z}{Y \mid Z ;\expparam}\mid Z \right] = \nabla_\expparam \mathbb E_{Y \mid Z} \left[ 1 \mid Z\right] = 0.
\end{equation}

At the same time, $Z$-score is already the function of $Z$ and can be moved outside of expectation:
\begin{equation}
\mathbb E_{Y \mid Z} \left[\nabla_\expparam \logprobnew{Y, Z}{Y, Z ;\expparam} \mid Z \right] = \mathbb E_{Y \mid Z} \left[\nabla_\expparam \logprobnew{Z}{Z ;\expparam}\mid Z \right] + 0 = \nabla_\expparam \logprobnew{Z}{Z ;\expparam}.
\end{equation}

Combining all previous steps, we arrive at:
\begin{equation}
\mathbb E_{Y \mid Z} \left[\loss{Z}\nabla_\expparam \logprobnew{Y}{Y ; \expparam}\mid Z \right] = \loss{Z}\nabla_\expparam \logprobnew{Z}{Z ; \expparam}.
\end{equation}
\end{proof}

The fact above also generalizes to the case of REINFORCE with baselines:

\begin{corollary}
Let $Y$ and $Z$ be random variables defined as in Lemma~\ref{lemma:rao} and $C$ be a random variable, independent of $Y$. Then, $Z$-REINFORCE with baseline $C$ is the marginalization of $Y$-REINFORCE with baseline $C$ with respect to the distribution $Y \mid Z, C$:

\begin{equation}
\mathbb E_{Y \mid Z, C} \left[ \left(\loss{Z} - C \right)\nabla_\expparam \logprobnew{Y}{Y ; \expparam} \mid Z, C\right] = \left(\loss{Z} - C\right)\nabla_\expparam \logprobnew{Z}{Z; \expparam}.
\end{equation}
\label{corollary:rao}
\end{corollary}
\begin{proof}
Start by observing that the loss part $\left(\loss{Z} - C\right)$ is a function of $Z$ and $C$; it can be moved outside of the conditional expectation:
\begin{equation}
\mathbb E_{Y \mid Z, C} \left[ \left(\loss{Z} - C \right)\nabla_\expparam \logprobnew{Y}{Y ; \expparam} \mid Z, C\right] = \left(\loss{Z} - C \right)\mathbb E_{Y \mid Z, C} \left[\nabla_\expparam \logprobnew{Y}{Y ; \expparam} \mid Z, C\right].
\end{equation}

Next, since $Y$ and $C$ are independent and $Z$ is a deterministic function of $Y$, the whole random vector $(Y, Z)$ is independent from $C$. Given this, one can remove the conditioning on $C$:
\begin{equation}
\left(\loss{Z} - C \right)\mathbb E_{Y \mid Z, C} \left[\nabla_\expparam \logprobnew{Y}{Y ; \expparam} \mid Z, C\right] = \left(\loss{Z} - C \right)\mathbb E_{Y \mid Z} \left[\nabla_\expparam \logprobnew{Y}{Y ; \expparam} \mid Z\right].
\end{equation}

Finally, Lemma~\ref{lemma:rao} states that the remaining conditional expectation is just the $Z$-score:
\begin{equation}
\left(\loss{Z} - C \right)\mathbb E_{Y \mid Z} \left[\nabla_\expparam \logprobnew{Y}{Y ; \expparam} \mid Z\right] = \left(\loss{Z} - C \right)\nabla_\expparam \logprobnew{Z}{Z ; \expparam}.
\end{equation}
\end{proof}

To draw the connection between the above statements and the estimators, defined in Section~\ref{sec:estimation}, we observe that the execution trace variable $\tracerv$ is the deterministic function of algorithm's input, exponential random vector $\exprv$, and the execution trace is defined in such a way, that the discrete random variable $\structrv$ can be recovered from it, i.e. $\structrv$ is a deterministic function of $\tracerv$, which is true for the algorithms we take into consideration (see Theorem~\ref{exp-min-thm}). Given this, we formulate the following lemma:

\begin{lemma}
Let $\exprv$, $\tracerv$ and $\structrv$ be random variables defined as in section \ref{sec:estimation}, i.e. $\exprv$ is the exponential random vector with (multidimensional) parameter $\expparam$, $\tracerv = \tracerv(\exprv)$ is the function of $\exprv$ (execution trace) and $\structrv = \structrv(\tracerv)$ is the function of $\tracerv$ (structured variable). For each of the variables define the corresponding REINFORCE estimator with baseline random variable $C$, independent of $\exprv$:
\begin{align}
g_{\exprv} &= \left(\loss{\structrv} - C \right)\nabla_\expparam \logprobnew{\exprv}{\exprv; \expparam}, \\
g_{\tracerv} &= \left(\loss{\structrv} - C \right)\nabla_\expparam \logprobnew{\tracerv}{\tracerv; \expparam}, \\
g_{\structrv} &= \left(\loss{\structrv} - C \right)\nabla_\expparam \logprobnew{\structrv}{\structrv; \expparam}.
\end{align}
Then, all the defined gradients are unbiased estimates of the true gradient:
\begin{equation}
\mathbb E_{\exprv, C}[g_\exprv] = \mathbb E_{\tracerv, C}[g_\tracerv] = \mathbb E_{\structrv, C}[g_\structrv] = \nabla_\expparam \mathbb E_X \loss{X}.
\end{equation}
With the following inequality between their variances:
\begin{equation}
\var{\structrv, C}{g_\structrv} \leq \var{\tracerv, C}{g_\tracerv} \leq \var{\exprv, C}{g_\exprv}.
\end{equation}
\end{lemma}
\begin{proof}
We first observe that using change of variables theorem for Lebesgue integral (particularly, law of the unconscious statistician) one can rewrite the true gradient using expectation with respect to $\exprv$ and then apply the log-derivative trick:
\begin{equation}
\nabla_\expparam \mathbb E_{\structrv} \loss{X} = \nabla_\expparam \mathbb E_{\exprv} \loss{X} = \mathbb E_{\exprv} \loss{X} \nabla_\expparam \logprobnew{\exprv}{\exprv; \expparam}.
\end{equation}

Since the baseline $C$ is independent from $\exprv$, subtracting it from the loss does not change the expectation:
\begin{equation}
\mathbb E_{\exprv, C} C \cdot \nabla_\expparam \logprobnew{\exprv}{\exprv; \expparam} = \mathbb E_{C} C \cdot \mathbb E_{\exprv}\nabla_\expparam \logprobnew{\exprv}{\exprv; \expparam} = \mathbb E_{C} C \cdot \nabla_\expparam \mathbb E_{\exprv} 1 = 0.
\end{equation}
Thus, gradient estimator with respect to $\exprv$ is unbiased:
\begin{equation}
\mathbb E_{\exprv, C}[g_\exprv] = \mathbb E_{\exprv, C} \left(\loss{X} - C \right) \nabla_\expparam \logprobnew{\exprv}{\exprv; \expparam} = \mathbb E_{\exprv} \loss{X} \nabla_\expparam \logprobnew{\exprv}{\exprv; \expparam} = \nabla_\expparam \mathbb E_{\structrv} \loss{X}.
\end{equation}
Next, we use the fact that $\tracerv$ is the function of $\exprv$ and $\structrv$ is the function of $\tracerv$. Corollary \ref{corollary:rao} then states that the estimates $g_\tracerv$ and $g_\structrv$ can be obtained by applying two sequential conditional marginalizations:
\begin{equation}
g_\tracerv = \mathbb E_{\exprv \mid \tracerv, C} \left[g_\exprv \mid \tracerv, C \right], \:\:g_\structrv = \mathbb E_{\tracerv \mid \structrv, C} \left[g_\tracerv \mid \structrv, C \right].
\end{equation}
Since the conditional expectation preserves the expectation, the above two gradients are also unbiased:
\begin{equation}
\mathbb E_{\tracerv, C}\left[g_\tracerv\right] = \mathbb E_{\structrv, C}\left[g_\structrv\right] = \nabla_\expparam \mathbb E_{\structrv} \loss{\structrv}
\end{equation}
Finally, conditional expectation reduces variance. For the trace random variable:
\begin{align*}
\var{\exprv, C}{g_\exprv} &= \mathbb E_{\tracerv, C} \var{\exprv \mid \tracerv, C}{g_\exprv \mid \tracerv, C} + \var{\tracerv, C}{\mathbb E_{\exprv \mid \tracerv, C} \left[g_\exprv \mid \tracerv, C \right]} = \\
&= \mathbb E_{\tracerv, C} \var{\exprv \mid \tracerv, C}{g_\exprv \mid \tracerv, C} + \var{\tracerv, C}{g_\tracerv}
\end{align*}

Observing that the conditional variance $\var{\exprv \mid \tracerv, C}{g_\exprv \mid \tracerv, C}$ is non-negative, we get:
\begin{equation}
\var{\exprv, C}{g_\exprv} \geq \var{\tracerv, C}{g_\tracerv}.
\end{equation}
Applying the same reasoning for $g_\structrv$, we obtain:
\begin{equation}
\var{\tracerv, C}{g_\tracerv} \geq \var{\structrv, C}{g_\structrv}.
\end{equation}
\end{proof}

\subsection{Distributions of \texorpdfstring{$\tracerv$}{Lg} and \texorpdfstring{$\exprv \mid \tracerv$}{Lg}}
Below we prove the main claim of this work. To simplify the argument, we will introduce an additional notation to denote the variables on various recursion depths. We will assume that the first level of recursion has index $j=1$. We will denote the input variables as $\exprv^1, K^1, R^1$ and as $\exprv, K, R$ interchangeeably.

Similarly, we will add the depth index to the partition, the trace variables and the updated exponential variables
\begin{align}
P^j_1, \dots, P^j_{m_j} &:= f_{\text{split}}(K^j, R^j) \\
\tracerv^j_i &:= \arg\min_{k \in P^j_i} \exprv^j_k \\
\exprv^{j+1}_k &:=
\begin{cases} 
\exprv^j_k - \exprv^j_{\tracerv^j_{i(k, j)}}, \:\:\text{if} \:\: k \in K^{j}\\
\exprv^j_k, \:\:\text{otherwise},
\end{cases}
\end{align}
where the function $i(k, j)$ returns the index of the partition set containing key $k$ on depth $j$. To update the index of the auxiliary variables $K$ and $R$ we use the $f_{\text{map}}(\cdot)$ subroutine:
\begin{equation}
K^{j+1}, R^{j+1} := f_{\text{map}}(K^j, R^j, \{\tracerv^j_i\}_{i=1}^{m_j}).
\end{equation}
Finally, we will use $\expparam_{k}^{j}$ to denote the updated parameter $\expparam^1_k := \expparam_{k}$ after transformations at the depth $j - 1$. It is equal to the $\expparam_{k}$, if $k$ was not argminimum at all recursion depths from $1$ to $j - 1$ or is equal to $+\infty$ otherwise. Given the above notation we formulate
\begin{theorem}
\label{exp-min-thm}
Let~$\exprv$ be a set of exponential random variables~$\exprv_k \sim \operatorname{Exp}{(\expparam_k)}$ indexed by~$k \in K$, let $\tracerv$ be the trace and $\structrv$ be the output of Algorithm~\ref{alg:general} respectively. If the subroutine $f_{\text{stop}}$ is defined in such a way that
$f_{\text{stop}}\left(\emptyset, R\right) = 1$ for all possible auxiliary variables $R$, then:
\begin{enumerate}
\item The output $\structrv$ is determined by the trace~$\tracerv$ and the auxiliary variables $K, R$: $X(E, T, K, R) = X(T, K, R)$;\footnote{In the paper we omit the dependence on $K$ and $R$ since these are typically fixed.}
\item The trace~$\tracerv$ is a sequence of categorical latent variables $\tracerv_{i}^{j}$. Each $\tracerv_{i}^{j}$ is defined at the recursion level $j$ and is either deterministic or has (conditional) outcome probabilities proportional to $\expparam_k$ for $k \in P_{i}^{j}$;
\item The elements of the conditional distribution $\exprv \mid \tracerv$ are distributed as a sum of exponential random variables
\begin{equation}
\exprv_k \mid T \sim \sum_{j = 1}^{N(k)} \operatorname{Exp}{\left(\sum_{k' \in P_{i(k, j)}^{j}} \expparam_{k'}^{j}\right)} + \operatorname{Exp}{\left(\expparam_{k}^{N(k) + 1}\right)},
\end{equation}
where $N(k)$ is the deepest recursion level $j$ such that the index $k$ is contained in the index set $K^j$.
\end{enumerate}
\end{theorem}
\begin{proof}
First, we note that the recursion depth of Algorithm~\ref{alg:general} is limited. Indeed, when we make a recursive call, the second argument of $F_\text{struct}$ is the index set $K'$, which, by construction, is a strict subset of the finite index set $K$ from the previous call of the function. It means that either $f_\text{stop}(K, R)$ is true at some point, or there exists a stage where we call $F_\text{struct}$ with $K = \emptyset$. By the assumption, the stop condition is true for the empty set $f_\text{stop}(\emptyset, R) = 1$, which means that we always reach the last level of recursion. Next, we will prove the first claim of the theorem by induction on the recursion depth.

We start by considering an input~$\exprv^1, K^1, R^1$ for which the algorithm has recursion depth $N=1$. For a single recursion layer Algorithm~\ref{alg:general} checks the stop condition and halts. The stop condition is a function of $K^1$ and $R^1$, therefore the output does not depend on $\exprv$. When an input leads to an arbitrary recursion depth $N$, the algorithm returns $f_{\text{combine}}(\structrv^2, K^1, R^1, \tracerv^1)$. We will argue that all of the arguments of the function depend on $\exprv^1$ only through $\tracerv$, therefore the algorithm output can be represented as a function of $\tracerv$ rather than $\exprv^1$.

By the induction hypothesis, the intermediate structure~$\structrv^2 = F_{\text{struct}}(\exprv^2, K^2, R^2)$ is determined by~$\tracerv^{>1}$, $K^2$ and $R^2$. To obtain $K^2,R^2 = f_{\text{map}}(K^1, R^1, \tracerv^1)$ we do not take $\exprv^1$ as an input, therefore $\structrv^2$ is determied by~$\tracerv^1$. The arguments $K^1$ and $R^1$ of $f_{\text{combine}}$ also do not depend on $\exprv^1$. The last argument~$\tracerv^1$ depends on $\exprv^1$, but for a given $\tracerv^1$ the variability in $\exprv^1$ does not change the output of $f_{\text{combine}}$. Therefore, $\structrv^1$ is determined by the trace~$\tracerv$ and the auxiliary inputs $K^1$ and $R^1$.

To prove the second and the third claim, we repeat the derivation of Lemma~\ref{lemma:exp-min-appendix} for the joint density of $\exprv$ and $\tracerv$. We denote the realizations of $\tracerv$ and $\exprv$ with lower-case letters. Here, variables of the form $\expreal_k^j$ are defined analogously to $\exprv_k^j$, but for the corresponding realizations. Recall that $\tracerv = \{\tracerv_i^j\}_{i,j}$ is the concatenation of trace variables $\tracerv_1^j, \dots, \tracerv^j_{m_j}$ for all recursion depths $j=1, \dots, k$. We regroup the joint density using the partition $P^1_1, \dots P^1_{m_1}$ of $K = \sqcup P^1_i$ and splitting the indicator, corresponding to the conditional distribution $\tracerv \mid \exprv$, into two indicators, corresponding to the first level of recursion and to the remaining ones respectively:
\begin{align}
\mathbb I(\tracereal = \tracerv(\expreal)) \prod_{k \in K} \operatorname{Exp}{\left(\expreal_k \mid \expparam_k \right)}
&= \mathbb I(\tracereal = \tracerv(\expreal)) \prod_{k \in K} \expparam_k \expreal^{-\expparam_k \expreal_k}\\
&= \mathbb I(\tracereal^{>1} = \tracerv^{>1}(\expreal)) \cdot \mathbb I(\tracereal^1 = \tracerv^1(\expreal)) \cdot \prod_{k \in K} \expparam_k \expreal^{-\expparam_k \expreal_k} = \\
&= \mathbb I(\tracereal^{>1} = \tracerv^{>1}(\expreal)) \cdot \prod\limits_{i=1}^{m_1} \mathbb \prod_{k \in P^1_i} \expparam_k \expreal^{-\expparam_k \expreal_k} \mathbb I(\expreal_{\tracereal^1_i} \leq \expreal_k).
\end{align}

Then we apply Lemma~\ref{lemma:exp-min-appendix} for each $i=1,\dots,m_1$ and rewrite the internal product $\prod_{k\in P^1_i}$ as
\begin{equation}
\label{eq:exp-min-trick-thm}
\frac{\expparam_{\tracereal^1_i}}{\sum_{k \in P^1_i} \expparam_k} \cdot
\left( \sum_{k \in P^1_i} \expparam_k \right) \expreal^{-\left(\sum_{k \in P_i} \expparam_k\right) \expreal_{\tracereal^1_i}} \cdot
\prod_{k \in P^1_i } \left( \expparam^2_k \expreal^{-\expparam^2_k \expreal^2_k} \mathbb I(0 \leq \expreal^2_k) \right),
\end{equation}
where $\expreal^2_k$ is the realization of the random variable $\exprv^2_k := \exprv_k - \exprv_{\tracerv^1_i}$ and $\expparam^2_k := \expparam_k$ for $k \neq \tracereal^1_i$ and $\expparam^2_{\tracereal^1_i} := +\infty$.

To improve readability, we rewrite the same product as
\begin{equation}
\frac{\expparam_{\tracereal^1_i}}{\sum_{k \in P^1_i} \expparam_k} \cdot
\operatorname{Exp}{\left(\expreal_{\tracereal_i^1} \Big| \sum\limits_{k \in P^1_i}\expparam_k\right)} \cdot
\prod_{k \in P^1_i }\operatorname{Exp}{\left(\expreal^2_k \mid \expparam^2_k \right)}
\end{equation}

and substitute it into the overall joint density:
\begin{equation}
\mathbb I(\tracereal^{>1} = \tracerv^{>1}(\expreal)) \cdot \prod_{i=1}^{m_1} 
\frac{\expparam_{\tracereal^1_i}}{\sum_{k \in P^1_i} \expparam_k} \cdot
\operatorname{Exp}{\left(\expreal_{\tracereal_i^1} \Big| \sum\limits_{k \in P^1_i}\expparam_k\right)} \cdot
\prod_{k \in P^1_i }\operatorname{Exp}{\left(\expreal^2_k \mid \expparam^2_k \right)}.
\end{equation}

Since the execution trace variables at depths $> 1$ are determined by the transformed values $\expreal^2 = \{\expreal^2_k \mid k \in K\}$, the indicator can be rewritten as $\mathbb I(\tracereal^{>1} = \tracerv^{>1}(\expreal^2))$. At the same time, $\prod_{i=1}^{m_1}\prod_{k \in P^1_i }\operatorname{Exp}{\left(\expreal^2_k \mid \expparam^2_k \right)}$ can be rewritten as just $\prod\limits_{k \in K}\operatorname{Exp}{\left(\expreal^2_k \mid \expparam^2_k \right)}$, since $P^1_1, \dots P^1_{m_1}$ is the partition of $K$. Given this, we rewrite the overall joint density one more time as
\begin{equation}
\prod_{i=1}^{m_1}\left[
\frac{\expparam_{\tracereal^1_i}}{\sum_{k \in P^1_i} \expparam_k} \cdot
\operatorname{Exp}{\left(\expreal_{\tracereal_i^1} \Big| \sum\limits_{k \in P^1_i}\expparam_k\right)}\right] \cdot
\mathbb I(\tracereal^{>1} = \tracerv^{>1}(\expreal^2)) \cdot
\prod\limits_{k \in K}\operatorname{Exp}{\left(\expreal^2_k \mid \expparam^2_k \right)},
\end{equation}
where the last two terms have the same form as the joint density, written in the beginning, but for the execution trace variables $\tracerv^{>1}$ and transformed exponential variables $\exprv^2$. 

We use this observation and apply the same transformations to the density
\begin{equation}
\mathbb I(\tracereal^{>1} = \tracerv^{>1}(\expreal^2)) \cdot \prod\limits_{k \in K}\operatorname{Exp}{\left(\expreal^2_k \mid \expparam^2_k \right)}
\end{equation}

based on the partition $P^j_1, \dots, P^j_{m_j}$ of $K^j \subset K$ for the next recursion steps. We apply the transformations until we reach the bottom of the recursion. By design, the algorithm excludes some of the indices from consideration $K^{j+1} \subsetneq K^{j} \subseteq K$. According to our notation, the variables excluded from the index set on a certain depth $j$ stay unchanged along with parameters, i.e. $\exprv^j_k = \exprv^{j\text{’}}_k$, $\expparam^j_k = \expparam^{j\text{’}}_k$ for $j\text{’} \geq j$. Such notation allows to preserve the product across all keys $\prod_{k \in K} \operatorname{Exp}(\expreal^{j}_k \mid \expparam^j_k)$ throughout the recursion.

After performing all transformations at recursion depths $j$ from $2$ to $N$ we arrive at the following representation of the joint density:
\begin{equation}
\prod\limits_{j = 1}^{N}\prod_{i=1}^{m_j}\left[
\frac{\expparam^j_{\tracereal^j_i}}{\sum_{k \in P^j_i} \expparam^j_k} \cdot
\operatorname{Exp}{\left(\expreal^j_{\tracereal_i^j} \Big| \sum\limits_{k \in P^j_i}\expparam^j_k\right)}\right] \cdot
\prod\limits_{k \in K}\operatorname{Exp}{\left(\expreal^{N + 1}_k \mid \expparam^{N + 1}_k \right)}.
\end{equation}

For each $k$ we observe one more time that $\expreal_k^{j}$ and $\expparam_k^{j}$ do not change after $j = N(k) + 1$, since $k$ is excluded from all the corresponding $K^j$. Given this, we rewrite the latter product as
\begin{equation}
\prod\limits_{k \in K}\operatorname{Exp}{\left(\expreal^{N(k) + 1}_k \mid \expparam^{N(k) + 1}_k \right)}.
\end{equation}

Finally, we recursively apply the definition of $\expreal^j_k$ to represent it as a function of the initial variable $\expreal^1_k = \expreal_k$ and the set of minima $\expreal^{j'}_{\tracereal_{i(k, j')}^{j'}}$, obtained at recursion depths $j'$ from $1$ to $j - 1$. For each depth $j$:
\begin{equation}
\expreal^{j + 1}_k = \expreal^{j}_k - \expreal^{j}_{\tracereal^{j}_{i(k, j)}} = \expreal^{j - 1}_k - \expreal^{j}_{\tracereal^{j}_{i(k, j)}} - \expreal^{j - 1}_{\tracereal^{j - 1}_{i(k, j - 1)}} = \ldots = \expreal_k - \sum\limits_{j' = 1}^{j}\expreal^{j'}_{\tracereal^{j'}_{i(k, j')}}.
\end{equation}

Applying the same observation for $j = N(k)$, we obtain:
\begin{equation}
\expreal^{N(k) + 1}_k = \expreal_k - \sum\limits_{j = 1}^{N(k)}\expreal^{j}_{\tracereal^{j}_{i(k, j)}},
\end{equation}
which leads to the final representation of the joint density:


\begin{equation}
\underbrace{\left[\prod\limits_{j = 1}^{N}\prod_{i=1}^{m_j}
\frac{\expparam^j_{\tracereal^j_i}}{\sum_{k \in P^j_i} \expparam^j_k}\right]}_{\probnew{\tracerv}{\tracereal ; \expparam}} \cdot
\underbrace{\left[\prod\limits_{j = 1}^{N}\prod_{i=1}^{m_j} \operatorname{Exp}{\left(\expreal^j_{\tracereal_i^j} \Big| \sum\limits_{k \in P^j_i}\expparam^j_k\right)}\right]}_{\probnew{\exprv_\tracerv \mid \tracerv}{\expreal_\tracereal \mid \tracereal ; \expparam}} \cdot
\underbrace{\prod\limits_{k \in K}\operatorname{Exp}{\left(\expreal_k - \sum\limits_{j = 1}^{N(k)}\expreal^{j}_{\tracereal^{j}_{i(k, j)}} \Bigg | \expparam^{N(k) + 1}_k \right)}}_{\probnew{\exprv \mid \tracerv, \exprv_\tracerv}{\expreal \mid \tracereal, \expreal_\tracereal ; \expparam}}.
\end{equation}

This representation defines the following generation process:
\begin{itemize}
\item First, the trace variables (argminima) are generated from $\probnew{\tracerv}{\tracereal ; \expparam}$, the marginal probability of $\tracerv$, represented as the product of conditional probabilities of $\tracerv^j_i$;
\item Second, the corresponding minima for all partition indices $i$ at all recursion depths $j$ are sampled from $\probnew{\exprv_\tracerv \mid \tracerv}{\expreal_\tracereal \mid \tracereal ; \expparam}$;
\item Finally, the set of exponential random variables $\exprv$ is obtained by sampling from $\probnew{\exprv \mid \tracerv, \exprv_\tracerv}{\expreal \mid \tracereal, \expreal_\tracereal ; \expparam}$. All the realizations $\expreal_k$ here come from the exponential distribution with parameter $\expparam_k^{N(k) + 1}$, shifted at the value $\sum\limits_{j = 1}^{N(k)}\expreal^{j}_{\tracereal^{j}_{i(k, j)}}$.
\end{itemize}

Note that we have started from the joint distribution on $\exprv, \tracerv$ and come to the joint distribution on $\tracerv, \exprv_\tracerv, \exprv$. We obtained larger set of variables, however, as initially, only $|K|$ of them are non-degenerate. This comes from the definition of $\expparam^{j}_k$ and the observation that at each step of taking minimum we either find a constant zero, which does not change the number of non-constant variables, or find a non-degenerate value, introduce a new (non-degenerate) exponential variable, corresponding to the minimum, and replace the corresponding $\expparam^{j}_k$ with $+\infty$. The latter corresponds to setting one of the variables to be constant, thus, the overall number of non-degenerate distributions does not change when we perform the above transformations with density.

The first item above proves the claim about the distribution of trace variables. The second tells that each minimum realization $\expreal^{j}_{\tracereal^j_{i(k, j)}}$ comes from the distribution $\operatorname{Exp}{\left(\sum\limits_{k' \in P^j_{i(k, j)}}\expparam^j_{k'}\right)}$. Combined with the third one, it proves that the conditional distribution of each $\exprv_k$ is the sum of the corresponding exponential distributions, claimed in the thorem:
\begin{equation}
\exprv_k \mid T \sim \sum_{j = 1}^{N(k)} \operatorname{Exp}{\left(\sum_{k' \in P_{i(k, j)}^{j}} \expparam_{k'}^{j}\right)} + \operatorname{Exp}{\left(\expparam_{k}^{N(k) + 1}\right)}.
\end{equation}


\end{proof}

Based on the above derivation, we propose a procedure to compute the log-probability of the trace and to draw the conditional sample. 

To compute the log-probability, we compute the log-probabilites of the top trace level $\{\tracereal^1_i\}_i$ as in Eq.~\ref{eq:exp-min-trick-thm}. Then we repeat the exp-min trick as in the above derivation and repeat the procedure. Assume the induction hypothesis that the procedure computes the log-prob of the rest of the trace. Then, by induction, we obtain the log-prob of the whole trace as a sum of the log-prob of the top trace $\{\tracereal^1_i\}_i$ and the rest of the trace $\{\tracereal^{>1}\}_j$. 

Similarly, assume we have a procedure to draw $\expreal_k', k \in K'$. At the bottom of the recursion $\expreal_k'$ are just exponential random variables. For the induction step, we draw $\exprv_k$ for $k \notin K \setminus K'$ and definee $\expreal_k := \expreal_k' + \expreal_{\tracereal_i}$, $k \in P_i$. In the next section, we provide the pseudo-code for the two procedures.

\begin{algorithm}[tbh]
\caption{$F_{\text{struct}}(\exprv, K, R)$ - returns structured variable $\structrv$ based on utilities $\exprv$ and auxiliary variables $K$ and $R$}
\begin{algorithmic}
    \REQUIRE $\exprv, K, R$
    \ENSURE $\structrv$
    \IF{$f_{\text{stop}}(K, R)$}
    \STATE {\bf return}
    \ENDIF
    \STATE $P_1, \dots, P_m \Leftarrow f_{\text{split}}(K, R)$ \hfill \COMMENT{$\sqcup_{i=1}^m P_i = K$}
    \FOR{$i=1$ to $m$}
    \STATE $\tracerv_i \Leftarrow \arg\min_{k \in P_i} \exprv_k$
    \FOR{$k \in P_i$}
    \STATE $\exprv_k' \Leftarrow \exprv_k - \exprv_{\tracerv_i}$
    \ENDFOR
    \ENDFOR
    \STATE $K', R' \Leftarrow f_{\text{map}}(K, R, \{\tracerv_i\}_{i=1}^m)$ \hfill \COMMENT{$K' \subsetneq K$}
    \STATE $E' \Leftarrow \{E_k' \mid k \in K'\}$
    \STATE $\structrv' \Leftarrow F_{\text{struct}}(E', K', R')$ \hfill \COMMENT{Recursive call}
    \STATE {\bf return} $f_{\text{combine}}(\structrv', K, R, \{\tracerv_i\}_{i=1}^m)$
\end{algorithmic}
\end{algorithm}

\begin{algorithm}[ht]
\caption{$F_{\text{log-prob}}(\tracereal, \expparam, K, R)$ - returns $\logprobnew{\tracerv}{\tracereal; \expparam}$ for trace $\tracereal$, rates $\expparam$, $K$ and $R$ as in Alg.~\ref{alg:general}}
\label{alg:log-prob-appendix}
\begin{algorithmic}
    \REQUIRE $\tracereal, \expparam, K, R$
    \ENSURE $\logprobnew{\tracerv}{\tracereal; \expparam}$
    \IF{$f_{\text{stop}}(K, R)$}
    \STATE {\bf return}
    \ENDIF
    \STATE $P_1, \dots, P_m \Leftarrow f_{\text{split}}(K, R)$ 
    \FOR{$i=1$ to $m$}
    \STATE $\logprobnew{\tracerv}{\tracereal^1_i; \expparam} \Leftarrow \log \expparam_{\tracereal^1_i} - \log \left(\sum_{k \in P_i} \expparam_k \right)$ \hfill \COMMENT{Index $j$ in $T^j_i$ denotes the recursion level}
    \FOR{$k \in P_i \setminus \{\tracereal^1_i\}$ }
    \STATE $\expparam'_{k} \Leftarrow \expparam_k$
    \ENDFOR
    \STATE $\expparam'_{\tracereal^1_i} \Leftarrow +\infty$ \hfill \COMMENT{Because $\exprv'(\tracereal^1_i) = 0$}
    \ENDFOR
    \STATE $K', R' \Leftarrow f_{\text{map}}(K, R, \{\tracereal^1_i\}_{i=1}^m)$
    \STATE $\expparam' \Leftarrow \{ \expparam_k' \mid k \in K'\}$
    \STATE $\logprobnew{\tracerv}{\tracereal^{>1} \mid \tracerv^{1} = \tracereal^{1}; \expparam} \Leftarrow F_{\text{log-prob}}(\tracereal^{>1}, \expparam', K', R')$ \hfill \COMMENT{Compute log-prob of $\tracereal^{>1} := \{\tracereal^j_i\}_{j > 1}$}
    \STATE {\bf return} $ \sum_{i=1}^m \logprobnew{\tracerv}{\tracereal^1_i; \expparam} + \logprobnew{\tracerv}{\tracereal^{>1} \mid \tracerv^1 = \tracereal^1 ; \expparam}$
\end{algorithmic}
\end{algorithm}

\begin{algorithm}[t]
\caption{$F_{\text{cond}}(\tracereal, \expparam, K, R)$ - returns a utility sample from $\exprv \mid \tracerv = \tracereal, \expparam$ with rates $\expparam$ conditioned on the execution trace $\tracereal = \{ \tracereal^j_i \}_{ij}$}
\label{alg:conditional}
\begin{algorithmic}
    \REQUIRE $\tracereal, \expparam, K, R$
    \ENSURE $\expreal$
    \IF{$f_{\text{stop}}(K, R)$}
    \STATE {\bf return}
    \ENDIF
    \STATE $P_1, \dots, P_m \Leftarrow f_{\text{split}}(K, R)$ 
    \FOR{$i=1$ to $m$}
    \STATE $\expreal_{\tracereal^1_i} \sim \operatorname{Exp}{(\sum_{k \in P_i} \expparam_k)}$ \hfill \COMMENT{Sample the $\min$}
    \FOR{$k \in P_i \setminus \{\tracereal^1_i\}$ }
    \STATE $\expparam'_{k} \Leftarrow \expparam_k$
    \ENDFOR
    \STATE $\expparam'_{\tracereal^1_i} \Leftarrow +\infty$ \hfill \COMMENT{Because $\expreal_{\tracereal^1_i}' = 0$}
    \ENDFOR
    \STATE $K', R' \Leftarrow f_{\text{map}}(K, R, \{\tracereal^1_i\}_{i=1}^m)$
    \STATE $\expparam' \Leftarrow \{\expparam_k \mid k \in K'\}$
    \STATE $\expreal' \Leftarrow F_{\text{cond}}(\tracereal^{>1}, \expparam', K', R')$ \hfill \COMMENT{Recursion, returns random variables indexed with $K'$}
    \FOR{$k \in K \setminus K'$}
    \STATE $\expreal_k' \sim \operatorname{Exp}{(\expparam_k')}$ 
    \hfill \COMMENT{Sample the rest of the utilities}
    \ENDFOR
    \FOR{$i=1$ to $m$}
    \FOR{$k \in P_i \setminus \{\tracereal^1_i\}$}
    \STATE $\expreal_k \Leftarrow \expreal_k' + \expreal_{\tracereal^1_i}$
    \hfill \COMMENT{Reverse the Exponential-Min trick}
    \ENDFOR
    \ENDFOR
    \STATE {\bf return} $\expreal$
\end{algorithmic}
\end{algorithm}

\newpage

\section{General Algorithms for Log-Probability and Conditional Sampling}
\label{sec:algorithms}
We provide pseudo-code for computing $\logprob{\tracerv; \expparam}$ in Algorithm~\ref{alg:log-prob} and sampling $\exprv \mid \tracerv$ in Algorithm~\ref{alg:conditional}. Both algorithms modify Algorithm~\ref{alg:general} and use the same subroutines $f_{\text{stop}}, f_{\text{split}}, f_{\text{map}},$ and $f_{\text{combine}}$. Algorithms~\ref{alg:log-prob}, \ref{alg:conditional} follow the structure as Algorithm~\ref{alg:general} and have at most linear overhead in time and memory for processing variables such as $\expparam'$ and $\logprob{T^j_i \mid \expparam}$.

The indexed set of exponential random variables $\exprv$ and the indexed set of the random variable parameters $\expparam$ have the same indices of indices $K$, which allows to call subroutines in the same way as in Algorithm~\ref{alg:general}.

Both algorithms take the trace variable $\tracereal = \{\tracereal^j_i\}_{j,i}$ as input. Note that index $j$ enumerate recursion levels. Both algorithms process the trace of the top recursion level $\tracereal^1_1, \dots, \tracereal^1_m$ and make a recursive call to process the subsequent trace~$\tracereal^{>1} := \{\tracereal^j_i\}_{i, j > 1}$.

\section{Implementation Details}
\label{sec:details}
In the paper, we chose the exponential random variables and the recursive form of Algorithm~\ref{alg:general} to simplify the notation. In practice, we parameterized the rate of the exponential distributions as $\expparam = \exp(-\gumbparam)$, where $\gumbparam$ was either a parameter or an output of a neural network. The parameter $\theta$ is essentially the location parameter of the Gumbel distribution and, unlike $\expparam$, can take any value in $\mathbb R$.

Additionally, the recursive form of Algorithm~\ref{alg:general} does not facilitate parallel batch computation. In particular, the recursion depth and the decrease in size of $\exprv$ may be different for different objects in the batch. Therefore, Algorithms~\ref{alg:log-prob},\ref{alg:conditional} may require further optimization.

For the top-k algorithm, we implemented the parallel batch version. To keep the input size the same, we masked the omitted random variables with $+\infty$. We modeled the recursion using an auxiliary tensor dimension.

For the Kruskal's algorithm, we implemented the parallel batch version and used the $+\infty$ masks to preserve the set size. We rewrote the recursion as a Python for loop.

To avoid the computation overhead for the Chu-Liu-Edmonds algorithm, we implemented the algorithms in C++ and processed the batch items one-by-one.

For the binary trees Algorithm \ref{alg:general} was implemented in C++ and processed the batch items one-by-one, while Algorithms \ref{alg:log-prob}, \ref{alg:conditional} utilize efficient parallel implementation.

Also, during optimization using RELAX gradient estimator we observed the following behaviour: sometimes $\exprv \mid \tracerv=t$ generates samples which do not lead to $t$ applying Algorithm \ref{alg:general}. Such behaviour occurs due to the usage of \textit{float} precision and does not show using \textit{double} precision. While it may be considered as a drawback, its worth noting that it occurs very rare (less than 0.1 \% of all conditional samples produced during optimization) and does not affect overall optimization procedure.


\section{Experimental Details}
\label{sec:experiments}

Setting up the experiments with Top-K, Spanning Tree and Arborescence we followed details about data generation, models and training procedures, described by \cite{paulus2020gradient}, to make a valid comparison of the proposed score function methods with Stochastic Softmax Tricks (SSTs). In each experiment we fixed the number of function evaluations $N$ per iteration instead of batch size to make a more accurate comparison in terms of computational resources. With $N$ fixed, RELAX and SST were trained with batch size equal to $N$, while \exprv-REINFORCE+ and \tracerv-REINFORCE+ were trained with batch size $N/K$ and $K$ samples of the latent structure for each object. 

To get rid of the influence of any factors other than efficacy of the gradient estimator we fixed the same random model initialization. Then, for each gradient estimator we chose best model hyperparameter's set with respect to validation task metric (MSE, ELBO, accuracy). Given best model hyperparameter's set we report mean and standard deviations of the metrics across different random model initializations.

\subsection{Top-K and Beer Advocate}
\label{sec:topkexp}

\begin{table*}[t]
\centering
\caption{Results of $k$-subset selection on Appearance aspect data. MSE ($\times 10^{-2}$) and subset precision (\%) is shown for best models selected on validation averaged across different random seeds.}  
\label{table:l2x-app}
\begin{center}
\begin{small}
\begin{sc}
\scalebox{0.85}{
\begin{tabular}{@{}llllllll@{}}
\toprule
\multirow{2.5}{*}{\shortstack[l]{Model}}& \multirow{2.5}{*}{\shortstack[l]{Estimator}} & \multicolumn{2}{c}{$k=5$}  & \multicolumn{2}{c}{$k=10$} & \multicolumn{2}{c}{$k=15$}  \\ \cmidrule(r){3-4} \cmidrule(r){5-6} \cmidrule(r){7-8}
 &  & mean $\pm$ std & Prec. & mean $\pm$ std & Prec. & mean $\pm$ std & Prec. \\
\cmidrule[\heavyrulewidth]{1-8}
\multirow{4.5}{*}{\shortstack[l]{Simple}}
  & \emph{SST (Our Impl.)}
  	& $3.44 \pm 0.13$
  	& $43.3 \pm 4.5$
  	& $3.09 \pm 0.12$
  	& $45.7 \pm 3.6$
  	& $\mathbf{2.67 \pm 0.12}$
  	& $\mathbf{42.1 \pm 1.1}$ \\
  \cmidrule[0.15pt]{2-8}
    & \emph{\exprv-REINFORCE+} 
    & $3.74 \pm 0.11$
    & $38.8 \pm 2.9$
    & $3.46 \pm 0.12$
    & $33.2 \pm 3.6$
    & $3.24 \pm 0.15$
    & $31.2 \pm 3.4$ \\
  & \emph{\tracerv-REINFORCE+}
    & $3.57 \pm 0.11$
    & $\mathbf{48.9 \pm 2.5}$
    & $3.02 \pm 0.11$
    & $\mathbf{47 \pm 4.1}$
    & $2.69 \pm 0.06$
    & $41.6 \pm 2.2$ \\
  & \emph{RELAX} 
  	& $\mathbf{3.36 \pm 0.1}$
  	& $44.2 \pm 3.2$
  	& $\mathbf{3.01 \pm 0.08}$
  	& $42.4 \pm 2.7$
  	& $2.85 \pm 0.09$
  	& $40.7 \pm 1.8$ \\
  \cmidrule[\heavyrulewidth]{1-8}
  \multirow{4.5}{*}{\shortstack[l]{Complex}}
  & \emph{SST (Our Impl.)}
    & $2.96 \pm 1.1$
    & $73.2 \pm 5.3$
    & $2.61 \pm 0.09$
    & $71.9 \pm 3.3$
    & $2.57 \pm 0.08$
    & $65.6 \pm 2.9$ \\
  \cmidrule[0.15pt]{2-8}
    & \emph{\exprv-REINFORCE+} 
    & $3.25 \pm 0.11$
    & $72.9 \pm 6.1$
    & $2.9 \pm 0.19$
    & $63.1 \pm 1$
    & $2.63 \pm 0.13$
    & $63.3 \pm 0.5$ \\
  & \emph{\tracerv-REINFORCE+}
    & $\mathbf{2.65 \pm 0.05}$
    & $\mathbf{82.9 \pm 1.3}$
    & $\mathbf{2.48 \pm 0.05}$
    & $74.5 \pm 3.7$
    & $\mathbf{2.41 \pm 0.03}$
    & $\mathbf{68.3 \pm 2}$ \\
  & \emph{RELAX} 
    & $2.67 \pm 0.06$
    & $81.3 \pm 1.5$
    & $2.54 \pm 0.03$
    & $\mathbf{74.8 \pm 1.3}$
    & $2.51 \pm 0.03$
    & $67.1 \pm 2.1$ \\
\bottomrule
\end{tabular}
}
\end{sc}
\end{small}
\end{center}
\end{table*}
\begin{table*}[t]
\centering
\caption{Results of $k$-subset selection on Taste aspect data. MSE ($\times 10^{-2}$) and subset precision (\%) is shown for best models selected on validation averaged across different random seeds.}  
\label{table:l2x-taste}
\begin{center}
\begin{small}
\begin{sc}
\scalebox{0.85}{
\begin{tabular}{@{}llllllll@{}}
\toprule
\multirow{2.5}{*}{\shortstack[l]{Model}}& \multirow{2.5}{*}{\shortstack[l]{Estimator}} & \multicolumn{2}{c}{$k=5$}  & \multicolumn{2}{c}{$k=10$} & \multicolumn{2}{c}{$k=15$}  \\ \cmidrule(r){3-4} \cmidrule(r){5-6} \cmidrule(r){7-8}
 &  & mean $\pm$ std & Prec. & mean $\pm$ std & Prec. & mean $\pm$ std & Prec. \\
\cmidrule[\heavyrulewidth]{1-8}
\multirow{4.5}{*}{\shortstack[l]{Simple}}
  & \emph{SST (Our Impl.)}
  	& $\mathbf{3.19 \pm 0.16}$
  	& $26.7 \pm 2.5$
  	& $\mathbf{2.93 \pm 0.12}$
  	& $28 \pm 0.9$
  	& $\mathbf{2.89 \pm 0.04}$
  	& $28.7 \pm 1.3$ \\
  \cmidrule[0.15pt]{2-8}
    & \emph{\exprv-REINFORCE+} 
    & $3.6 \pm 0.4$
    & $23.6 \pm 2.6$
    & $3.51 \pm 0.36$
    & $21.4 \pm 2.2$
    & $3.12 \pm 0.16$
    & $24.6 \pm 3.2$\\
  & \emph{\tracerv-REINFORCE+}
    & $3.24 \pm 0.2$
    & $\mathbf{28.5 \pm 2.4}$
    & $3.07 \pm 0.05$
    & $\mathbf{28.5 \pm 1.4}$
    & $2.9 \pm 0.04$
    & $\mathbf{29.2 \pm 3.2}$\\
  & \emph{RELAX} 
  	& $3.26 \pm 0.08$
  	& $24 \pm 3.4$
  	& $3.13 \pm 0.09$
  	& $25.8 \pm 2.1$
  	& $2.95 \pm 0.09$
  	& $24.4 \pm 2.6$ \\
  \cmidrule[\heavyrulewidth]{1-8}
  \multirow{4.5}{*}{\shortstack[l]{Complex}}
  & \emph{SST (Our Impl.)}
    & $2.7 \pm 0.21$
    & $36.2 \pm 3.1$
    & $2.66 \pm 0.19$
    & $36 \pm 5.1$
    & $\mathbf{2.2 \pm 0.02}$
    & $\mathbf{43.2 \pm 1}$ \\
  \cmidrule[0.15pt]{2-8}
    & \emph{\exprv-REINFORCE+} 
    & $3.43 \pm 0.52$
    & $33.2 \pm 4.8$
    & $3.15 \pm 0.33$
    & $33 \pm 4.3$
    & $2.81 \pm 0.16$
    & $39.1 \pm 3$ \\
  & \emph{\tracerv-REINFORCE+}
    & $\mathbf{2.62 \pm 0.2}$
    & $\mathbf{40.2 \pm 2.4}$
    & $\mathbf{2.45 \pm 0.04}$
    & $\mathbf{40.6 \pm 2.6}$
    & $2.43 \pm 0.04$
    & $40.3 \pm 2.3$ \\
  & \emph{RELAX} 
    & $2.78 \pm 0.07$
    & $34.7 \pm 2.5$
    & $2.99 \pm 0.2$
    & $32.1 \pm 3.6$
    & $2.64 \pm 0.04$
    & $33.9 \pm 3.8$ \\
\bottomrule
\end{tabular}
}
\end{sc}
\end{small}
\end{center}
\end{table*}
\begin{table*}[!t]
\centering
\caption{Results of $k$-subset selection on Palate aspect data. MSE ($\times 10^{-2}$) and subset precision (\%) is shown for best models selected on validation averaged across different random seeds.}  
\label{table:l2x-pal}
\begin{center}
\begin{small}
\begin{sc}
\scalebox{0.85}{
\begin{tabular}{@{}llllllll@{}}
\toprule
\multirow{2.5}{*}{\shortstack[l]{Model}}& \multirow{2.5}{*}{\shortstack[l]{Estimator}} & \multicolumn{2}{c}{$k=5$}  & \multicolumn{2}{c}{$k=10$} & \multicolumn{2}{c}{$k=15$}  \\ \cmidrule(r){3-4} \cmidrule(r){5-6} \cmidrule(r){7-8}
 &  & mean $\pm$ std & Prec. & mean $\pm$ std & Prec. & mean $\pm$ std & Prec. \\
\cmidrule[\heavyrulewidth]{1-8}
\multirow{4.5}{*}{\shortstack[l]{Simple}}
  & \emph{SST (Our Impl.)}
  	& $\mathbf{3.63 \pm 0.17}$
  	& $\mathbf{28.1 \pm 2.7}$
  	& $3.37 \pm 0.08$
  	& $25 \pm 1.2$
  	& $3.14 \pm 0.09$
  	& $22.1 \pm 1.3$ \\
  \cmidrule[0.15pt]{2-8}
    & \emph{\exprv-REINFORCE+} 
    & $4.15 \pm 0.22$
    & $21.3 \pm 6.3$
    & $3.79 \pm 0.23$
    & $19.6 \pm 3.1$
    & $3.71 \pm 0.22$
    & $15.8 \pm 2.1$ \\
  & \emph{\tracerv-REINFORCE+}
    & $3.81 \pm 0.2$
    & $26.7 \pm 3.8$
    & $\mathbf{3.33 \pm 0.09}$
    & $\mathbf{26.9 \pm 1.1}$
    & $\mathbf{3.14 \pm 0.07}$
    & $21.6 \pm 1.2$ \\
  & \emph{RELAX} 
  	& $3.79 \pm 0.18$
  	& $26.8 \pm 3.4$
  	& $3.45 \pm 0.11$
  	& $23.6 \pm 1.6$
  	& $3.32 \pm 0.1$
  	& $\mathbf{22.3 \pm 1.2}$ \\
  \cmidrule[\heavyrulewidth]{1-8}
  \multirow{4.5}{*}{\shortstack[l]{Complex}}
  & \emph{SST (Our Impl.)}
    & $2.98 \pm 0.09$
    & $53.6 \pm 1$
    & $\mathbf{2.79 \pm 0.01}$
    & $45 \pm 1.2$
    & $\mathbf{2.75 \pm 0.03}$
    & $37.2 \pm 1.3$ \\
  \cmidrule[0.15pt]{2-8}
    & \emph{\exprv-REINFORCE+} 
    & $3.48 \pm 0.22$
    & $47.3 \pm 5.3$
    & $3.22 \pm 0.2$
    & $39.7 \pm 3$
    & $2.96 \pm 0.06$
    & $36.8 \pm 3.2$ \\
  & \emph{\tracerv-REINFORCE+}
    & $\mathbf{2.92 \pm 0.03}$
    & $\mathbf{56.3 \pm 0.8}$
    & $2.87 \pm 0.03$
    & $\mathbf{47.5 \pm 1.9}$
    & $2.82 \pm 0.06$
    & $\mathbf{40.4 \pm 1.8}$ \\
  & \emph{RELAX} 
    & $3.05 \pm 0.03$
    & $52.6 \pm 1.9$
    & $3.03 \pm 0.09$
    & $42.6 \pm 3.6$
    & $2.86 \pm 0.05$
    & $36.6 \pm 1.2$ \\
\bottomrule
\end{tabular}
}
\end{sc}
\end{small}
\end{center}
\end{table*}

\subsubsection{Data}
As a base, we used the BeerAdvocate \citep{mcauley2012learning} dataset, which consists of beer reviews and ratings for different aspects: Aroma, Taste, Palate and Appearance. In particular, we took its decorrelated subset along with the pretrained embeddings from \cite{lei2016rationalizing}. Every review was cut to $350$ embeddings, aspect ratings were normalized to $[0,1]$.

\subsubsection{Model}
We used the Simple and Complex models defined by \cite{paulus2020gradient} for parameterizing the mask. The Simple model architecture consisted of Dropout (with $p = 0.1$) and a one-layered convolution with one kernel. In the Complex model architecture, two more convolutional layers with 100 filters and kernels of size 3 were added.

\subsubsection{Training}
We trained all models for 10 epochs with $N=100$. We used the same hyperparameters ranges as in \cite{paulus2020gradient}, where it was possible. Hyperparameters for our training procedure were learning rate, final decay factor, weight decay. They were sampled from $\{1, 3, 5, 10, 30, 50, 100\} \times 10^{-4}, \{1, 10, 100, 1000\} \times 10^{-4}, \{0, 1, 10, 100\} \times 10^{-6} $ respectively. We also considered regularizer type for SST as hyperparameter ($\{\text{Euclid., Cat. Ent., Bin. Ent., E.F. Ent.}\}$). For $\exprv$-REINFORCE+ and $\tracerv$-REINFORCE+ number of latent samples for every example in a batch was considered as hyperparameter with range $\{1, 2, 4\}$. We tuned hyperapameters over considered ranges with uniform search with 25 trials. Best model were chosen with respect to best validation MSE.

Results for Appearance aspect can be found in Table~\ref{table:l2x-app}, for Taste aspect in Table~\ref{table:l2x-taste}, for Palate aspect in Table~\ref{table:l2x-pal}. Mean and standard deviations reported in the tables are computed across 16 different random model initializations. In conclusion, we can state that the proposed method is comparable with SST on BeerAdvocate dataset.

\subsection{Spanning Tree and Graph Layout}
\label{sec:graphlayoutexp}

\begin{table*}[!h]
\centering
\caption{Graph Layout experiment results for T=20 iterations. Metrics are obtained by choosing models with best validation ELBO and averaging results across different random seeds on the test set.}
\label{nri2-table}
\vskip 0.15in
\begin{center}
\begin{small}
\begin{sc}
\begin{tabular}{@{}lcccccc@{}}
    \toprule
     & \multicolumn{6}{c}{$T=20$} \\ 
     \cmidrule(l){2-7}     
     Estimator & \multicolumn{2}{c}{ELBO} & \multicolumn{2}{c}{Edge Prec.} & \multicolumn{2}{c}{Edge Rec.} \\
               & mean $\pm$ std & max & mean $\pm$ std & max & mean $\pm$ std & max \\ 
     \midrule
    \emph{SST (Our Impl.)} 
    & $-2039.42 \pm 1079.56$ 
    & $-1483.31$ 
    & $83 \pm 30$ 
    & $98$ 
    & $93 \pm 9$ 
    & $98$ \\
    \cmidrule[0.15pt]{1-7}
    \emph{\tracerv-REINFORCE+} 
    & $-1976.16 \pm 980.12$ 
    & $-1458.81$ 
    & $83 \pm 30$ 
    & $98$ 
    & $94 \pm 8$ 
    & $98$ \\
    \emph{RELAX} 
    & $-3129.51 \pm 1464.88$ 
    & $-1594.85$ 
    & $60 \pm 37$ 
    & $98$ 
    & $90 \pm 8$ 
    & $98$ \\
    \bottomrule
    \end{tabular}
\end{sc}
\end{small}
\end{center}
\vskip -0.1in
\end{table*}

\subsubsection{Data}
For each dataset entry we obtained the corresponding ground truth spanning tree by sampling a fully-connected graph on 10 nodes and applying Kruskal algorithm. Graph weights were sampled independently from $\text{Gumbel}(0, 1)$ distribution. Initial vertex locations in $\mathbb{R}^2$ were distributed according to $N(0, I)$. Given the spanning tree and initial vertex locations, we applied the force-directed algorithm \cite{fruchterman1991graph} for $T=10$ or $T=20$ iterations to obtain system dynamics. We dropped starting positions and represented each dataset entry as the obtained sequence of $T=10$ or $T=20$ observations. We generated 50000 examples for the training set and 10000 examples for the validation and test sets.

\subsubsection{Model}
Following \cite{paulus2020gradient} we used the NRI model with encoder and decoder architectures analogous to the MLP encoder and MLP decoder defined by \cite{kipf2018neural}.

\textbf{Encoder.} Given the observation of dynamics, GNN encoder passed messages over the fully connected graph. Denoting its final edge representation by $\theta$, we obtained parameters of the distribution over undirected graphs as $\frac{1}{2}(\theta_{ij} + \theta_{ji})$ for an edge $i \leftrightarrow j$. Hard samples of spanning trees
were then obtained by applying the Kruskal algorithm on the perturbed symmetrized matrix of parameters $\expparam_{ij} = \exp\left(-\frac{1}{2}\left(\theta_{ij} + \theta_{ji}\right)\right)$.

\textbf{Decoder.} GNN decoder took observations from previous timesteps and the adjacency matrix $X$ of the obtained spanning tree as its input. It passed messages over the latent tree aiming at predicting future locations of the vertices. We used two separate networks to send messages over two different connection types ($X_{ij} = 0$ and $X_{ij} = 1$). Since parameterization of the model was ambiguous in terms of choosing the correct graph between $X$ and $1 - X$, we measured structure metrics with respect to both representations and reported them for the graph with higher edge precision.

In experiments with RELAX we needed to define a critic. It was a simple neural network defined as an MLP which took observations concatenated with the perturbed weights and output a scalar. It had one hidden layer and ReLU activations.

\textbf{Objective.} During training we maximized ELBO (lower bound on the observations' log-probability) with gaussian log-likelihood and KL divergence measured in the continuous space of exponential noise. It resulted in an objective which was also a lower bound on ELBO with KL divergence measured with respect to the discrete distributions.

\subsubsection{Training}
We fixed the number of function evaluations per iteration at $N = 128$. All models were trained for 50000 iterations. We used constant learning rates and Adam optimizer with default hyperparameters. For all estimators we tuned separate learning rates for encoder, decoder and RELAX critic by uniform sampling from the range $[1, 100] \times 10^{-5}$ in log scale. Additionally, for $\tracerv$-REINFORCE+ we tuned $K$ in $\{2, 4, 8, 16\}$ and for RELAX we tuned size of the critic hidden layer in $\{256, 512, 1024, 2056\}$. We did not train $\exprv$-REINFORCE+ since \cite[Appendix, Sec. C.1]{paulus2020gradient} report its bad performance (REINFORCE (\textit{Multi-sample}) according to their namings). We used Gumbel Spanning Tree SST because it showed the best performance on the corresponding task in \cite[Section 8.1]{paulus2020gradient}. We tuned hyperapameters over considered ranges with uniform search with 20 trials. Best model were chosen with respect to best validation ELBO.

Mean and standard deviations reported in the tables are computed across 10 different random model initializations. Table \ref{nri2-table} reports results for T=20 iterations. Despite the fact that the dataset for this experiment is highly synthetic we can note that model initialization plays big role in the final performance of the gradient estimator. Overall, we can see that $\tracerv$-REINFORCE+ performs slightly better in terms of ELBO which is expected since score function based methods give unbiased gradients of ELBO, while relaxation-based SST optimizes relaxed objective.

\subsection{Arborescence and Unsupervised Parsing}
\label{sec:edmondsexp}
\subsubsection{Data}
We took the ListOps \cite{nangia2018listops} dataset, containing arithmetical prefix expressions, e.g.  $\texttt{min[3 med[3 5 4] 2]}$, as a base, and modified its sampling procedure. We considered only the examples of length in $[10, 50]$ that do not include the $\texttt{summod}$ operator and have bounded depth $d$. Depth was measured with respect to the ground truth parse tree,  defined as a directed graph with edges going from functions to their arguments. We generated equal number of examples for each $d$ in $\{1, \ldots, 5\}$. Train dataset contained 100000 samples, validation and test sets contained 20000 samples.

\subsubsection{Model}
Model mainly consisted of two parts which we call encoder and classifier.

Encoder was the pair of identical left-to-right LSTMs with one layer, hidden size 60 and dropout probability 0.1. Both LSTMs used the same embedding lookup table. Matrices that they produced by encoding the whole sequence were multiplied to get parameters of the distribution over latent graphs. Equivalently, parameter for the weight of the edge $i \rightarrow j$ was computed as $\theta_{ij} = \langle v_i, w_j \rangle$, where $v_i$ and $w_j$ are hidden vectors of the corresponding LSTMs at timesteps $i$ and $j$. Given $\expparam = \exp\left(-\theta\right) \in \mathbb{R}^{n \times n}$, we sampled matrix weights from the corresponding factorized exponential distribution. Hard samples of latent arborescences, rooted at the first token, were obtained by applying Chu-Liu Edmonds algorithm to the weighted graph.

Classifier mainly consisted of the graph neural network which had the initial sequence embedding as an input and ran 5 message sending iterations over the sampled arborescence's adjacency matrix. It had its own embedding layer different from used in the encoder. GNN's architecture was based on the MLP decoder model by \cite{kipf2018neural}. It had a two-layered MLP and did not include the last MLP after message passing steps. Output of the GNN was the final embedding of the first token which was passed to the last MLP with one hidden layer. All MLPs included ReLU activations and dropout with probability 0.1.

In experiments with RELAX we needed to define a critic. It contained LSTM used for encoding of the initial sequence. It was left-to-right, had a single layer with hidden size 60 and dropout probability 0.1. It had its own embedding lookup table. LSTM's output corresponding to the last token of the input sequence was concatenated with a sample of the graph adjacency matrix and fed into the output MLP with one hidden layer of size 60 and ReLU activations. Before being passed to the MLP, weights of the adjacency matrix were centered and normalized.

\subsubsection{Training}
We fixed the number of function evaluations per iteration at $N = 100$ and trained models for 50000 iterations. We used AdamW optimizer, separate for each part of the model: encoder, classifier and critic in case of RELAX. They all had constant, but not equal in general case, learning rates, and default hyperparameters. We used Gumbel arborescence SST because it showed the best performance on the corresponding task in \cite[Section 8.2]{paulus2020gradient}. We tuned learning rates and weight decays in range $[1, 100] \times 10^{-5}$ in log scale and the number of latent samples in $\{2, 4, 5\}$ for $\exprv$-REINFORCE+ and $\tracerv$-REINFORCE+. We tuned hyperapameters over considered ranges with uniform search with 20 trials. Best model were chosen with respect to best validation accuracy.

Table \ref{edmonds-table} with results indicates more stable performance of score function based gradient estimators with respect to different random model initializations.

\subsection{Binary Tree and Non-monotonic Generation}
\label{sec:nonmonexp}

\subsubsection{Data}
In this experiment, we constructed 5 datasets of balanced parentheses, varying the number of their types in $\{10, 20, \ldots, 50\}$. For each number of parentheses' types we constructed a dataset by generating independent sequences with the following procedure:
\label{sec:nonmonexp_data}
\begin{enumerate}
  \item Sample length $l$ of the sequence from the uniform distribution on $\{2, 4, \ldots, 20\}$.
  \item Uniformly choose current type of parentheses.
  \item Choose one of the configurations $\texttt{"( sub )"}$ or $\texttt{"() sub"}$ with equal probabilities, where $\texttt{"("}$ and $\texttt{")"}$ denote the pair corresponding to the current parentheses type.
  \item Make a recursive call to generate substring $\texttt{sub}$ with length $l - 2$.
  \item Return the obtained sequence.
\end{enumerate}

Each train dataset contained 20000 samples, validation and test sets contained 2500 samples. In case of semi-supervised experiments, datasets were modified to contain 10\% of supervision.

\subsubsection{Model}
Language model consisted of the decoder with non-monotonic architecture, defined in \cite{welleck2019non}, and of the encoder (in case of semi-supervised training). In this experiment all models shared the same hidden and embedding dimensions equal to 300.

\textbf{Decoder.}
We fixed decoder's architecture to be a single-layer left-to-right LSTM. While training, we processed a tree-ordered input by first adding leaf nodes, labeled with $\texttt{EOS}$ token, to all places with a child missing, and transforming the modified tree into a sequence by applying the level-order traversal. The obtained sequence was then used for training in the teacher-forcing mode. While generating, we sampled raw sequences (treated as level-order traversals), transformed them into binary trees and output the in-order traversal of the obtained tree.

\textbf{Encoder.}
For semi-supervised training we defined the encoder as a single-layer bidirectional LSTM. Given an input sequence of length $l$, it output a vector of exponential parameters $\expparam = (\expparam_1, \ldots, \expparam_l)$. Hard samples of latent trees were obtained by applying Algorithm \ref{alg:tree} on the perturbed $\expparam$.

\textbf{Critic.}
Critic, used for estimating encoder's gradients with RELAX, was defined as a single-layer bidirectional LSTM. It took a sequence, concatenated with perturbed output of the encoder along the embedding dimension, as its input, and output a single value.

\subsubsection{Smart order}
\begin{figure}[ht]
\begin{center}
\centerline{\includegraphics[width=0.3\textwidth]{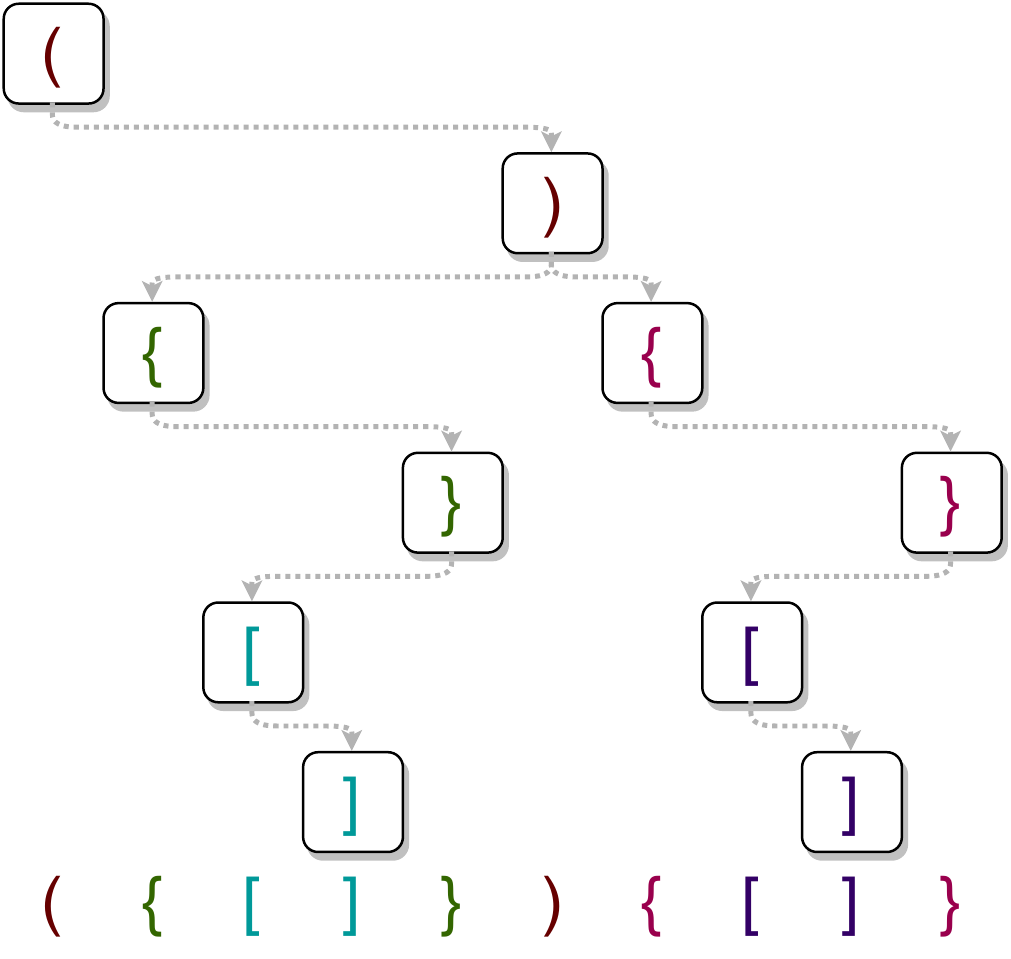}}
\caption{Visualization of $\textit{smart}$ order for binary trees}
\label{smart-order-trees}
\end{center}
\end{figure}
We defined $\textit{smart}$ order for binary trees in the way, visualized in Figure \ref{smart-order-trees}. Opening parentheses do not have left children, while their right children are fixed to be the corresponding closing parentheses. Construction starts from the first token; each time we generate a pair of parentheses and have a substring between them, we make a recursion step, generating the corresponding subtree at the left from the current closing parenthesis. If there is a substring at the right of the generated pair, corresponding subtree is attached as the closing parenthesis' right child.

From decoder's perspective (level-order generation) this order corresponds to an altering process of generation, where blocks of opening parentheses are followed by the corresponding closing ones. Intuitively, this type of process should simplify producing balanced parentheses sequences, since we do not mix opening and closing parentheses at each stage.

\subsubsection{Training}
Models with fixed order ($\textit{left-to-right}$ and $\textit{smart}$) were trained by minimization of cross-entropy using teacher-forcing. Semi-supervised models were trained in a manner of variational autoencoders. Unsupervised part of the training objective was defined by ELBO, lower bound on the marginal likelihood of training sequences, while supervised part consisted of joint likelihood (of sequence and fixed order), defined as negative cross-entropy between the decoder's output and train sequences, and the encoder's likelihood of the $\textit{smart}$ order.

All models were trained for 50 epochs. We chose the best model by measuring perplexity of the validation set. It was calculated explicitly for fixed-order models and approximated by IWAE bound \cite{burda2015importance} for semi-supervised ones. We observed that distribution of the encoder became degenerate during optimization, while decoder did not follow this behaviour. It made IWAE estimation with variational distribution highly underestimated. Instead of variational distribution, we used the empirical distribution on orders, obtained by sampling 10000 trees from decoder. Number of latent samples for IWAE estimation was fixed at $K = 1000$.

Results from Table \ref{toy-nonmonotonic} suggest that generative metrics of the model can be improved by training on the non-trivial order of generation even using semi-supervised approach with relatively small amount of supervision.

\subsection{Permutations by \texorpdfstring{$\operatorname{argsort}$}{Lg} and Non-monotonic Generation}
\label{sec:argsort_nonmonexp}
\begin{figure}[ht]
\begin{center}
\centerline{\includegraphics[width=\textwidth]{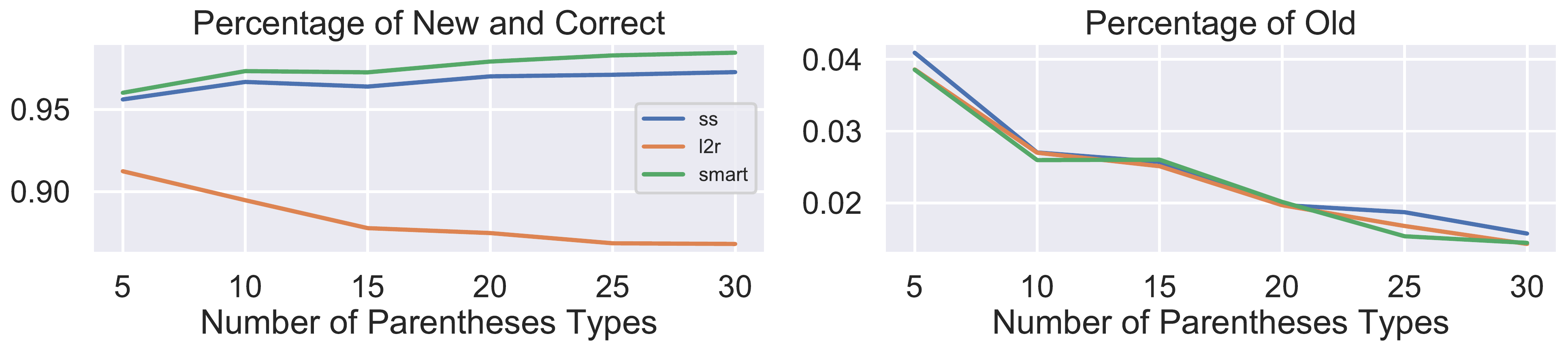}}
\caption{Generative statistics of insertion-based non-monotonic language model with different orders. Semi-supervised order is defined by Plackett-Luce distribution supervised with \textit{smart} order.}
\label{toy-nonmonotonic-argsort}
\end{center}
\end{figure}

To explore the applicability of the stochastic invariants to other models for non-monotonic text generation we consider the same task as in the previous experiment but examine another generative model as well as the latent variable which defines orderings.

\subsubsection{Data}
We used the same data generative process as in the previous experiment (Appendix \ref{sec:nonmonexp_data}).

\subsubsection{Model}
Model utilizes encoder-decoder architecture. For the decoder we took Transformer-InDIGO model from \cite{gu2019insertion}. It generates sequences by insertions using relative position representations. \cite{li2021discovering} discovered the one-to-one correspondence between relative positions and permutations, therefore we used them interchangeably in the model. Encoder is simple 1 layer bidirectional LSTM network which outputs parameters of the Plackett-Luce distribution given sequence of tokens. We used RELAX gradient estimator to train encoder parameters. Critic was also 1 layer bidirectional LSTM network which outputs scalar given concatenation of sequence of tokens and exponential noise.

\subsubsection{Training}
We trained the model with different orders: \textit{left-to-right}, \textit{smart} and \text{semi-supervised} with 10\% supervision with \textit{smart} orders for 100 epochs. \textit{Smart} order was defined by sequential generation of opening parenthesis and the corresponding closing parenthesis. Intuitively, it should be easier to generate balanced parentheses sequences using this order since model does not need stack to remember number of opened parentheses. Semi-supervised model was trained in the manner of semi-supervised variational autoencoders with teacher forcing for reconstruction term.

For each order we chose the best model with respect to the decoder perplexity (for semi-supervised we estimated marginal likelihood using IWAE estimator \cite{burda2015importance} with variational distribution as the proposal). During training we observed that different orders achieve the same perplexity which is expected since the data is too simple to model with any order. From the Figure \ref{toy-nonmonotonic-argsort} we can observe the same behaviour as with binary trees. While different orders achieve the same perplexity on the test set, using non-nomonotonic orders improves generation quality.

\subsection{Additional Tables}

\begin{table*}[!h]
\centering
\caption{Standard deviation (std) of the gradient estimators on Graph Layout experiment for T=10 iterations. Results are obtained by choosing models with best validation ELBO and averaging std estimates across \textit{train} set. Standard deviation is estimated with 10 samples for each batch.}
\label{nri-var-table}
\begin{center}
\begin{small}
\begin{sc}
\begin{tabular}{@{}lccc@{}}
    \toprule
     & \multicolumn{3}{c}{Mean Gradient Std} \\ 
     \cmidrule(l){2-4}     
     Estimator & \multicolumn{1}{c}{beginning} & \multicolumn{1}{c}{25k} & \multicolumn{1}{c}{50k} \\
     \midrule
    \emph{SST (Our Impl.)} 
    & $0.1674$ 
    & $0.0343$ 
    & $0.0379$ \\
    \emph{\tracerv-REINFORCE+} 
    & $1.6320$ 
    & $1.4409$ 
    & $0.9924$ \\
    \emph{RELAX} 
    & $3.0592$ 
    & $1.0292$ 
    & $0.8874$ \\
    \bottomrule
    \end{tabular}
\end{sc}
\end{small}
\end{center}
\end{table*}

\begin{table*}[!h]
\centering
\caption{Time per one gradient update (ms) for different gradient estimators and structured variables.}
\label{timings-table}
\begin{center}
\begin{small}
\begin{sc}
\begin{tabular}{@{}lccc@{}}
    \toprule
     & \multicolumn{3}{c}{Time per iter (ms)} \\ 
     \cmidrule(l){2-4}     
     Structure & \multicolumn{1}{c}{SST} & \multicolumn{1}{c}{\tracerv-REINFORCE+} & \multicolumn{1}{c}{RELAX} \\
     \midrule
    \emph{Spanning Tree} 
    & $123$ 
    & $137$ 
    & $180$ \\
    \emph{Arborescence} 
    & $175$ 
    & $249$ 
    & $535$ \\
    \bottomrule
    \end{tabular}
\end{sc}
\end{small}
\end{center}
\end{table*}

\begin{table*}[!ht]
\centering
\caption{Results of $k$-subset selection on Aroma aspect \textit{train} data. MSE ($\times 10^{-2}$) is shown for best models selected on validation averaged across different random seeds.}  
\label{table:l2x-train}
\begin{center}
\begin{small}
\begin{sc}
\scalebox{1.0}{
\begin{tabular}{@{}lllll@{}}
\toprule
\multirow{2.5}{*}{\shortstack[l]{Model}} & \multirow{2.5}{*}{\shortstack[l]{Estimator}} & \multicolumn{1}{c}{$k=5$}  & \multicolumn{1}{c}{$k=10$} & \multicolumn{1}{c}{$k=15$}  \\
 &  & mean $\pm$ std & mean $\pm$ std & mean $\pm$ std \\
\cmidrule[\heavyrulewidth]{1-5}
\multirow{4.5}{*}{\shortstack[l]{Simple}}
  & \emph{SST (Our Impl.)}
  	& $\mathbf{3.22 \pm 0.17}$
  	& $2.96 \pm 0.14$ 
  	& $2.54 \pm 0.13$  \\
  \cmidrule[0.15pt]{2-5}
    & \emph{\exprv-REINFORCE+} 
    & $3.45 \pm 0.19$ 
    & $3.38 \pm 0.23$ 
    & $2.85 \pm 0.21$ \\
  & \emph{\tracerv-REINFORCE+}
    & $3.38 \pm 0.13$ 
    & $\mathbf{2.93 \pm 0.19}$ 
    & $\mathbf{2.43 \pm 0.12}$ \\
  & \emph{RELAX} 
  	& $3.36 \pm 0.12$ 
  	& $3.13 \pm 0.18$ 
  	& $2.78 \pm 0.17$ \\
  \cmidrule[\heavyrulewidth]{1-5}
  \multirow{4.5}{*}{\shortstack[l]{Complex}}
  & \emph{SST (Our Impl.)}
    & $2.69 \pm 0.14$ 
    & $2.24 \pm 0.18$ 
    & $2.18 \pm 0.11$ \\
  \cmidrule[0.15pt]{2-5}
    & \emph{\exprv-REINFORCE+} 
    & $2.85 \pm 0.22$ 
    & $2.56 \pm 0.25$ 
    & $2.44 \pm 0.19$ \\
  & \emph{\tracerv-REINFORCE+}
    & $\mathbf{2.5 \pm 0.16}$ 
    & $\mathbf{2.19 \pm 0.15}$ 
    & $\mathbf{2.09 \pm 0.1}$ \\
  & \emph{RELAX} 
    & $2.53 \pm 0.15$ 
    & $2.21 \pm 0.17$ 
    & $2.19 \pm 0.14$ \\
\bottomrule
\end{tabular}
}
\end{sc}
\end{small}
\end{center}
\end{table*}
\begin{table*}[!ht]
\centering
\caption{Graph Layout experiment results for T=10 iterations. Metrics are obtained by choosing models with best validation ELBO and averaging results across different random seeds on the \textit{train} set.}
\label{nri1-train-table}
\begin{center}
\begin{small}
\begin{sc}
\begin{tabular}{@{}lcccccc@{}}
    \toprule
     & \multicolumn{6}{c}{$T=10$} \\ 
     \cmidrule(l){2-7}     
     Estimator & \multicolumn{2}{c}{ELBO} & \multicolumn{2}{c}{Edge Prec.} & \multicolumn{2}{c}{Edge Rec.} \\
               & mean $\pm$ std & max & mean $\pm$ std & max & mean $\pm$ std & max \\ 
     \midrule
    \emph{SST (Our Impl.)} 
    & $-1846.93 \pm 1124.23$ 
    & $-1357.03$ 
    & $\mathbf{88 \pm 22}$ 
    & $\mathbf{96}$ 
    & $\mathbf{94 \pm 4}$ 
    & $\mathbf{96}$ \\
    \cmidrule[0.15pt]{1-7}
    \emph{\tracerv-REINFORCE+} 
    & $\mathbf{-1584.11 \pm 572.16}$ 
    & $\mathbf{-1193.11}$ 
    & $70 \pm 31$ 
    & $91$ 
    & $86 \pm 8$ 
    & $91$ \\
    \emph{RELAX} 
    & $-2086.93 \pm 573.72$ 
    & $-1207.83$ 
    & $43 \pm 31$ 
    & $90$ 
    & $81 \pm 6$ 
    & $90$ \\
    \bottomrule
    \end{tabular}
\end{sc}
\end{small}
\end{center}
\end{table*}

\begin{table*}[!ht]
\centering
\caption{Unsupervised Parsing on ListOps. We report the average \textit{train}-performance of the model with the best validation accuracy across different random initializations.}
\label{edmonds-table-train}
\begin{center}
\begin{small}
\begin{sc}
\begin{tabular}{*7l}
\toprule 
Estimator &
\multicolumn{2}{c}{Accuracy} &
\multicolumn{2}{c}{Precision} &
\multicolumn{2}{c}{Recall} \\
& mean $\pm$ std & max
& mean $\pm$ std & max
& mean $\pm$ std & max
 \\
\midrule
   \emph{SST (Our Impl.)}
   & $79.31 \pm 8.17$ & $\mathbf{94.73}$ 
    & $57.15 \pm 19.92$ & $\mathbf{82.51}$
    & $30.58 \pm 19.03$ & $73.28$
    \\
  \cmidrule[0.15pt]{1-7}
     \emph{\exprv-REINFORCE+}
   & $60.64 \pm 2.51$ & $65.21$
   & $41.12 \pm 6.68$ & $45.95$
   & $40.99 \pm 6.75$ & $45.68$
     \\
   \emph{\tracerv-REINFORCE+}
   & $\mathbf{88.69 \pm 3.02}$ & $93.06$
   & $\mathbf{78.6 \pm 7.37}$ & $80.68$
   & $\mathbf{61.78 \pm 14.52}$ & $\mathbf{80.68}$
    \\
   \emph{RELAX} 
   & $79.92 \pm 9.35$ & $88.64$ 
   & $54.84 \pm 17.51$ & $75.27$ 
   & $53.73 \pm 17.18$ & $75.27$
   \\
\bottomrule
\end{tabular}
\end{sc}
\end{small}
\end{center}
\end{table*}

\newpage
\section{Pseudo-Code}
\label{sec:pseudo-code}
This section containes pseudo-code for the algorithms with stochastic invariants discussed in the paper.
\subsection{Pseudo-Code for \texorpdfstring{$\argtopk$}{Lg}}
We refer the reader to Algorithm~\ref{alg:top-k} in the main paper.
\subsection{Pseudo-Code for \texorpdfstring{$\operatorname{argsort}$}{Lg}}
Algorithm~\ref{alg:argsort} presents a recursive algorithm for sorting. The algorithm implements the insertion sorting. Although insertion sorting may not be the most efficient sorting algorithm, it has a stochastic invariant. Indeed, the algorithm recursively finds the minimum element and then excludes the element from the consideration. As opposed to $\argtopk$, the algorithm does not omit the order of $\structrv'$. As a result, for the algorithm, the trace $\tracerv$ coincides with the output $\structrv$.
\begin{algorithm}[ht]
\caption{$F_{\text{sort}}(E, K)$ - sorts the set $K$ based on the corresponding $\exprv$ values}
\label{alg:argsort}
\begin{algorithmic}
    \REQUIRE $\exprv, K$
    \ENSURE $\structrv$
    \IF{$\exprv = \emptyset$}
    \STATE {\bf return}
    \ENDIF
    \STATE $\tracerv \Leftarrow \arg\min_{j \in K} \exprv_j$ \hfill \COMMENT{Find the smallest element}
    \FOR{$j \in K$}
    \STATE $\exprv_j' \Leftarrow \exprv_j - \exprv_\tracerv$
    \ENDFOR
    \STATE $K' \Leftarrow K \setminus \{T\}$ \hfill \COMMENT{Exclude $\operatorname{arg}\min$ index $\tracerv$}
    \STATE $\exprv' \Leftarrow  \{ \exprv_k' \mid k \in K'\}$ 
    \STATE $\structrv' \Leftarrow F_{\text{sort}}(E', K')$ \hfill \COMMENT{Sort the subset $K'$}
    \STATE \textbf{return} $(\tracerv, \structrv'_1, \dots, \structrv'_{\text{size}(\structrv')})$ \hfill \COMMENT{Concatenate $\tracerv$ and the subset sorting $\structrv'$}
\end{algorithmic}
\end{algorithm}

\subsection{Pseudo-Code for the Bespoke Matching Variable}

\begin{figure}[ht]
\begin{center}
\includegraphics[width=0.5\textwidth]{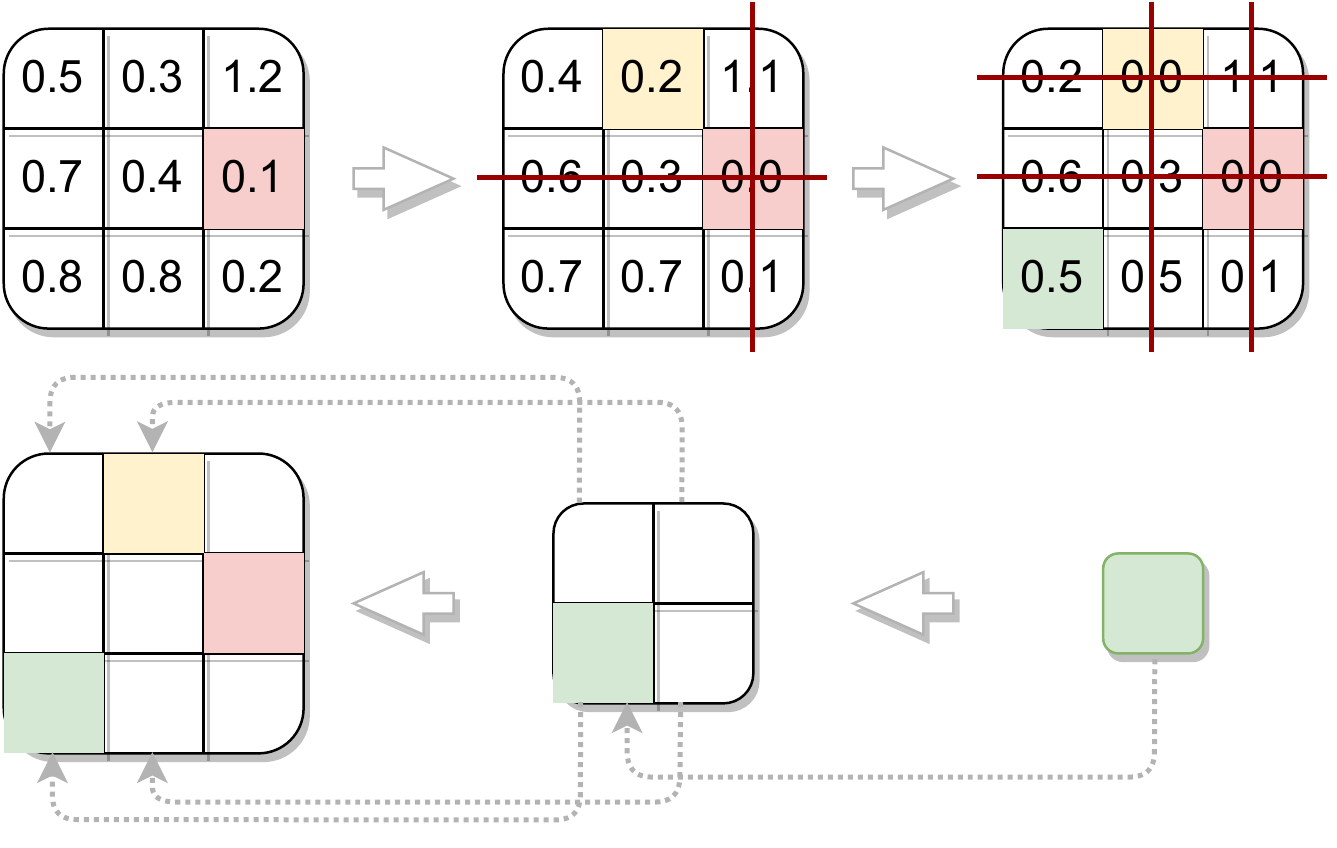}
\end{center}
\caption{The generative process for perfect matchings. On the top row, the algorithm recursively finds the minimum element and excludes the corresponding row and column. On the bottom row, the algorithm iteratively combines the subset matching~$\structrv'$ and the minimum element~$\tracerv$.}
\label{fig:matching}
\end{figure}

\begin{algorithm}[!h]
\caption{$F_{\text{match}}(E, K)$ - returns a matching between the columns and the rows of a square matrix $\exprv$ with the elements indexed with $K$.}
\label{alg:matching}
\begin{algorithmic}
    \REQUIRE $\exprv, K$
    \ENSURE $\structrv$
    \IF{$\exprv = \emptyset$}
    \STATE {\bf return}
    \ENDIF
    \STATE $\tracerv \Leftarrow \arg\min_{(u, v) \in K} \exprv_{(u,v)}$ \hfill \COMMENT{Find the smallest element, $\tracerv$ is an integer tuple}
    \FOR{$(u, v) \in K$}
    \STATE $\exprv_{(u, v)}' \Leftarrow \exprv_{(u, v)} - \exprv_\tracerv$
    \ENDFOR
    \STATE \COMMENT{Cross out the row and the column containing the minimum}
    \STATE $K' \Leftarrow \{(u, v) \in K \mid u \neq \tracerv[0] \lor v \neq \tracerv[1]\}$
    \STATE $\exprv' \Leftarrow  \{ \exprv_{(u,v)}' \mid (u, v) \in K'\}$
    \STATE $\structrv' \Leftarrow F_{\text{match}}(E', K')$ \hfill \COMMENT{Find a matching for the sub-matrix}
    \STATE \textbf{return} $\{\tracerv\} \cup \structrv'$ \hfill \COMMENT{Add an edge (tuple) to the matching}
\end{algorithmic}
\end{algorithm}

We present an algorithm with the stochastic invariant that returns a matching between the rows and the columns of a square matrix. Although such a variable is in one-to-one correspondence with permutations, the distribution has more parameters than the Plackett-Luce distribution. We speculate that the distribution may be more suitable for representing finite one-to-one mappings with a latent variable. Figure~\ref{fig:matching} illustrates the idea behind the algorithm. In particular, the algorithm iteratively finds the minimum element and excludes the row and the column containing the element from the matrix. Then the algorithm uses recursion to construct a matching for the submatrix.

Notably, we were unable to represent the Hungarian algorithm for the minimum matching to the problem. Algorithm~\ref{alg:matching} returns the same output when the column-wise minimum elements form a matching in the matrix. In general, the output matching may not be the minimum matching. 

\subsection{Pseudo-Code for the Bespoke Binary Tree Variable}
\begin{figure}[ht]
\begin{center}
\includegraphics[width=0.5\textwidth]{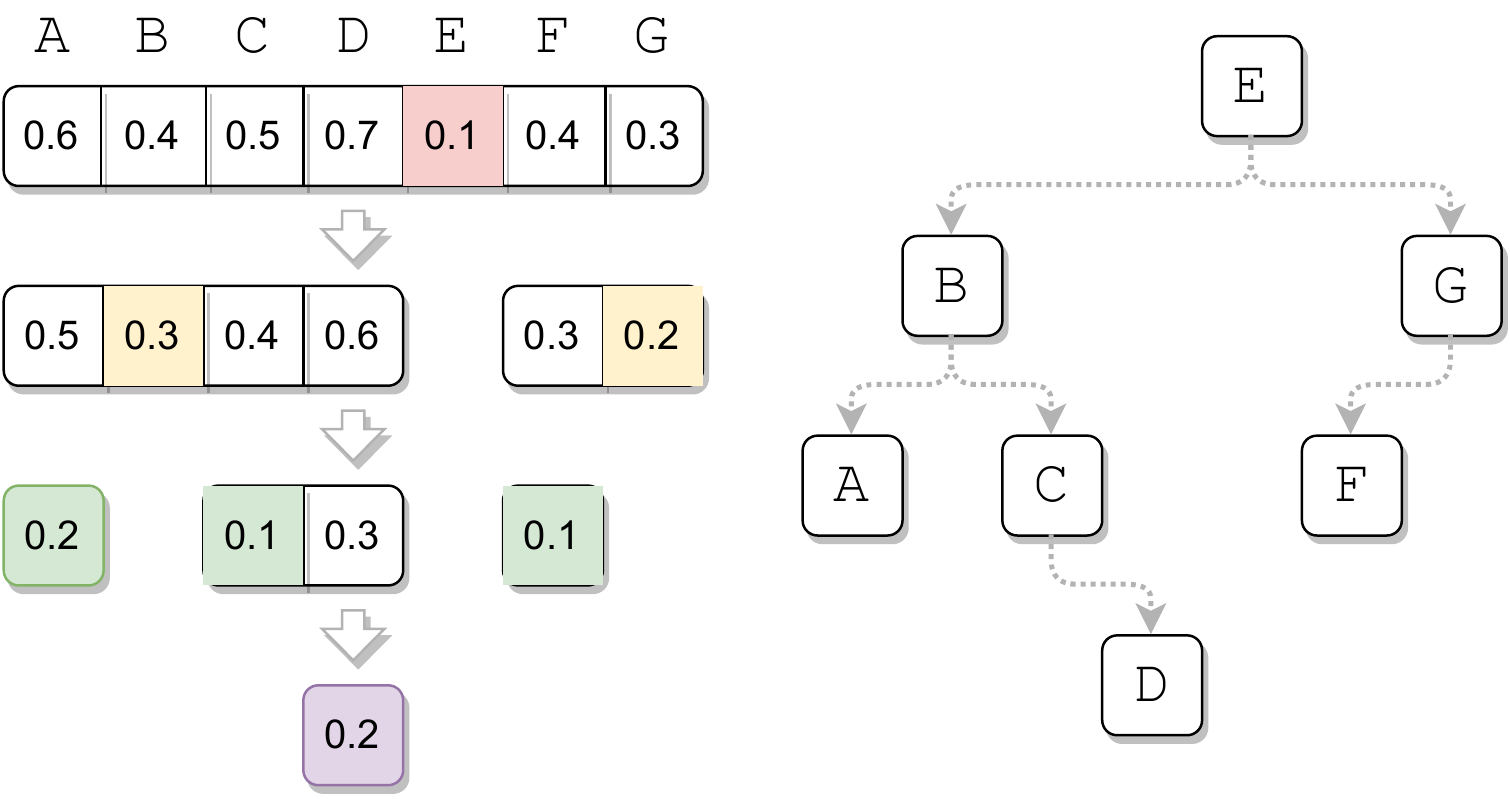}
\end{center}
\caption{The generative process for binary trees. On the left, we assign weights to tokens and set the minimum weight token to be the root. Then we recursively construct trees for the tokens on the left-hand side and the right-hand side of the root. On the right, we present the resulting tree.}
\label{fig:tree}
\end{figure}

\begin{algorithm}[!h]
\caption{$F_{\text{tree}}(E, K, R)$ - constructs a binary tree based on weight $\exprv$ with the node set $K$. The auxiliary variable $R$ is a partition of nodes initialized as a single set $K$.}
\label{alg:tree}
\begin{algorithmic}
    \REQUIRE $\exprv, K, R$
    \ENSURE $\structrv$
    \IF{$\exprv = \emptyset$}
    \STATE {\bf return}
    \ENDIF
    \STATE $P_1, \dots, P_m \Leftarrow R$
    \FOR{$i=1$ to $m$}
    \STATE $\tracerv_i \Leftarrow \arg\min_{j \in P_i} \exprv_j$
    \FOR{$j \in P_i$}
    \STATE $\exprv_j' \Leftarrow \exprv_j - \exprv_{\tracerv_i}$
    \ENDFOR
    \ENDFOR
    \STATE \COMMENT{$R'$ splits the partition sets $P_i$ into the left-hand side and the right-hand side nodes relative to $\tracerv_i$}
    \STATE \COMMENT{for example, when $\tracerv = 5$ we split $P = \{3, 4, 5, 6, 7, 8\}$ into $\{3, 4\}$ and $\{6, 7, 8\}$}
    \STATE $R' \Leftarrow \{P_i[0, \tracerv_i) \mid i=1,\dots,m\} \cup \{P_i(\tracerv_i, -1] \mid i=1, \dots, m\}$
    \STATE $K' \Leftarrow K \setminus \{\tracerv_1, \dots, \tracerv_m\}$
    \STATE $E' \Leftarrow \{E_k' \mid k \in K'\}$
    \STATE $\structrv' \Leftarrow F_{\text{tree}}(E', K', R')$ \hfill \COMMENT{Recursive call returns a sequence of $2m$ subtrees}
    \STATE {\bf return} $\big((\tracerv_i, \structrv'[2i], \structrv'[2i+i]), i=1, \dots, m \big)$ \hfill \COMMENT{Join the $2m$ trees into $m$ with roots $\tracerv_1, \dots, \tracerv_m$}
\end{algorithmic}
\end{algorithm}

For our experiments with the non-monotonic generation, we propose a distribution over binary trees. 
Given a sequence of tokens, we assign an exponential random variable to each token. Then we construct a tree with the following procedure illustrated in Figure~\ref{fig:tree}. First, we set the token with the minimum weight to be the root of the tree. The tokens on the left-hand side from the root will be the left descendants of the root, the tokens on the right-hand side will be the right descendants of the root. Then we repeat the procedure for the left-hand descendants and the right-hand descendants independently. We summarise the above in Algorithm~\ref{alg:tree}.

Intuitively, the algorithm should include two recursive calls: one for the left-hand side subtree and the other for the right-hand side subtree. According to the general framework (see Algorithm~\ref{alg:general}), our pseudo-code is limited to a single recursive call. In particular, at each recursion depth $k$ Algorithm~\ref{alg:tree} processes all subtrees with a root at the given depth $k$. As a result, at depth $k$ the partition includes $m=2^k$ subsets, some of which may be empty. Alternatively, Algorithm~\ref{alg:general} can be extended to multiple recursive calls.

\subsection{Kruskal's Algorithm}
Kruskal's algorithm \citep{kruskal1956shortest, dasgupta2008algorithms} for the minimum spanning tree gives another illustration of the framework. In this case, the edge weights are the exponential random variables. The input variable $\exprv$ is indexed by the edges in a graph, i.e. $\exprv_{(u, v)}$ is the weight of the edge~$(u, v)$.

Algorithm~\ref{alg:kruskal} contains the pseudo-code for the Kruskal's algorithm. The auxiliary variable $R$ is a set of disjoint sets of nodes. It represents the connected components of the current subtree. The algorithm build the tree edge-by-edge. It starts with an empty set of edges and all sets in $R$ of size one. Then the algorithm finds the lightest edge connecting between the connected components of $R$ and joins the two connected components. The algorithm repeats this greedy strategy until $R$ contains a single connected component.

From the recursion viewpoint, the algorithm constructs a tree with the nodes being the elements of $R$. First, the algorithm adds the lightest edge $\tracerv = (u, v)$ to the tree. Then it joins the sets $R_u$ and $R_v$ containing $u$ and $v$, we denote the result as $R'$. Next, the algorithm uses the recursion to construct a tree where the nodes are the elements of $R'$. The size of $R$ decreases with each step, therefore the recursion will stop. The resulting tree is a tree for the connected components $R'$ along with the edge $\tracerv$.

Notably, Prim's algorithm is a similar greedy algorithm for finding the minimum spanning tree. However, the algorithm considers different subsets of edges; as a result, we could not represent Prim's algorithm as an instance of Algorithm~\ref{alg:general}.

\begin{algorithm}[ht]
\caption{$F_\text{Kruskal}(\exprv, K, R)$ - finds the minimum spanning tree given edges~$K$ with the corresponding weights~$\exprv$; Call with $R := \{ \{v\} \mid v \in V\}$ set as node singletons} 
\label{alg:kruskal}
\begin{algorithmic}
    \REQUIRE $\exprv, K, R$
    \ENSURE $\structrv$
    \IF{$|R| = 1$}
    \STATE {\bf return}
    \ENDIF
    \STATE $\tracerv \Leftarrow \arg\min_{k \in K} \exprv_k$ \hfill\COMMENT{Find the smallest edge}
    \FOR{$k \in K$}
    \STATE $\exprv_k' \Leftarrow \exprv_k - \exprv_\tracerv$
    \ENDFOR
    \STATE \COMMENT{For $\tracerv = (u, v)$ find $R_u, R_v \in R$ s.t. $u \in R_u$, $v \in R_v$}
    \STATE $R_u, R_v \Leftarrow \operatorname{find\_connected\_components}(R, T)$
    \hfill \COMMENT{Merge the connected components $R_u$ and $R_v$}
    \STATE $R' \Leftarrow (R \setminus \{R_u, R_v\}) \cup (\{R_u \cup R_v\})$
    \hfill \COMMENT{Remove edges connecting $R_u$ and $R_v$}
    \STATE $K' \Leftarrow K \setminus \{(u', v') \in K \mid u', v' \in R_u \cup R_v\}$
    \STATE $\exprv' \Leftarrow \{\exprv_k' \mid k \in K'\}$
    \STATE $\structrv' \Leftarrow F_\text{Kruskal}(\exprv', K', R')$
    \hfill \COMMENT{Edges in $\structrv'$ form a spanning tree for nodes in $R'$}
    \STATE \hfill \COMMENT{$X \cup \{\tracerv = (u, v)\}$ is a spanning tree for nodes in $R$}
    \STATE {\bf return} $\structrv' \cup \{\tracerv\}$
\end{algorithmic}
\end{algorithm}

\subsection{Chu-Liu-Edmonds Algorithm}
We adopt Chu-Liu-Edmonds algorithm from~\cite{kleinberg2006algorithm} in Algorithm~\ref{alg:cle}.
Similarly to Kruskal's algorithm, the perturbed input $\exprv$ represents the weights of the graph edges, its indices are the edges of the input directed graph.

As opposed to the previous examples, the algorithm considers multiple subsets of indices $P_1, \dots, P_m$ at each recursion level.
In particular, for each node except $r$ the algorithm finds the incoming edge with minimal weight.
If $\{\tracerv_i\}_{i \neq r}$ is an arborescence, $f_{\text{combine}}$ returns it.

Otherwise, $\{\tracerv_i\}_{i \neq r}$ contains a cycle and $f_{\text{map}}$ constructs a new graph with the cycle nodes contracted to a single node. Similarly to Kruskal's algorithm, we use the variable $R$ to store sets of nodes. In this case, $R$ represent the contracted node as a set of the original nodes.
To construct~$\structrv$, the subroutine~$f_{\text{combine}}$ expands the contracted loop in~$\structrv'$ and adds all edges in $C$ but one.

\begin{algorithm}[ht]
\caption{$F_{\text{CLE}}(E, K, R, r)$ - finds the minimum arborescence $\structrv$ of a directed graph with edges $K$ of weight $\exprv$ and root node $r$. Auxiliary variable $R$ is a partition of nodes indicating merged nodes, initialized with node singletons $R := \{ \{v \} \mid v \in V \}.$}
\label{alg:cle}
\begin{algorithmic}
    \REQUIRE $\exprv, K, R, r$
    \ENSURE $\structrv$
    \IF{$|R| = 1$}
    \STATE {\bf return}
    \ENDIF
    \STATE $P_1, \dots, P_m \Leftarrow f_{\text{split}}(K, R, r)$  \hfill \COMMENT{Split $K$ into sets of edges ending at $R_v \in R$, $r \notin R_v$;} 
    \FOR{$i=1$ to $m$}
    \STATE $\tracerv_i \Leftarrow \arg\min_{k \in P_i} \exprv_k$
    \FOR{$k \in P_i$}
    \STATE $\exprv_k' \Leftarrow \exprv_k - \exprv_{\tracerv_i}$
    \ENDFOR
    \ENDFOR
    \STATE $C \Leftarrow \operatorname{find\_loop}(R, \{T_i\}_{i=1}^m)$
    \hfill \COMMENT{Find a loop $C$ assuming nodes $R$ and edges $\{T_i\}_{i=1}^{m}$}
    \IF{$C = \emptyset$}
    \STATE {\bf return} $\{T_i\}_{i=1}^m$ \hfill
    \ENDIF
    \STATE $R' \Leftarrow (R \setminus \{C_i \mid C_i \in C\}) \cup (\cup_{i=1}^{|C|} C_i)$
    \hfill \COMMENT{Contract the loop nodes into a single node}
    \STATE $K' \Leftarrow K \setminus \{(u, v) \in K \mid u \in C_i, v \in C_{j}\}$ 
    \STATE $\exprv' \Leftarrow \{\exprv_k' \mid k \in K' \}$
    \STATE $\structrv' \Leftarrow F_{\text{CLE}}(\exprv', K', R', r)$
    \hfill \COMMENT{Find arborescence for the contracted graph}
    \STATE $\structrv \Leftarrow \structrv' \cup \{T_i \mid T_i \text{ in cycle, preserves arborescence}\}$
    \hfill \COMMENT{Add to $\structrv'$ all loop edges but one}
    \STATE {\bf return} $\structrv$
\end{algorithmic}
\end{algorithm}

\end{document}


\section*{Outline}
\begin{itemize}
\item Section~\ref{sec:proofs}: Proofs of the lemmas and the theorem;
\item Section~\ref{sec:algorithms}: Pseudo-code the score function~$\logprob{\tracerv; \expparam}$ and the conditional reparameterization of~$\exprv \mid \tracerv$;
\item Section~\ref{sec:details}: Discussion of the algorithm implementations;
\item Section~\ref{sec:experiments}: Additional results and experimental details.
\end{itemize}

\section{Proofs}
\label{sec:proofs}
In this section, we provide the proofs for the lemmas and the theorem formulated in the main part.
\subsection{Exponential Min Trick}
\begin{lemma}
\label{lemma:exp-min-trick}
\textit{(Exponential min trick)} If $\exprv_i \sim \operatorname{Exp}{(\expparam_i)}, i \in \{1, \dots, \thetadim\}$ are independent, then for $\catrv := \operatorname{argmin}\limits_{i} \exprv_i$
\begin{enumerate}
\item $\prob{\catrv = i} \propto \expparam_i$
\item $\exprv_i - \exprv_\catrv \sim \begin{cases} \operatorname{Exp}({\lambda_i}) &\text{ if } i \neq \catrv \\ 0 &\text{ otherwise} \end{cases}$
\end{enumerate}
\end{lemma}
\begin{proof}
Starting with the joint density of $\catrv$ and $\exprv$
\begin{equation}
\prod_{j=1}^{\thetadim} \left( \expparam_j e^{-\expparam_j e_j} \cdot \mathbb{1}(e_x \leq e_j) \right),
\end{equation}
we make the substitute $e'_j := e_j - e_x$ for $j \in \{1, \dots, \thetadim\}$ and rewrite the density
\begin{align}
\prod_{j=1}^{\thetadim} \left( \expparam_j e^{-\expparam_j (e'_j + e_x)} \cdot \mathbb{1}(0 \leq e'_j) \right) = \\
\expparam_x e^{-\sum_{j=1}^\thetadim \expparam_j e_x} \prod_{j = 1}^{\thetadim} \left( \expparam_j e^{-\expparam_j e'_j} \mathbb{1}(0 \leq e'_j) \right) = \\
\tfrac{\expparam_x}{\sum_{j=1}^\thetadim \expparam_j} \times (\sum_{j=1}^\thetadim \expparam_j) e^{-(\sum_{j=1}^\thetadim \expparam_j) e_x} \\
\times \prod_{j = 1}^{\thetadim} \left( \expparam_j e^{-\expparam_j e'_j} \mathbb{1}(0 \leq e'_j) \right).
\end{align}
The latter is the joint density of
\begin{enumerate}
\item a categorical latent variable $\catrv$ with $\prob{\catrv = x} \propto \expparam_x$;
\item an independent exponential random variable $\exprv_\catrv$ with rate parameter $\sum_{j=1}^\thetadim \expparam_j$;
\item and a sequence of random variables $\exprv'_j := \exprv_j - \exprv_\catrv, j \neq x$ with mutually independent exponential distributions $\operatorname{Exp}{(\expparam_j)}$ conditioned on $\catrv$ and $\exprv_\catrv$.
\end{enumerate}
\end{proof}
\subsection{Alternative Score-Function Estimators}
We extend the formulation of Lemma~\ref{lemma:unbiased-estimates} from the main paper with the explicit requirements of $\tracerv(\cdot)$. The requirements are the sufficient conditions for the log-derivative trick and hold for $\tracerv$ defined by Algorithm~\ref{alg:general}.
\begin{lemma} For a sequence of exponential random variables $\exprv$ with parameter $\expparam$, a \textbf{piecewise constant} function $\tracerv = \tracerv(\exprv)$ \textbf{with a finite number of outputs $|\operatorname{Im} \tracerv | < \infty$} and a function $\structrv = \structrv(\tracerv)$ we have
\begin{equation}
\nabla_\expparam \mathbb E_{\structrv} \loss{\structrv} = \mathbb E_{\exprv} \loss{\structrv} \nabla_\expparam \logprob{\exprv; \expparam} 
\end{equation}
and
\begin{equation}
\nabla_\expparam \mathbb E_{\structrv} \loss{\structrv} = \mathbb E_{\tracerv} \loss{\structrv} \nabla_\expparam \logprob{\tracerv; \expparam}
\end{equation}
\label{lemma:unbiased-estimates}
\end{lemma}
\begin{proof}
Since $\structrv$ and $\tracerv$ are functions of $\exprv$, by the law of the unconscious statistician
\begin{equation}
\mathbb E_{\structrv} \loss{\structrv} =
\mathbb E_{\tracerv} \loss{\structrv} = 
\mathbb E_{\exprv} \loss{\structrv}.
\end{equation}
Moreover, the gradient $\nabla_\expparam \mathbb E_{\exprv} \loss{\structrv}$ exists and is equal to
\begin{equation}
\label{eq:log-derivative-trick}
\mathbb E_{\exprv} \loss{\structrv} \nabla_\expparam \logprob{\exprv; \expparam}.
\end{equation}
The above follows from the Leibniz integral rule applied to the expectation of a piece-wise constant function $\loss{\structrv(\tracerv(\cdot))}$ of the exponential random variables. The existence of the gradient also implies
\begin{equation}
\nabla_\expparam
\mathbb E_{\structrv} \loss{\structrv} = 
\nabla_\expparam
\mathbb E_{\tracerv} \loss{\structrv} =
\nabla_\expparam
\mathbb E_{\exprv} \loss{\structrv}.
\end{equation}

The second expectation $\mathbb E_{\tracerv} \loss{\structrv}$ is the finite weighted sum
\begin{equation}
\sum_{t} \prob{\tracerv = t; \expparam} \loss{X(t)}
\end{equation}
with $\prob{\tracerv = t; \expparam} = \int \prob{\exprv = e; \expparam} \mathbb 1(\tracerv(e) = t) \mathrm d e$. Again, the Leibniz integral rule implies that each weight $\prob{\tracerv = t; \expparam}$ is differentiable w.r.t. $\expparam$. Similar to Eq.~\ref{eq:log-derivative-trick}, we obtain
\begin{equation}
\nabla_\expparam \mathbb E_{\tracerv} \loss{\structrv} = \mathbb E_{\tracerv} \loss{\structrv} \nabla_\expparam \logprob{\tracerv; \expparam}
\end{equation}
using the log-derivative trick.
\end{proof}
\subsection{On the Variance of the Score-Function Estimators}
The lemma below explicitly specifies the conditions on the baseline $Y$ that ensure that the first and the second moments of the estimators are finite.
\begin{lemma} Define $\hat\exprv, \hat\tracerv, \hat\structrv$ samples from the corresponding distributions from Lemma~\ref{lemma:unbiased-estimates}. Then for a random variable $Y \perp \exprv$ \textbf{with $\mathbb E Y < \infty$ and $\mathbb E Y^2 < \infty$} and the corresponding sample $\hat Y$, the gradient estimator
\begin{equation}
g_\tracerv = \left( \loss{\hat\structrv} - \hat Y\right) \nabla_\expparam \logprob{\hat\tracerv; \expparam}
\end{equation}
and the gradient estimator
\begin{equation}
g_\exprv = \left( \loss{\hat\structrv} - \hat Y\right) \nabla_\expparam \logprob{\hat\exprv; \expparam},
\end{equation} we have
\begin{equation}
\operatorname{var}{g_\tracerv} \leq \operatorname{var}{g_\exprv}.
\end{equation}
\end{lemma}
\begin{proof}
We first note that both both $g_\tracerv$ and $g_\exprv$ are unbiased estimators of $\nabla_\expparam \mathbb E_{\structrv} \loss{\structrv}$, therefore 
\begin{equation}
\operatorname{var}{g_\tracerv} - \operatorname{var}{g_\exprv} = \mathbb E_{\tracerv, Y} g_\tracerv^2 -  \mathbb E_{\exprv, Y} g_\exprv^2.
\end{equation}
Indeed, the random baseline $Y$ does not introduce bias
\begin{align}
\mathbb E Y \nabla_\expparam \logprob{\tracerv; \expparam} = \mathbb E Y \mathbb E \nabla_\expparam \logprob{\tracerv; \expparam} = 0
\end{align}
because the expected value of the score function is zero.

Next we rewrite the second moment of $g_\exprv$ in terms of the second moment of $g_\tracerv$ and a non-negative term. We note that $\prob{\exprv, \tracerv; \expparam} = \prob{\exprv; \expparam} \mathbb 1(\tracerv = \tracerv(\exprv))$, therefore
\begin{align}
& \nabla_\expparam \logprob{\exprv; \expparam} = \nabla_\expparam \logprob{\exprv, \tracerv; \expparam} = \\
& \nabla_\expparam \logprob{\tracerv; \expparam} + \nabla_\expparam \logprob{\exprv \mid \tracerv; \expparam}.
\end{align}
Using the above equation, we rewrite $\mathbb E_{\exprv, Y} g_\exprv^2$ as
\begin{align}
\mathbb E_{\exprv, Y} \big(&(\loss{\structrv} - Y) \\
& (\nabla_\expparam \logprob{\tracerv; \expparam} + \nabla_\expparam \logprob{\exprv \mid \tracerv; \expparam})\big)^2
\end{align}
Expanding the squared sum, we obtain three terms. The first term is
\begin{equation}
\mathbb E_{\exprv, Y} \left((\loss{\structrv} - Y) \nabla_\lambda \logprob{\tracerv; \lambda} \right)^2 = \mathbb E_{\tracerv, Y} g_\tracerv^2.
\end{equation}
The second term is zero. Indeed, after marginalizing w.r.t. $\exprv$ one of the multipliers becomes zero
\begin{align}
\mathbb E_{\exprv, Y} \nabla_\expparam \logprob{\exprv \mid \tracerv; \expparam} = \\
\mathbb E_{\tracerv, Y} \mathbb E_{\exprv \mid \tracerv} \nabla_\expparam \logprob{\exprv \mid \tracerv; \expparam}
\end{align}
as the expected value of the score function $\mathbb E_{\exprv \mid \tracerv} \nabla_\expparam \logprob{\exprv \mid \tracerv; \expparam} = 0$. Therefore the whole term is zero (the marginalization w.r.t. $\exprv \mid \tracerv$ is possible because $\structrv = \structrv(\tracerv)$ and $\exprv \mid \tracerv$ are independent).
Finally, the third term is non-negative
\begin{equation}
\mathbb E_{\exprv, Y} \left( (\loss{\structrv} - Y ) \nabla_\exprv \logprob{\tracerv \mid \exprv; \expparam } \right)^2 \geq 0.
\end{equation}
Combining the three terms, we obtain
$\mathbb E_{\tracerv, Y} g_\tracerv^2 \leq  \mathbb E_{\exprv, Y} g_\exprv^2$,
therefore
$ \operatorname{var}{g_\tracerv} \leq \operatorname{var}{g_\exprv}.$
\end{proof}
\subsection{Distributions of $\tracerv$ and $\exprv \mid \tracerv$}
Below we prove the main claim of this work. Again, Algorithm~\ref{alg:general} assumes that 
\begin{itemize}
\item the output of $f_{\text{split}}$ is a sequence of disjoint sets, i.e. $K_i \cap K_j = \emptyset$, and $\cup K_i = \dom \exprv$;
\item the index map $K'$ returned by $f_{\text{map}}$ is an injective map and $\operatorname{Im}K'$ is a proper subset of $\dom \exprv$.
\end{itemize}
The two assumptions above are crucial for the correctness of the algorithm.
\begin{theorem}
Let~$\exprv$ be a set of exponential random variables~$\exprv(k) \sim \operatorname{Exp}{(\expparam(k))}$ indexed by~$k \in K$, let $\tracerv$ be the trace and $\structrv$ be the output of Algorithm~\ref{alg:general} respectively. Then
\begin{enumerate}
\item The output $\structrv$ is a function of trace $X( T, E) = X(T)$;
\item The trace~$\tracerv$ is a sequence of categorical latent variables. The outcome probabilities are either deterministic or proportional to $\expparam(k)$ for $k \in M \subset K$;
\item The elements of the conditional distribution $\exprv \mid \tracerv$ are distributed as a sum of exponential random variables
\begin{equation}
E(k) \mid T \overset{d}{=} \sum_{j \in M_k} \operatorname{Exp}{(\sum_{k' \in N_{k,j}} \expparam(k'))},
\end{equation}
where the subsets of keys $N_{k, j} \subset \dom \exprv$ and indices $M_k$ are determined by Algorithm~\ref{alg:general}  
\end{enumerate}
\end{theorem}
\begin{proof}
First, we show the recursion depth of Algorithm~\ref{alg:general} is limited and the algorithm halts if $f_{\text{split}}, f_{\text{map}}$ and $f_{\text{combine}}$ always halt.

Indeed, by construction $K'$ is injective and $\operatorname{Im} K'$ is a proper subset of $\operatorname{Dom} E$. For the next recursive call the domain size $|\dom \exprv' \circ K'|$ is strictly smaller than the domain $|\dom E|$. Therefore, at some point the Algorithm~\ref{alg:general} will meet the stop condition $\dom \exprv = \emptyset$.

Next, we will prove the three claims of the theorem by induction on the recursion depth. All the claims are trivial when $\dom \exprv = \emptyset$.

For an arbitrary recursion depth, the output of the algorithm is the output of $f_{\text{combine}}$. The induction hypothesis implies that the first argument~$\structrv'$ is a function of the trace~$\tracerv$. The second and the third arguments are hyper-parameters that do not change as the values of $\exprv$ change. The last argument is the trace. Therefore, $\structrv$ is a function of the trace~$\tracerv$ (modulo the omitted hyper-parameters $\dom \exprv$ and $J$).

To prove the second and the third claim, we repeat the derivation of Lemma~\ref{lemma:exp-min-trick} for the joint density of $\exprv$ and $\tracerv$. We denote the realizations of $\tracerv$ and $\exprv$ with lower-case letters. Recall that $\tracerv = \{T_i^j\}_{i,j}$ is the concatenation of trace variables $T_1^j, \dots, T_m^j$ for all recursion depths $j=1, \dots, k$. We regroup the joint density
\begin{equation}
\mathbb 1(\tracereal = \tracerv(\expreal)) \prod_{k \in \dom \exprv} \expparam(k) e^{-\expparam(k) \expreal(k)}
\end{equation}
using the disjoint subsets $K_i, \dom \exprv = \sqcup K_i$
\begin{align}
& \mathbb 1(\tracereal^{>1} = \tracerv^{>1}(\expreal)) \times \\
& \prod_{i=1}^m \mathbb \prod_{k \in K_i} \expparam(k) e^{-\expparam(k) \expreal(k)} \mathbb 1(\expreal(\tracereal^1_i) \leq \expreal(k)).
\end{align}
Then we apply Lemma~\ref{lemma:exp-min-trick} for $i=1,\dots,m$ and rewrite the product $\prod_{k\in K_i}$ as
\begin{align}
\label{eq:exp-min-trick-thm}
\frac{\expparam(\tracereal^1_i)}{\sum_{k \in K_i} \expparam(k)} \left( \sum_{k \in K_i} \expparam(k) \right) e^{-(\sum_{k \in K_i} \expparam(k)) \expreal(\tracereal^1_i)} \times \\
\prod_{k \in K_i } \left( \expparam'(k) e^{-\expparam'(k) \expreal'(k)} \mathbb 1(0 \leq \expreal'(k)) \right),
\end{align}
where $\expreal'(k)$ is the realization of the random variable $\exprv'(k) := \exprv(k) - \exprv(\tracerv^1_i)$ and $\expparam'(k) := \expparam(k)$ for $k \neq \tracereal_i$ and $\expparam(\tracereal^1_i) := +\infty$\footnote{As a convention, $0 \overset{d}{=} \operatorname{Exp}{(+\infty)}$ and $\expparam e^{-\expparam x} = \delta(x)$ for $\expparam = +\infty$}. 

The recursive call applies to a subset of $\exprv'$ indexed with $\operatorname{Im}{K'}$. By the induction hypothesis, $\{T^j_i\}$ for $j > 1$ is a set of categorical variables with outcome probabilities proportional to $\expparam'$. The equation~\ref{eq:exp-min-trick-thm} implies that for $j=1$ the trace variables are also categorical with outcome probabilities proportional to $\expparam$. This proves the second claim of the theorem.

Finally, by the induction hypothesis, $\exprv'(k)$ with $k \in \operatorname{Im}{K'}$ is distributed as a sum of exponential random variables $\exprv'(k) \mid \tracerv \overset{d}{=} \sum_{j \in M'_k} \operatorname{Exp}{(\sum_{k' \in N'_{k, j}} \expparam'(k'))}.$ At the same time, the other variables~$\exprv'(k)$ with $k \in \dom{E} \setminus \operatorname{Im}{K'}$ are mutually independent with $\exprv'(k'), k' \in \operatorname{Im}{K'}$ and $\exprv'(k) \mid \tracerv^1_1, \dots, \tracerv^1_m \overset{d}{=} \operatorname{Exp}{\expparam'(k)}$.

Since $\exprv'(k) = \exprv(k) - \exprv(\tracerv^1_i)$, for $k \in K_i \cap \operatorname{Im}{K'}$ we have
\begin{align}
& \exprv(k) \mid \tracerv \overset{d}{=} \\
& \operatorname{Exp}{(\sum_{k \in K_i} \expparam(k))} + \sum_{j \in M'_k} \operatorname{Exp}{(\sum_{k' \in N'_{k, j}} \expparam'{(k')})}
\end{align}
and for $k \in K_i \cap \dom \exprv \setminus \operatorname{Im}{K'}$ we have
\begin{equation}
\exprv(k) \mid \tracerv \overset{d}{=} \operatorname{Exp}{(\sum_{k \in K_i} \expparam(k))} + \operatorname{Exp}{\expparam'(k)}.
\end{equation}
We replace $\expparam'(k)$ with $\expparam(k)$ for $k \notin \{\tracerv^1_1, \dots, \tracerv^1_m\}$ and with $+\infty$ for $k \in \{\tracerv^1_1, \dots, \tracerv^1_m\}$ and obtain the third claim of the theorem.
\end{proof}
\section{Algorithms}
\label{sec:algorithms}
We provide pseudo-code for computing $\logprob{\tracerv; \expparam}$ in Algorithm~\ref{alg:log-prob} and sampling $\exprv \mid \tracerv$ in Algorithm~\ref{alg:conditional}. Both algorithms modify Algorithm~\ref{alg:general} and use the same subroutines $f_{\text{split}}, f_{\text{map}}$ and $f_{\text{combine}}$. Algorithms~\ref{alg:log-prob}, \ref{alg:conditional} follow the structure as Algorithm~\ref{alg:general} and have at most linear overhead in time and memory for processing variables such as $\expparam'$ and $\logprob{T^j_i \mid \expparam}$.

The indexed set of exponential random variables $\exprv$ and the indexed set of the random variable parameters $\expparam$ have the same indices. Therefore, we replace $\dom \exprv$ with $\dom \expparam$. 

Both algorithms take the trace variable $\tracerv = \{T^j_i\}_{j,i}$ as input. Note that index $j$ enumerate recursion levels. Both algorithms process trace of the top recursion level $\tracerv^1_1, \dots, \tracerv^1_m$ and make a recursive call to process the subsequent trace~$\tracerv^{>1} := \{T^j_i\}_{i, j > 1}$. 
\begin{algorithm}[t]
\caption{$F_{\text{struct}}(\exprv, K, R)$ - returns structured variable $\structrv$ based on utilities $\exprv$ and auxiliary variables $K$ and $R$}
\label{alg:general}
\begin{algorithmic}
    \REQUIRE $\exprv, K, R$
    \ENSURE $\structrv$
    \IF{$f_{\text{stop}}(K, R)$}
    \STATE {\bf return}
    \ENDIF
    \STATE $P_1, \dots, P_m \Leftarrow f_{\text{split}}(K, R)$ \hfill \COMMENT{$\sqcup_{i=1}^m P_i = K$}
    \FOR{$i=1$ to $m$}
    \STATE $\tracerv_i \Leftarrow \arg\min_{k \in P_i} \exprv_k$
    \FOR{$k \in P_i$}
    \STATE $\exprv_k' \Leftarrow \exprv_k - \exprv_{\tracerv_i}$
    \ENDFOR
    \ENDFOR
    \STATE $K', R' \Leftarrow f_{\text{map}}(K, R, \{\tracerv_i\}_{i=1}^m)$ \hfill \COMMENT{$K' \subsetneq K$}
    \STATE $E' \Leftarrow \{E_k' \mid k \in K'\}$
    \STATE $\structrv' \Leftarrow F_{\text{struct}}(E', K', R')$ \hfill \COMMENT{Recursive call}
    \STATE {\bf return} $f_{\text{combine}}(\structrv', K, R, \{\tracerv_i\}_{i=1}^m)$
\end{algorithmic}
\end{algorithm}

\begin{algorithm}[t]
\caption{$F_{\text{log-prob}}(\tracerv, \expparam, K, R)$ - returns $\logprob{\tracerv; \expparam}$ for trace $\tracerv = \{\tracerv^j_i\}_{i,j}$, rates $\expparam$ and auxiliary variables $K, R$}
\label{alg:log-prob}
\begin{algorithmic}
    \REQUIRE $\tracerv, \expparam, K, R$
    \ENSURE $\logprob{\tracerv; \expparam}$
    \IF{$f_{\text{stop}}(K, R)$}
    \STATE {\bf return}
    \ENDIF
    \STATE $P_1, \dots, P_m \Leftarrow f_{\text{split}}(K, R)$ 
    \FOR{$i=1$ to $m$}
    \STATE \COMMENT{The superscript $j$ in $T^j_i$ denotes the recursion level}
    \STATE $\logprob{\tracerv^1_i; \expparam} \Leftarrow \log \expparam_{\tracerv^1_i} - \log \left(\sum_{k \in P_i} \expparam_k \right)$
    \FOR{$k \in P_i \setminus \{\tracerv^1_i\}$ }
    \STATE $\expparam'_{k} \Leftarrow \expparam_k$
    \ENDFOR
    \STATE $\expparam'_{\tracerv^1_i} \Leftarrow +\infty$ \hfill \COMMENT{Because $\exprv'(\tracerv^1_i) = 0$}
    \ENDFOR
    \STATE $K', R' \Leftarrow f_{\text{map}}(K, R, \{\tracerv^1_i\}_{i=1}^m)$
    \STATE $\expparam' \Leftarrow \{ \expparam_k' \mid k \in K'\}$
    \STATE \COMMENT{Recurse to compute log-prob for $\tracerv^{>1} := \{T^j_i\}_{j > 1}$}
    \STATE $\logprob{\tracerv^{>1} \mid \tracerv^{1}; \expparam} \Leftarrow F_{\text{log-prob}}(\tracerv^{>1}, \expparam', K', R')$
    \STATE {\bf return} $ \sum_{i=1}^m \logprob{\tracerv^1_i; \expparam} + \logprob{\tracerv^{>1} \mid \tracerv^1 ; \expparam}$
\end{algorithmic}
\end{algorithm}

\begin{algorithm}[t]
\caption{$F_{\text{cond}}(\tracerv, \expparam, K, R)$ - returns a utility sample $\exprv \mid \tracerv, \expparam$ with rates $\expparam$ conditioned on the trace $\tracerv = \{ T^j_i \}_{ij}$}
\label{alg:conditional}
\begin{algorithmic}
    \REQUIRE $\tracerv, \expparam, K, R$
    \ENSURE $\exprv$
    \IF{$f_{\text{stop}}(K, R)$}
    \STATE {\bf return}
    \ENDIF
    \STATE $P_1, \dots, P_m \Leftarrow f_{\text{split}}(K, R)$ 
    \FOR{$i=1$ to $m$}
    \STATE $\exprv_{\tracerv^1_i} \sim \operatorname{Exp}{(\sum_{k \in P_i} \expparam_k)}$ \hfill \COMMENT{Sample the $\min$}
    \FOR{$k \in P_i \setminus \{\tracerv^1_i\}$ }
    \STATE $\expparam'_{k} \Leftarrow \expparam_k$
    \ENDFOR
    \STATE $\expparam'_{\tracerv^1_i} \Leftarrow +\infty$ \hfill \COMMENT{Because $\exprv'(\tracerv^1_i) = 0$}
    \ENDFOR
    \STATE $K', R' \Leftarrow f_{\text{map}}(K, R, \{\tracerv^1_i\}_{i=1}^m)$
    \STATE $\expparam' \Leftarrow \{\expparam_k \mid k \in K'\}$
    \STATE $\exprv' \Leftarrow F_{\text{cond}}(\tracerv^{>1}, \expparam', K', R')$ \hfill \COMMENT{Recursive call}
    \STATE \COMMENT{Sample the rest of the utilities}
    \FOR{$k \in K \setminus K'$}
    \STATE $\exprv_k' \sim \operatorname{Exp}{(\expparam_k')}$ 
    \ENDFOR
    \STATE \COMMENT{Reverse the exponential-min trick}
    \FOR{$i=1$ to $m$}
    \FOR{$k \in P_i \setminus \{\tracerv^1_i\}$}
    \STATE $\exprv_k \Leftarrow \exprv_k' + \exprv_{\tracerv^1_i}$     \ENDFOR
    \ENDFOR
    \STATE {\bf return} $\exprv$
\end{algorithmic}
\end{algorithm}
\section{Implementation Details}
\label{sec:details}
In the paper, we chose the exponential random variables and the recursive form of Algorithm~\ref{alg:general} to simplify the notation.

In practice, we parameterized the rate of the exponential distributions as $\expparam = \exp(-\gumbparam)$, where $\theta$ was either a parameter or an output of a neural network. The parameter $\theta$ is essentially the location parameter of the Gumbel distribution and, unlike $\expparam$, can take any value in $\mathbb R$.

Additionally, the recursive form of Algorithm~\ref{alg:general} does not facilitate parallel batch computation. In particular, the recursion depth and the decrease in size of $\exprv$ may be different for different objects in the batch. Therefore, Algorithms~\ref{alg:log-prob},\ref{alg:conditional} may require further optimization.

For the top-k algorithm, we implemented the parallel batch version. To keep the input size the same, we masked the omitted random variables with $+\infty$. We modeled the recursion using an auxiliary tensor dimension.

For the Kruskal's algorithm, we implemented the parallel batch version and used the $+\infty$ masks to preserve the set size. We rewrote the recursion as a Python for loop.

To avoid the computation overhead for the Chu-Liu-Edmonds algorithm, we implemented the algorithms in C++ and processed the batch items one-by-one.

\section{Experimental Details}
\label{sec:experiments}
Setting up all the experiments we followed details about data generation, models and training procedures, described by \cite{paulus2020gradient}, to make a valid comparison of the proposed score function methods with Stochastic Softmax Tricks. In each experiment we fixed the number of function evaluations $N$ per iteration instead of batch size to make a more accurate comparison in terms of computational resources. With $N$ fixed, RELAX estimator was trained with batch size equal to $N$, while REINFORCE+ and G-REINFORCE+ were trained with batch size $N/K$ and $K$ samples of the latent structure for each object. Results reported in tables were obtained by taking the best model with respect to the validation task metric (MSE, ELBO, accuracy) and evaluating it on the test set. We applied the bootstrapping procedure, defined by \cite{paulus2020gradient}, over results of the trained models to obtain standard deviations of metrics.

\subsection{Top-K and Beer Advocate}

\textbf{Data.} As a base, we used the BeerAdvocate \citep{mcauley2012learning} dataset, which consists of beer reviews and ratings for different aspects: Aroma, Taste, Palate and Appearance. In particular, we took its decorrelated subset along with the pretrained embeddings from \cite{lei2016rationalizing}. Every review was cut to $350$ embeddings, aspect ratings were normalized to $[0,1]$.

\textbf{Model.} We used the Simple and Complex models defined by \cite{paulus2020gradient} for parameterizing the mask. The Simple model architecture consisted of Dropout (with $p = 0.1$) and a one-layered convolution with one kernel. In the Complex model architecture, two more convolutional layers with 100 filters and kernels of size 3 were added.

\textbf{Training.} We trained all models with $N=100$. We trained our models for 30 epochs, unlike 10 for the SST, since the score function methods generally have larger variance than the relaxation-based approaches. We used the same hyperparameters as in \cite{paulus2020gradient}, where it was possible. Hyperparameters for our training procedure were learning rate, final decay factor, weight decay, and number of latent samples for every example in batch. They were sampled from $\{1, 3, 5, 10, 30, 50, 100\} \times 10^{-4}, \{1, 10, 100, 1000\} \times 10^{-4}, \{0, 1, 10, 100\} \times 10^{-6}, \{1, 2, 4\} $ respectively.

Results for Appearance aspect can be found in Table~\ref{table:app}, for Taste aspect in Table~\ref{table:tas}, for Palate aspect in Table~\ref{table:pal}. In conclusion, we can state that the proposed method is comparable with SST on BeerAdvocate dataset.

\subsection{Spanning Tree and Graph Layout}

\textbf{Data.} For each dataset entry we obtained the corresponding ground truth spanning tree by sampling a fully-connected graph on 10 nodes and applying Kruskal algorithm. Graph weights were sampled independently from $\text{Gumbel}(0, 1)$ distribution. Initial vertex locations in $\mathbb{R}^2$ were distributed according to $N(0, I)$. Given the spanning tree and initial vertex locations, we applied the force-directed algorithm \cite{fruchterman1991graph} for $T=10$ or $T=20$ iterations to obtain system dynamics. We dropped starting positions and represented each dataset entry as the obtained sequence of $T=10$ or $T=20$ observations. We generated 50000 examples for the training set and 10000 examples for the validation and test sets.

\textbf{Model.} Following \cite{paulus2020gradient} we used the NRI model with encoder and decoder architectures analogous to the MLP encoder and MLP decoder defined by \cite{kipf2018neural}.

Given the observation of dynamics, GNN encoder passed messages over the fully connected graph. Denoting its final edge representation by $\theta$, we obtained parameters of the distribution over undirected graphs as $\frac{1}{2}(\theta_{ij} + \theta_{ji})$ for an edge $i \leftrightarrow j$. Hard samples of spanning trees
were then obtained by applying the Kruskal algorithm on the perturbed symmetrized $\theta$.

GNN decoder took observations from previous timesteps and the adjacency matrix $X$ of the obtained spanning tree as its input. It passed messages over the latent tree aiming at predicting future locations of the vertices. We used two separate networks to send messages over two different connection types ($X_{ij} = 0$ and $X_{ij} = 1$). Since parameterization of the model was ambiguous in terms of choosing the correct graph between $X$ and $1 - X$, we measured structure metrics with respect to both representations and reported them for the graph with higher edge precision.

In experiments with RELAX we needed to define a critic. It was a simple neural network defined as an MLP which took observations concatenated with the perturbed weights and output a scalar. It had one hidden layer and ReLU activations.

\textbf{Objective.} During training we maximized ELBO (lower bound on the observations' log-probability) with gaussian log-likelihood and KL divergence measured in the continuous space of exponential noise. It resulted in an objective which was also a lower bound on ELBO with KL divergence measured with respect to the discrete distributions.

\textbf{Training.} We fixed the number of function evaluations per iteration at $N = 128$. We trained each of the score function models for 200000 iterations and evaluated them every epoch. SST models, which results were obtained from \cite{paulus2020gradient}, were trained for 50000 iterations. We found it important to run our proposed methods for a larger number of iterations since the score function methods generally have larger variance than the relaxation based approaches.

We saved the best model with respect to ELBO on validation. We used constant learning rate and Adam optimizer with default hyperparameters. For all estimators we tuned learning rate sampling it from $\{0.00005, 0.0001, 0.0003, 0.0005, 0.001\}$. Additionally, for REINFORCE+ we tuned $K$ in $\{2, 4, 6, 8, 16\}$ and for RELAX we tuned size of the critic hidden layer in $\{256, 512, 1024, 2056\}$.

\subsection{Arborescence and Unsupervised Parsing}

\textbf{Data.} We took the ListOps \cite{nangia2018listops} dataset, containing arithmetical prefix expressions, e.g.  $\texttt{min[3 med[3 5 4] 2]}$, as a base, and modified its sampling procedure. We considered only the examples of length in $[10, 50]$ that do not include the $\texttt{summod}$ operator and have bounded depth $d$. Depth was measured with respect to the ground truth parse tree,  defined as a directed graph with edges going from functions to their arguments. We generated equal number of examples for each $d$ in $\{1, \ldots, 5\}$. Train dataset contained 100000 samples, validation and test sets contained 20000 samples.

\textbf{Model.} Model mainly consisted of two parts which we call encoder and classifier.

Encoder was the pair of identical left-to-right LSTMs with one layer, hidden size 60 and dropout probability 0.1. Both LSTMs used the same embedding lookup table. Matrices that they produced by encoding the whole sequence were multiplied to get parameters of the distribution over latent graphs. Equivalently, parameter for the weight of the edge $i \rightarrow j$ was computed as $\theta_{ij} = \langle v_i, w_j \rangle$, where $v_i$ and $w_j$ are hidden vectors of the corresponding LSTMs at timesteps $i$ and $j$. Given $\theta \in \mathbb{R}^{n \times n}$, we sampled matrix weights from the corresponding factorized exponential distribution. Hard samples of latent arborescences, rooted at the first token, were obtained by applying Chu-Liu Edmonds algorithm to the weighted graph.

Classifier mainly consisted of the graph neural network which had the initial sequence embedding as an input and ran 5 message sending iterations over the sampled arborescence's adjacency matrix. It had its own embedding layer different from used in the encoder. GNN's architecture was based on the MLP decoder model by \cite{kipf2018neural}. It had a two-layered MLP and did not include the last MLP after message passing steps. Output of the GNN was the final embedding of the first token which was passed to the last MLP with one hidden layer. All MLPs included ReLU activations and dropout with probability 0.1.

In experiments with RELAX we needed to define a critic. It contained LSTM used for encoding of the initial sequence. It was left-to-right, had a single layer with hidden size 60 and dropout probability 0.1. It had its own embedding lookup table. LSTM's output corresponding to the last token of the input sequence was concatenated with a sample of the graph adjacency matrix and fed into the output MLP with one hidden layer of size 60 and ReLU activations. Before being passed to the MLP, weights of the adjacency matrix were centered and normalized.

\textbf{Training.} We fixed the number of function evaluations per iteration at $N = 100$. We used AdamW optimizer, separate for each part of the model: encoder, classifier and critic in case of RELAX. They all had constant, but not equal in general case, learning rates, and default hyperparameters. We tuned learning rates for all models and the number of latent samples for REINFORCE+ and G-REINFORCE+ by making 8 runs with their different combinations.

We trained G-REINFORCE+ for 600000, REINFORCE+ for 300000 and RELAX for 200000 iterations. SST models, which results were obtained from \cite{paulus2020gradient}, were trained for 50000 iterations. Table with results indicates that exponential score based models, even trained for a much more number of gradient updates, perform worse than their discrete score counterparts both in terms of task accuracy and structure-based metrics. We explain the necessity of training score function methods for a higher number of iterations than the relaxation based SST by a general empirical observation that they have larger variance.

\section{Kruskal Pseudo-Code}

\begin{algorithm}
\caption{$F_\text{Kruskal}(\exprv, K, R)$ - finds the minimum spanning tree given edges~$K$ with the corresponding weights~$\exprv$; Call with $R := \{ \{v\} \mid v \in V\}$ set as node singletons} 
\label{alg:kruskal}
\begin{algorithmic}
    \REQUIRE $\exprv, K, R$
    \ENSURE $\structrv$
    \IF{$|R| = 1$}
    \STATE {\bf return}
    \ENDIF
    \STATE $\tracerv \Leftarrow \arg\min_{k \in K} \exprv_k$ \hfill\COMMENT{Find the smallest edge}
    \FOR{$k \in K$}
    \STATE $\exprv_k' \Leftarrow \exprv_k - \exprv_\tracerv$
    \ENDFOR
    \STATE \COMMENT{For $\tracerv = (u, v)$ find $R_u, R_v \in R$ s.t. $u \in R_u$, $v \in R_v$}
    \STATE $R_u, R_v \Leftarrow \operatorname{find\_connected\_components}(R, T)$
    \STATE \COMMENT{Merge the connected components $R_u$ and $R_v$}
    \STATE $R' \Leftarrow (R \setminus \{R_u, R_v\}) \cup (\{R_u \cup R_v\})$
    \STATE \COMMENT{Remove edges connecting $R_u$ and $R_v$}
    \STATE $K' \Leftarrow K \setminus \{(u', v') \in K \mid u', v' \in R_u \cup R_v\}$
    \STATE $\exprv' \Leftarrow \{\exprv_k' \mid k \in K'\}$
    \STATE \COMMENT{Edges in $\structrv'$ form a spanning tree for nodes in $R'$}
    \STATE $\structrv' \Leftarrow F_\text{Kruskal}(\exprv', K', R')$
    \STATE \COMMENT{$X \cup \{\tracerv = (u, v)\}$ is a spanning tree for nodes in $R$}
    \STATE {\bf return} $\structrv' \cup \{\tracerv\}$
\end{algorithmic}
\end{algorithm}

\section{CLE Pseudo-Code}

\begin{algorithm}[t]
\caption{\Cyril{Modify caption}Find the minimum arborescence $\structrv$ with the root $r$ for a digraph with edge weights $\exprv$}
\caption{$F_{\text{CLE}}(E, K, R, r)$ - finds the minimum arborescence $\structrv$ of a directed graph with edges $K$ of weight $\exprv$ and root node $r$. Auxiliary variable $R$ is a partition of nodes indicating merged nodes, initialized with node singletons $R := \{ \{v \} \mid v \in V \}.$}
\label{alg:cle}
\begin{algorithmic}
    \REQUIRE $\exprv, K, R, r$
    \ENSURE $\structrv$
    \IF{$|R| = 1$}
    \STATE {\bf return}
    \ENDIF
    \STATE \COMMENT{Split $K$ into sets of edges ending at $R_v \in R$, $r \notin R_v$} 
    \STATE $P_1, \dots, P_m \Leftarrow f_{\text{split}}(K, R, r)$ \hfill \COMMENT{$m = |R| - 1$}
    \FOR{$i=1$ to $m$}
    \STATE $\tracerv_i \Leftarrow \arg\min_{k \in P_i} \exprv_k$
    \FOR{$k \in P_i$}
    \STATE $\exprv_k' \Leftarrow \exprv_k - \exprv_{\tracerv_i}$
    \ENDFOR
    \ENDFOR
    \STATE \COMMENT{Find a loop $C$ assuming nodes $R$ and edges $\{T_i\}_{i=1}^{m}$}
    \STATE $C \Leftarrow \operatorname{find\_loop}(R, \{T_i\}_{i=1}^m)$
    \IF{$C = \emptyset$}
    \STATE {\bf return} $\{T_i\}_{i=1}^m$ \hfill
    \ENDIF
    \STATE \COMMENT{Contract the loop nodes into a single node}
    \STATE $R' \Leftarrow (R \setminus \{C_i \mid C_i \in C\}) \cup (\cup_{i=1}^{|C|} C_i)$
    \STATE $K' \Leftarrow K \setminus \{(u, v) \in K \mid u \in C_i, v \in C_{j}\}$ 
    \STATE $\exprv' \Leftarrow \{\exprv_k' \mid k \in K' \}$
    \STATE \COMMENT{Find arborescence for the contracted graph}
    \STATE $\structrv' \Leftarrow F_{\text{CLE}}(\exprv', K', R', r)$
    \STATE \COMMENT{Add to $\structrv'$ all loop edges but one}
    \STATE $\structrv \Leftarrow \structrv' \cup \{T_i \mid T_i \text{ in cycle, preserves arborescence}\}$
    \STATE {\bf return} $\structrv$
\end{algorithmic}
\end{algorithm}

\bibliographystyle{unsrtnat}
\bibliography{references}